\documentclass[11pt]{article}
\usepackage{bbm}
\usepackage{eqnarray,amsmath,amsfonts,amsthm,mathrsfs}
\usepackage{color}
\usepackage{bm}
\usepackage{amssymb}
\usepackage[dvips]{graphicx}
\usepackage{epsfig}
\usepackage{epsf}
\usepackage{float}
\usepackage{subfigure}
\usepackage{graphicx,float,fancyhdr,multirow,hyperref}
\usepackage{booktabs,longtable,authblk}
\usepackage{hhline}
\usepackage[inline]{enumitem}

\usepackage{fancyhdr}
\usepackage[top=2.5cm, bottom=2.5cm, left=3cm, right=3cm]{geometry}
\usepackage{graphicx}
  
\newtheorem{theorem}{Theorem}
\newtheorem{assumption}{Assumption}
\newtheorem{corollary}{Corollary}
\newtheorem{definition}{Definition}

\newtheorem{lemma}{Lemma}
\newtheorem{proposition}{Proposition}
\newtheorem{remark}{Remark}
\newtheorem{example}{Example}

\allowdisplaybreaks[4]

\def\begeqn{\begin{equation}}
	\def\endeqn{\end{equation}}
\def\begth{\begin{theorem}}
	\def\endth{\end{theorem}}
\def\begprop{\begin{proposition}}
	\def\endprop{\end{proposition}}
\def\begcor{\begin{corollary}}
	\def\endcor{\end{corollary}}
\def\begdef{\begin{definition}}
	\def\enddef{\end{definition}}
\def\beglemm{\begin{lemma}}
	\def\endlemm{\end{lemma}}
\def\begexm{\begin{example}}
	\def\endexm{\end{example}}
\def\begrem{\begin{remark}}
	\def\endrem{\end{remark}}
\def\begassum{\begin{assumption}}
	\def\endassum{\end{assumption}}

\numberwithin{equation}{section}

\newcommand{\EE}{\mathbb E}
\newcommand{\PP}{\mathbb P}
\newcommand{\II}{\mathbb I}
\newcommand{\sF}{\mathscr F}
\newcommand{\sH}{\mathscr H}
\newcommand{\sN}{\mathscr N}
\newcommand{\sT}{\mathscr T}

\newcommand{\sU}{\mathscr U}
\newcommand{\sL}{\mathscr L}
\newcommand{\sW}{\mathscr W}
\newcommand{\sC}{\mathscr C}
\newcommand{\cC}{\mathcal C}
\newcommand{\cI}{\mathcal I}
\newcommand{\cG}{\mathcal G}
\newcommand{\cK}{\mathcal K}
\newcommand{\cL}{\mathcal L}
\newcommand{\cQ}{\mathcal Q}
\newcommand{\cT}{\mathcal T}

\newcommand{\bx}{\textbf x}
\newcommand{\by}{\textbf y}

\title{Distributed Learning with Discretely Observed Functional Data}

\author{Jiading Liu}
\author{Lei Shi}
\affil{School of Mathematical Sciences and Shanghai Key Laboratory for Contemporary Applied Mathematics, Fudan University, Shanghai 200433, China.\\
Emails: jiadingliu1999@gmail.com, leishi@fudan.edu.cn\\}

\date{}

\begin{document}
	\maketitle
	\begin{abstract}
By selecting different filter functions, spectral algorithms can generate various regularization methods to solve statistical inverse problems within the learning-from-samples framework. This paper combines distributed spectral algorithms with Sobolev kernels to tackle the functional linear regression problem. The design and mathematical analysis of the algorithms require only that the functional covariates are observed at discrete sample points. Furthermore, the hypothesis function spaces of the algorithms are the Sobolev spaces generated by the Sobolev kernels, optimizing both approximation capability and flexibility. Through the establishment of regularity conditions for the target function and functional covariate, we derive matching upper and lower bounds for the convergence of the distributed spectral algorithms in the Sobolev norm. This demonstrates that the proposed regularity conditions are reasonable and that the convergence analysis under these conditions is tight, capturing the essential characteristics of functional linear regression. The analytical techniques and estimates developed in this paper also enhance existing results in the previous literature.

	\end{abstract}
    
    
	\section{Introduction}\label{section: introduction}

Recent years have witnessed the precipitous development of functional data analysis (FDA) across various fields, including neuroscience, linguistics, medicine, economics and so on (See \cite{ramsay2005functional,chen2017modelling,petersen2019frechet,tavakoli2019spatial} and the references therein). Functional linear regression, a pivotal subfield of the FDA, has garnered substantial interest in statistics and machine learning communities. Without loss of generality (and allowing for rescaling when necessary), we assume the domain of the functional covariate to be $\sT:= [0,1]$. The square-integrable functions over $\sT$ form a Hilbert space, denoted by $\sL^2(\sT)$, with the inner product
$$\langle g_1, g_2\rangle_{\sL^2}:= \int_{\sT} g_1(t)g_2(t) dt, \quad \forall g_1,g_2 \in \sL^2(\sT).$$  This paper focuses on the functional linear regression model expressed as
 \begin{equation}\label{LFRmodel}
 Y =\langle\beta_0, X \rangle_{\sL^2} + \epsilon,
 \end{equation}
where $Y\in \mathbb{R}$ represents a scalar response, $X\in \sL^2(\sT)$ is a functional covariate, $\beta_0\in\sL^2(\sT)$ denotes the slope function, and $\epsilon$ is a zero-mean random noise, which is independent of $X$ and has finite variance. The primary objective of functional linear regression is to estimate $\beta_0$ utilizing a training sample set generated by \eqref{LFRmodel}.

In the context of regression, estimating $\beta_0$ directly, rather than using the functional  $L_{\beta_0}(\cdot):=\langle \beta_0, \cdot \rangle_{\sL^2}$ for prediction, is a typical inverse problem and has attracted widespread attention in the fields of statistical inference and inverse problems (see, e.g., \cite{blanchard2018optimal,goldenshluger2000adaptive,reimherr2023optimal,yuan2010reproducing}). Moreover, instead of assuming functional covariates are fully observed, this paper adopts a more practical setting by considering functional covariates that are observed only at discrete points. The training sample set in this paper is
 \[S:=\left\{\left(X_i(r_1),X_i(r_2),\cdots,X_i(r_m),X_i(r_{m+1}),Y_i\right)\right\}_{i=1}^N,\] where $\left\{(X_i,Y_i)\right\}_{i=1}^N$ are $N$ independent copies of the random variable $(X,Y)$, and functional covariates $\left\{X_i\right\}_{i=1}^N$ are observed at discrete points $\left\{r_k\right\}_{k=1}^{m+1}$, with $m\geq1$ being an integer and $0\leq r_1<\cdots<r_m<r_{m+1}\leq 1$.
 
This paper employs the spectral regularization algorithms based on Sobolev kernels to estimate $\beta_0$. We first clarify some notations. The subspace of $\sL^2(\sT)$ where the weak derivatives up to order $\alpha \geq 1$ remain in $\sL^2(\sT)$, called the Sobolev space of order $\alpha$ on $\sT$, is given by 
\begin{equation}\label{sobolev}
  \sW^{\alpha,2}(\sT):=\left\{g:[0,1]\to\mathbb{R} \mid D^k g \in \sL^2(\sT), \ \forall 0 \leq k \leq \alpha\right\},
\end{equation}where $D^{k}g$ denotes the $k$-th weak derivative of $g$. The integer parameter $\alpha$ serves as a direct measure of smoothness within $\sW^{\alpha,2}(\sT)$. The Sobolev embedding theorem (see, e.g., Theorem 4.12 of \cite{adams2003sobolev}) guarantees that $\sW^{\alpha,2}(\sT)$ is continuously embedded into $\sC^{(\alpha-1)}([0,1])$ for any integer $\alpha \geq 1$, where $\sC^{(\alpha-1)}(\sT)$ represents the space of functions that are $(\alpha-1)$-times continuously differentiable on $\sT$. If equipped with the standard inner product
\begin{align*}
	\left\langle g_1,g_2\right\rangle_{\sW^{\alpha,2}_{std}} := \sum_{k=0}^{\alpha}\int_{\sT}D^{k}g_1(t)D^{k}g_2(t)dt, \quad \forall g_1,g_2\in \sW^{\alpha,2}(\sT),
\end{align*}
then $\sW^{\alpha,2}(\sT)$ becomes a reproducing kernel Hilbert space (RKHS), known as the standard Sobolev space of order $\alpha$, with reproducing kernel denoted by $\cK_{\alpha}^{std}:\sT\times \sT \rightarrow \mathbb{R}$ (see, e.g., \cite{berlinet2011reproducing}). Let $\left\|\cdot\right\|_{\sW^{\alpha,2}_{std}}$ represent the norm induced by the standard inner product $\left\langle \cdot,\cdot\right\rangle_{\sW^{\alpha,2}_{std}}$. Although the standard Sobolev spaces are widely used in previous papers on statistical learning and inverse problems, they have a significant computational drawback: when $\alpha\geq 2$, there is no explicit formula for the reproducing kernel $\cK_{\alpha}^{std}$ of standard Sobolev space of order $\alpha$. Therefore, we consider equipping $\sW^{\alpha,2}(\sT)$
with an alternative inner product 
\begin{equation}\label{innerproduct}
\begin{split}
\langle g_1, g_2 \rangle_{\sW^{\alpha,2}}&:=\sum_{k=0}^{\alpha-1}\left(\int_{\sT} D^k g_1(t) dt\right)\left(\int_{\sT} D^k g_2(t) dt\right)\\
&\qquad +\int_{\sT} D^{\alpha} g_1(t) D^{\alpha} g_2(t)dt, \quad \forall g_1,g_2 \in \sW^{\alpha,2}(\sT).
\end{split}
\end{equation}
As a result of Cauchy-Schwartz inequality and Poincar\'e-Wirtinger inequality (see, e.g., \cite{brezis2011functional}), the norm $\left\|\cdot\right\|_{\sW^{\alpha,2}}$ induced by $\langle \cdot, \cdot \rangle_{\sW^{\alpha,2}}$ is equivalent to the norm $\left\|\cdot\right\|_{\sW^{\alpha,2}_{std}}$, i.e., there exists $0<c<C<\infty$, such that $c\left\|g\right\|_{\sW^{\alpha,2}_{std}}\leq\left\|g\right\|_{\sW^{\alpha,2}} \leq C\left\|g\right\|_{\sW^{\alpha,2}_{std}}$ for all $g\in \sW^{\alpha,2}(\sT)$.
  The space $\sW^{\alpha,2}(\sT)$ equipped with $\langle \cdot, \cdot \rangle_{\sW^{\alpha,2}}$ is also an RKHS with reproducing kernel denoted by $\cK_{\alpha}:\sT\times \sT \rightarrow \mathbb{R}$. This space is called the unanchored Sobolev space of order $\alpha$ to be distinguished from the standard Sobolev spaces. Its reproducing kernel, known as the (unanchored) Sobolev kernel of order $\alpha$, can be explicitly expressed as (see, e.g., \cite{gu2013smoothing}):
$$
 \cK_{\alpha}(t,t'):=\sum_{k=0}^{\alpha} \frac{B_k(t) B_k(t')}{(k !)^2}+\frac{(-1)^{\alpha+1}}{(2 \alpha)!} B_{2 \alpha}(|t-t'|),\quad \forall t,t'\in \sT,
$$
where $\left(B_k\right)_{k \geq 0}$ denotes the sequence of Bernoulli polynomials. Furthermore, for a non-integer parameter $\alpha>0$, the unanchored Sobolev space $\sW^{\alpha,2}(\sT)$ can be defined via the real interpolation. All details will be demonstrated in Section \ref{subsection: sobolev space}.

For clarity and convenience, we fix a parameter $\alpha^*>1/2$ throughout this paper. Given a family of functions $\left\{\Psi_\lambda:[0,\infty) \rightarrow \mathbb{R}| \lambda\in (0,1)\right\}$ being a filter function indexed by the parameter $\lambda\in (0,1)$ (see Section \ref{subsection: filter functions} or \cite{bauer2007regularization} for details). Define the empirical operator $\cG_{{\alpha^*},\bx}: \sL^2(\sT)\rightarrow \mathbb{R}^N$ with $\bx:=\left\{\left(X_i(r_1),\cdots,X_i(r_{m+1})\right)\right\}_{i=1}^N$ by
\begin{align*}
	\cG_{{\alpha^*},\bx}(f)
	:=\left(\sum_{k=1}^{m}(r_{k+1}-r_k)\left\langle f,\cK^{1/2}_{{\alpha^*}}(\cdot, r_k)\right\rangle_{\sL^2}X_1(r_k),\cdots, \sum_{k=1}^{m}(r_{k+1}-r_k)\left\langle f,\cK^{1/2}_{{\alpha^*}}(\cdot, r_k)\right\rangle_{\sL^2}X_N(r_k)\right)^{T}
\end{align*}
for any $f\in \sL^{2}(\sT)$, where $\cK_{\alpha^*}^{1/2}$ is the square root of the unanchored Sobolev kernel $\cK_{\alpha^*}$ (see Section \ref{subsection: RKHS}). And we define the empirical operator $\cT_{{\alpha^*},\bx}:\sL^2(\sT)\rightarrow\sL^2(\sT)$ by
\begin{align*}
	\cT_{{\alpha^*},\bx}&:= \frac{1}{N}\cG_{{\alpha^*},\bx}^*\cG_{{\alpha^*},\bx},
\end{align*}
where $\cG_{{\alpha^*},\bx}^*:\mathbb{R}^N\rightarrow \sL^2(\sT)$ is the adjoint operator of $\cG_{{\alpha^*},\bx}$ defined by
\[\cG_{{\alpha^*},\bx}^*(a):= \sum_{i=1}^{N}\sum_{k=1}^{m}a_i(r_{k+1}-r_k)\cK_{{\alpha^*}}^{1/2}(\cdot,r_k)X_i(r_k), \quad \forall a\in \mathbb{R}^N.\]
Then we can construct an spectral regularization estimator $\hat{\beta}_{S, \alpha^*,\Psi_\lambda}$ with $\Psi_\lambda$ in $\sW^{\alpha^*,2}(\sT)$ to approximate $\beta_0$ based on the discretely observed functional data $S$ as:
\begin{align}\label{spectral regularization estimator}
	\hat{\beta}_{S, \alpha^*,\Psi_\lambda} = \cL_{\cK_{{\alpha^*}}}^{1/2}\Psi_\lambda(\cT_{\alpha^*,\bx})\frac{1}{N}\cG_{{\alpha^*},\bx}^*\by,
\end{align}
where $\cL_{\cK_{\alpha^*}}^{1/2}$ is the $1/2$-th power of the integral operator induced by the unanchored Sobolev kernel $\cK_{\alpha^*}$ and $\by:=(Y_1,\cdots,Y_N)^{T}\in \mathbb{R}^N$.

The spectral regularization algorithms cover many standard algorithms for functional linear regression, including the Tikhonov regularization algorithm, iterated Tikhonov regularization algorithm, and gradient descent algorithm, and we will illustrate them in Section \ref{subsection: filter functions}.

In addressing the challenges posed by massive datasets, the algorithm \eqref{spectral regularization estimator} is hindered by significant algorithmic complexity involving both computational time and memory requirements. To improve computational feasibility for large-scale training databases, we adopt a distributed strategy for implementing algorithm \eqref{spectral regularization estimator}. This approach entails randomly partitioning the sample set $S$ into $M$ disjoint subsets of equal size, denoted by $S_1,\cdots,S_M$. Subsequently, applying algorithm \eqref{spectral regularization estimator} individually to each subset $S_j$ yields local estimators $\hat{\beta}_{S_j,\alpha^*,\Psi_\lambda}$, defined as
\begin{align*}
\hat{\beta}_{S_j, \alpha^*,\Psi_\lambda} = \cL_{\cK_{{\alpha^*}}}^{1/2}\Psi_\lambda(\cT_{\alpha^*,\bx_j})\frac{1}{|S_j|}\cG_{{\alpha^*},\bx_j}^*\by_j,
\end{align*}
where $|S_j|$ represents the number of elements in $S_j$, ${\bx}_j$ comprises the discrete samples of $X$ in $S_j$, and ${\by}_j\in \mathbb{R}^{|S_j|}$ is a vector containing the samples of $Y$ in $S_j$. The final distributed estimator is derived by averaging the local estimators $\{\hat{\beta}_{S_j,\alpha^*,\Psi_\lambda}\}_{j=1}^M$, expressed as 
\begin{equation}\label{finalestimator}
    \overline{\beta}_{S,\alpha^*,\Psi_\lambda}:=\frac{1}{M}\sum_{j=1}^M \hat{\beta}_{S_j,\alpha^*,\Psi_\lambda}.
\end{equation}
Implementing the distributed version of algorithm \eqref{spectral regularization estimator} substantially mitigates the computational complexity in terms of time and memory, reducing it to approximately $\frac{1}{M^2}$ of the original complexity.

In order to handle the functional covariates observed on discrete sample points, we need to introduce some regularity conditions on the functional covariate $X$ and the slope function $\beta_0$. In the present paper, we assume that the functional covariate and slope function $X,\beta_0\in \sW^{\alpha^*,2}(\sT)$ for the same $\alpha^*$ in algorithms \eqref{spectral regularization estimator} and \eqref{finalestimator}. Such an assumption requires the functional covariate and slope function to have some sort of continuity and is common when aiming to recover the integral of a function from its discrete observations (see, e.g., \cite{raskutti2012minimax,wang2022functional,yuan2016minimax}). Then the performance of distributed estimator $\overline{\beta}_{S,\alpha^*,\Psi_\lambda}$ can be evaluated via the estimation error:
\begin{align}\label{estimation error}
	\left\|\overline{\beta}_{S,\alpha^*,\Psi_\lambda} - \beta_0\right\|_{\sW^{\alpha^*,2}}^2.
\end{align}

The main contributions of our paper are summarized as follows. Under some mild assumptions on the functional covariate $X$, we establish the upper bounds on the convergence rates of the estimation error given by \eqref{estimation error} for different regularity conditions of $\beta_0$. Then, when the functional covariate $X$ satisfies some additional Gaussian property, we establish the upper bounds for the estimation error given by \eqref{estimation error} in a global sense, i.e., the upper bounds for the expectation of \eqref{estimation error}. We develop an innovative mathematical analysis by combining the asymptotic analysis techniques and concentration estimates for random operators with bounded arbitrary-order moment, which extends several previous results in published literature and is tight as the rates of upper and lower bounds on the performance of distributed estimators match.

We organize the rest of this paper as follows. In Section \ref{section: preliminaries}, we start with an introduction to notations,
background, and some preliminary results. In Section \ref{section: main results}, we present main assumptions and theorems in this paper. In Section \ref{section: comparison}, we provide a discussion of the assumptions, compare our analysis with related results, and present several directions for future research. We leave all proofs to Section \ref{section: convergence analysis} and Appendixes.

\section{Preliminaries}\label{section: preliminaries}
In this section, we will introduce some basic notations and background in our study.
\subsection{Basic Notations} 

We first recall some basic notations in operator theory (see, e.g., \cite{conway2000course}). Let $A: \sH \to \sH' $ be a linear operator, where $(\sH,\langle\cdot,\cdot\rangle_{\sH})$ and $(\sH',\langle\cdot,\cdot\rangle_{\sH'})$ are Hilbert spaces with the corresponding norms $\|\cdot\|_{\sH}$ and $\|\cdot\|_{\sH'}$. The set of bounded linear operators from $\sH$ to $\sH'$ is a Banach space with respect to the operator norm $\|A\|_{\sH,{\sH'}}=\sup_{\|f\|_{\sH}=1}\|Af\|_{{\sH'}}$, which is denoted by $\mathcal{B}(\sH, \sH)$ or $\mathcal{B}(\sH)$ if $\sH={\sH'}$. When $\sH$ and ${\sH'}$ are clear from the context, we will omit the subscript and simply denote the operator norm as $\|\cdot\|$. Let $A^*$ be the adjoint operator of $A$ such that $\langle Af,f'\rangle_\sH = \langle f,A^*f'\rangle_{\sH'}, \forall f \in \sH, f' \in {\sH'}$. We say that $A\in \mathcal{B}(\sH)$ is self-adjoint if $A^*=A$, and positive if $A$ is self-adjoint and $\langle Af, f\rangle_{\sH} \geq 0$ for all $f\in \sH$. For $f\in \sH$ and $f'\in {\sH'}$,
	define a rank-one operator $f\otimes f': \sH \to {\sH'}$ by $f \otimes f' (h)= \langle f, h \rangle_\sH f', \forall h \in \sH$. If $A\in \mathcal{B}({\sH})$ is compact and positive, Spectral Theorem ensures that there exists an orthonormal basis $\{e_k\}_{k\geq 1}$ in ${\sH}$ consisting of eigenfunctions of $A$ such that $A=\sum_{k\geq 1} \lambda_k e_k \otimes e_k,$ where the eigenvalues $\{\lambda_k\}_{k\geq 1}$ (with geometric multiplicities) are non-negative and arranged in decreasing order, and either the set $\{\lambda_k\}_{k\geq 1}$ is finite or $\lambda_k \to 0$ when $k\to \infty$. Moreover, for any $r>0$, we define the $r-$th power of $A$ as $A^r=\sum_{k \geq 1} \lambda^r_k e_k \otimes e_k,$ which is itself a positive compact operator on ${\sH}$. An operator $A \in \mathcal{B}(\sH,\sH')$ is Hilbert-Schmidt if  $\sum_{k\geq 1}\|Ae_k\|^2_{\sH'}<\infty$  for some (any) orthonormal basis $\{e_k\}_{k\geq 1}$ of $\sH$. All Hilbert-Schmidt operators can form a Hilbert space endowed with the inner product $\langle A, B\rangle_{\sF}:=\sum_{k\geq 1} \langle Ae_k, Be_k\rangle_{\sH'}$ and we denote the corresponding norm by $\|\cdot\|_{\sF}$. In particular, a Hilbert-Schmidt operator $A$ is compact, and we have the following inequality between its two different norms:
	\begin{equation}\label{relationship between L2 and HS norm}
		\|A\| \leq \|A\|_{\sF}.
	\end{equation}
	For any Hilbert-Schmidt operator $A\in \mathcal{B}(\sH)$ and any bounded operator $B\in \mathcal{B}(\sH)$, the product operators $AB$ and $BA$ are also Hilbert-Schmidt operators satisfying
	\begin{align}\label{equation: HS norm of product operator}
		\left\|AB\right\|_{\sF}\leq \left\|A\right\|_{\sF}\left\|B\right\| \mbox{ and $\left\|BA\right\|_{\sF}\leq \left\|A\right\|_{\sF}\left\|B\right\|$}
	\end{align}
	For any $f\in\sH$ and $g\in\sH'$, the rank-one operator $f\otimes g$ is Hilbert-Schmidt with the Hilbert-Schmidt norm
	\begin{align}\label{equation: HS norm of rank-one operator}
		\left\|f\otimes g\right\|_{\sF} = \left\|f\right\|_{\sH}\left\|g\right\|_{\sH'}.
	\end{align} 
	An operator $A \in \mathcal{B}(\sH,\sH')$ is trace class if $\sum_{k\geq 1}\langle \left(A^*A\right)^{1/2}e_k,e_k \rangle_{\sH}<\infty$ for some (any) orthonormal basis $\{e_k\}_{k\geq 1}$ of $\sH$. All trace class operators constitute a Banach space endowed with the norm $\mbox{Tr}(A):= \sum_{k\geq 1}\langle \left(A^*A\right)^{1/2}e_k,e_k \rangle_{\sH}$. For any positive operator $A\in \mathcal{B}(\sH)$, we have
	\begin{equation}\label{equation: trace class}
		\mbox{Tr}(A)=\sum_{k\geq1}\langle Ae_k,e_k \rangle_\sH.
	\end{equation}
 Recall that ${\sL}^2(\sT)$ is the Hilbert space of real functions on $\sT$ square-integrable with respect to the Lebesgue measure. We denote the corresponding norm of ${\sL}^2(\sT)$ induced by the inner product $\langle f, g\rangle_{{\sL}^2}=\int_{\sT} f(t) g(t) dt$ by $\|\cdot\|_{{\sL}^2}$. And we denote by $\cI$ the identity operator on $\sL^2(\sT)$.
 
 Without loss of generality, we assume the functional covariate $X$ satisfy $\EE\left[X\right] = 0$ and $\EE\left[\left\|X\right\|_{\sL^2}^2\right]<\infty$. Then the covariance kernel $\cC:\sT\times \sT \rightarrow \mathbb{R}$, given by $\cC(s,t) := \EE\left[X(s)X(t)\right], \forall s,t\in \sT$, defines a compact and non-negative operator $\cL_\cC:\sL^2(\sT)\rightarrow \sL^2(\sT)$ through
 \begin{align*}
 	\cL_\cC(f)(t)=\int_\mathscr{T} \cC(s,t)f(s)ds,\quad\forall f\in {\sL}^2(\mathscr{T})\mbox{ and } \forall t\in\mathscr{T}.
 \end{align*}
\subsection{Reproducing Kernel Hilbert Space}\label{subsection: RKHS}
Consider a Hilbert space $\sH\subset \sL^2(\sT)$ endowed with the inner product $\langle\cdot,\cdot\rangle_{\sH}$. We say that $\sH$ is a reproducing kernel Hilbert space (RKHS), if and only if there exists a bivariate function $\cK:\sT\times\sT \rightarrow \mathbb{R}$ which is called as the reproducing kernel associated to $\sH$, such that for any $t\in \sT$ and $f\in \sH$,
\begin{align}\label{reproducing property}
	\cK(\cdot,t) \in \sH \quad\mbox{and}\quad \langle f, \cK(\cdot,t) \rangle_{\sH} = f(t).
\end{align}
 To emphasize the relationship between the RKHS $\sH$ and its reproducing kernel $\cK$, we rewrite the RKHS as $\sH_\cK$ and its equipped inner product as $\langle\cdot,\cdot\rangle_{\sH_\cK}$. The reproducing kernel $\cK$ is always symmetric and non-negative. 
 
 If the reproducing kernel $\cK$ is continuous, then $\cK$ can induce a compact, symmetric and non-negative integral operator $\cL_{\cK}:\sL^2(\sT)\rightarrow \sL^2(\sT)$ given by
 \begin{align*}
 	\cL_{\cK}(f)(t) = \int_{\sT} \cK(s,t)f(s)ds, \quad \forall t\in\sT, f\in \sL^2(\sT),
 \end{align*}
 and following from Mercer's theorem (see, e.g., \cite{hsing2015theoretical}), $\cK$ can be expressed as
 \begin{align*}
 	\cK(s,t) = \sum_{j=1}^{\infty} \lambda_j e_j(s)e_j(t), \quad\forall s,t\in \sT,
 \end{align*}
 where $\left\{\lambda_j\right\}_{j=1}^{\infty}$ is a non-increasing, non-negative sequences and $\left\{e_j\right\}_{j=1}^{\infty}$ is an orthonormal basis of $\sL^2(\sT)$. 
 
 We define the square root of $\cK$ as
 \begin{align*}
 	\cK^{1/2}(s,t) := \sum_{j=1}^{\infty} \sqrt{\lambda_j} e_j(s)e_j(t), \quad\forall s,t\in \sT.
 \end{align*}
 One can verify that $\cK^{1/2}\in \sL^2(\sT\times \sT)$ and $\cK^{1/2}$ satisfies the following equation:
 \begin{align*}
 	\cK(s,t) = \int_{\sT}\cK^{1/2}(u,s)\cK^{1/2}(u,t)du, \quad \forall s,t\in \sT.
 \end{align*}
Then we write
\begin{align}\label{equation: operator relationship}
	\cL_{\cK}^{1/2} = \cL_{\cK^{1/2}}, \quad\mbox{as}\quad \cL_{\cK^{1/2}}\cL_{\cK^{1/2}} = \cL_{\cK},
\end{align}
where $\cL_{\cK}^{1/2}$ denotes the $1/2$-th power of $\cL_{\cK}$ and $\cL_{\cK^{1/2}}$ is the integral operator induced by $\cK^{1/2}$. It is well known that
$\cL_{\cK}^{1/2}$ is compact and forms an isomorphism from $\overline{\sH_\cK}$, the closure of $\sH_\cK$ in $\sL^2(\sT)$, to the RKHS $\sH_\cK$. Thus, for any $f,g\in \overline{\sH_\cK}$, there holds
\begin{align}\label{equation: norm and inner product relationship}
	\cL_{\cK}^{1/2}f,\cL_{\cK}^{1/2}g\in \sH_\cK, \mbox{ }\left\|\cL_{\cK}^{1/2}f\right\|_{\sH_\cK}=\left\|f\right\|_{\sL^2} \mbox{ and } \left\langle\cL_{\cK}^{1/2}f,\cL_{\cK}^{1/2}g \right\rangle_{\sH_\cK} = \left\langle f,g \right\rangle_{\sL^2}.
\end{align}

 \subsection{Unanchored Sobolev Spaces}\label{subsection: sobolev space}

In this subsection, we provide the definition of unanchored Sobolev spaces and review several relevant results from previous studies.

 Recall that, the unanchored Sobolev space $\sW^{\alpha,2}(\sT)$ for a positive integer $\alpha\geq 1$ is defined by \eqref{sobolev} with the corresponding inner product \eqref{innerproduct}. For non-integer $\alpha>0$, the unanchored Sobolev space $\sW^{\alpha,2}(\sT)$ can be defined via the real interpolation. To this end, we introduce the following definition of real interpolation based on K-functional (see, e.g., \cite{sawano2018theory,tartar2007introduction}).
\begin{definition}[real interpolation]\label{definition: real interpolation}
	Let $(\sH_0, \left\|\cdot\right\|_{\sH_0})$ and $(\sH_1,\left\|\cdot\right\|_{\sH_1})$ be two normed spaces. For any element $a\in \sH_0+\sH_1$ and $t>0$, we define the K-functional
	\begin{align*}
		K(t;a) = \inf_{a = a_0+a_1}\Big\{\left\|a_0\right\|_{\sH_0} + t\left\|a_1\right\|_{\sH_1}\Big\}.
	\end{align*}
   For any $0<\eta<1$ and any $1\leq q \leq \infty$ (or for $\eta = 0, 1$ with $q =\infty$), we define the real interpolation space
   \begin{align*}
   	(\sH_0,\sH_1)_{\eta, q} := \left\{ a\in \sH_0+\sH_1\mid t^{-\eta}K(t;a)\in \sL^q\left(\mathbb{R}_{+},\frac{dt}{t}\right)\right\}
   \end{align*}
   with the norm $$\left\|a\right\|_{(\sH_0,\sH_1)_{\eta,q}} = \left\|t^{-\eta}K(t;a)\right\|_{\sL^q\left(\mathbb{R}_{+},\frac{dt}{t}\right)} = \left(\int_{0}^\infty \Big\{t^{-\eta }K(t;a)\Big\}^q\frac{dt}{t}\right)^{\frac{1}{q}}, \quad \forall a\in (\sH_0,\sH_1)_{\eta, q}.$$
\end{definition}

Using Definition \ref{definition: real interpolation}, the unanchored Sobolev space $\sW^{\alpha,2}(\sT)$ for any real number $\alpha>0$ can be defined as:
\begin{align*}
	\sW^{\alpha,2}(\sT) = \left(\sL^2(\sT), \sW^{\lceil \alpha\rceil,2}(\sT)\right)_{\frac{\alpha}{\lceil\alpha\rceil},2}.
\end{align*}
It is well known that for any $\alpha>1/2$, the unanchored Sobolev space $\sW^{\alpha,2}(\sT)$ is a reproducing kernel Hilbert space with a continuous reproducing kernel (see, e.g., \cite{sarazin2023new,wang2022functional}) and the following continuous embedding condition holds (see, e.g., \cite{adams2003sobolev}):
\begin{align}\label{equation: continuous embedding condition}
	\sW^{\alpha,2}(\sT) \hookrightarrow \sC^{0,\alpha-1/2}(\sT) \hookrightarrow \sL^{\infty}(\sT),
\end{align}
where $\hookrightarrow$ represents the continuous embedding, $\sC^{0,\alpha-1/2}$ denotes the H\"older space of order $\alpha -1/2$ and $\sL^{\infty}$ represents the bounded function space.

In particular, for a positive integer $\alpha > 0$, the reproducing kernel of $\sW^{\alpha,2}(\sT)$ can be explicitly expressed as
$$
\cK_{\alpha}(t,t'):=\sum_{k=0}^{\alpha} \frac{B_k(t) B_k(t')}{(k !)^2}+\frac{(-1)^{\alpha+1}}{(2 \alpha)!} B_{2 \alpha}(|t-t'|),\quad \forall t,t'\in \sT,
$$
where $\left(B_k\right)_{k \geq 0}$ denotes the sequence of Bernoulli polynomials.

 In the rest part of this paper, we fix a parameter $\alpha^*>1/2$. Consequently, the reproducing property \eqref{reproducing property} and isometric isomorphic property \eqref{equation: norm and inner product relationship} hold for the unanchored Sobolev space $\sW^{\alpha^*,2}(\sT)$ and its corresponding Sobolev kernel $\cK_{\alpha^*}$.

\subsection{Filter Functions}\label{subsection: filter functions}

Noting that $\cL_{\cK_{\alpha^*}}^{1/2}$ is a bounded operator for a fixed parameter $\alpha^*>1/2$ and that $\EE\left[\left\|X\right\|_{\sL^2}^2\right]<\infty$, there exists a constant $\rho_{\alpha^*}>0$, such that  
\begin{align}\label{equation: bound of expectation of L_KX}
	\EE\left[\left\|\cL_{\cK_{\alpha^*}}^{1/2}X\right\|_{\sL^2}^2\right]\leq \rho_{\alpha^*}^2<\infty. 
\end{align} 

We then introduce the following definition of filter functions.
\begin{definition}[filter functions]\label{definition: filter functions}
	Let $\nu_{\Psi} \geq 1$ be a constant and $\left\{\Psi_\lambda:[0,\infty) \rightarrow \mathbb{R}| \lambda\in (0,1)\right\}$ be a family of functions. We say $\left\{\Psi_\lambda:[0,\infty) \rightarrow \mathbb{R}| \lambda\in (0,1)\right\}$ is a filter function with qualification $\nu_{\Psi}$ if:
	\begin{enumerate}
		\item[(1)] There exists a constant $B>0$, such that
		\[\sup_{t\in\left[0, \frac{3}{2}\rho_{\alpha^*}+\frac{1}{2}\right]}\left|(\lambda+t)\Psi_\lambda(t)\right| \leq B,\quad \forall \lambda\in (0,1).\]
		\item[(2)] For any $0\leq\nu\leq \nu_{\Psi}$, there exists a constant $F_\nu>0$ only depending on $\nu$, such that
		\[\sup_{t\in\left[0, \frac{3}{2}\rho_{\alpha^*}+\frac{1}{2}\right]}\left|1-t\Psi_\lambda(t)\right|t^{\nu}\leq F_{\nu}\lambda^{\nu}, \quad \forall \lambda \in (0,1).\]
		\item[(3)] There exists a constant $D>0$, such that
		\begin{align*}
			\sup_{t\in\left(\frac{3}{2}\rho_{\alpha^*}+\frac{1}{2}, \infty\right)}\left|(\lambda+t)\Psi_\lambda(t)\right| \leq D,\quad \forall \lambda\in (0,1).
		\end{align*}
	\end{enumerate}
	For simplicity of notations, we denote $E:=\max\{B,D\}$.
\end{definition}

In the previous studies, filter functions are typically defined over bounded intervals of $t$ without property (3) (see, e.g., \cite{bauer2007regularization,lu2020balancing}). This is primarily because such studies often assume that empirical operators (such as $\cT_{\alpha^*,\bx}$ in this paper) are bounded almost everywhere. However, this assumption is not appropriate in the context of functional linear regression, as satisfying it would require that $\left\|\cL_{\cK_{\alpha^*}}^{1/2}X\right\|_{\sL^2}<\infty$ almost everywhere. This condition, however, excludes the most common case where $X$ is a Gaussian random variable taking values in $\sL^2(\sT)$. To avoid imposing such restrictive assumptions and to include the most common case in our analysis, we define the filter functions over the interval $[0,\infty)$ and introduce property (3).

There is a wide variety of algorithms for functional linear regression that satisfy Definition \ref{definition: filter functions}, and we list only a few typical examples below.
\begin{example}[Tikhonov regularization]\label{example: filter function1}
		We begin with the Tikhonov regularization algorithm. This algorithm constructs operators in the unanchored Sobolev space $\sW^{\alpha^*,2}(\sT)$ based on the discretely observed data $S$ through
		\begin{align}\label{Tikhonov regularization algorithm}
				\hat{\beta}^{TR}_{S,\alpha^*,\lambda}:=\mathop{\mathrm{argmin}}_{\beta \in \sW^{\alpha^*,2}(\sT)}\left\{\frac{1}{N}\sum_{i=1}^N \left(Y_i-\sum_{k=1}^{m}\left(r_{k+1}-r_{k}\right)\beta(r_k)X_i(r_k)\right)^2+\lambda \|\beta\|^2_{\sW^{\alpha^*,2}}\right\}.
		\end{align}
Using the the isometric isomorphism property of $\cL_{\cK_{\alpha^*}}^{1/2}$, the estimator $\hat{\beta}^{TR}_{S,\alpha^*,\lambda}$ can also be expressed as $\hat{\beta}^{TR}_{S,\alpha^*,\lambda} = \cL_{\cK_{\alpha^*}}^{1/2}\hat{f}^{TR}_{S,\alpha^*,\lambda}$ with
\begin{align*}
		\hat{f}^{TR}_{S,\alpha^*,\lambda}:=\mathop{\mathrm{argmin}}_{f \in \sL^2(\sT)}\left\{\frac{1}{N}\sum_{i=1}^N \left(Y_i-\sum_{k=1}^{m}\left(r_{k+1}-r_{k}\right)\left(\cL_{\cK_{\alpha^*}}^{1/2}f\right)(r_k)X_i(r_k)\right)^2+\lambda \|f\|^2_{\sL^2}\right\}.
\end{align*}
Then following from Theorem 6.2.1 in \cite{hsing2015theoretical}, we can solve $\hat{f}^{TR}_{S,\alpha^*,\lambda}$ explicitly as
\begin{align*}
	\hat{f}^{TR}_{S,\alpha^*,\lambda} = \left(\lambda \cI + \cT_{\alpha^*,\bx}\right)^{-1}\frac{1}{N}\cG_{{\alpha^*},\bx}^{*}\by,
\end{align*}
where $\cI$ denotes the identity operator on $\sL^2(\sT)$, $\cT_{\alpha^*,\bx}$ and $\cG_{{\alpha^*},\bx}^{*}$ are defined in Section \ref{section: introduction}.

Therefore, we write
\begin{align*}
	\hat{\beta}^{TR}_{S,\alpha^*,\lambda} = \cL_{\cK_{\alpha^*}}^{1/2}\hat{f}^{TR}_{S,\alpha^*,\lambda} = \Psi_\lambda^{TR}(\cT_{\alpha^*,\bx})\frac{1}{N}\cG_{{\alpha^*},\bx}^{*}\by,
\end{align*}
where $\Psi^{TR}_\lambda(t) = \left(\lambda + t\right)^{-1}, \forall t\in [0,\infty), \forall\lambda\in (0,1)$. It is easy to show that $\left\{\Psi_\lambda^{TR}\mid \lambda \in (0,1) \right\}$ satisfies Definition \ref{definition: filter functions} with $\nu_{\Psi} = 1$, $F_\nu = 1$ for any $0\leq \nu \leq \nu_{\Psi}=1$, and $B=D = 1$.
\end{example}
\begin{example}[iterated Tikhonov regularization]\label{example: filter function2}
The second example is an improved version of the Tikhonov regularization algorithm, named the iterated Tikhonov regularization algorithm. Let integer $s\geq 1$ be the total number of iterations. The $r$-th ($1\leq r\leq s$) iteration of iterated Tikhonov regularization algorithm establishes estimators through
\begin{align*}
	\hat{\beta}^{ITR,s}_{S,\alpha^*,\lambda,r}:=\mathop{\mathrm{argmin}}_{\beta \in \sW^{\alpha^*,2}(\sT)}\left\{\frac{1}{N}\sum_{i=1}^N \left(Y_i-\sum_{k=1}^{m}\left(r_{k+1}-r_{k}\right)\beta(r_k)X_i(r_k)\right)^2+\lambda \|\beta - \hat{\beta}^{ITR,s}_{S,\alpha^*,\lambda,r-1}\|^2_{\sW^{\alpha^*,2}}\right\},
\end{align*}
where $\hat{\beta}^{ITR,s}_{S,\alpha^*,\lambda, r-1}$ is the estimator given by the $r-1$-th iteration of iterated Tikhonov regularization algorithm and $\hat{\beta}^{ITR,s}_{S,\alpha^*,\lambda, 0} = 0$.

Following from the same arguments in Example \ref{example: filter function1}, we have $\hat{\beta}^{ITR,s}_{S,\alpha^*,\lambda,r} = \cL_{\cK_{\alpha^*}}^{1/2}\hat{f}^{ITR,s}_{S,\alpha^*,\lambda,r}, \forall 1\leq r\leq s,$ with $\hat{f}_{S,\alpha^*,\lambda, 0} = 0$ and
\begin{align*}
	\hat{f}^{ITR,s}_{S,\alpha^*,\lambda,r} = \hat{f}^{ITR,s}_{S,\alpha^*,\lambda,r-1}+ \left(\lambda \cI + \cT_{\alpha^*,\bx}\right)^{-1}\frac{1}{N}\cG_{{\alpha^*},\bx}^{*}\Big(\by-\cG_{{\alpha^*},\bx}\left(\hat{f}^{ITR,s}_{S,\alpha^*,\lambda,r-1}\right)\Big), \quad \forall 1\leq r \leq s.
\end{align*}
Then noting that $\cT_{\alpha^*,\bx} = \frac{1}{N}\cG_{{\alpha^*},\bx}^{*}\cG_{{\alpha^*},\bx}$, we can explicitly solve the final estimator of iterated Tikhonov regularization algorithm as
\begin{align*}
	\hat{\beta}^{ITR,s}_{S,\alpha^*,\lambda,s} =\cL_{\cK_{\alpha^*}}^{1/2}\hat{f}^{ITR,s}_{S,\alpha^*,\lambda,s}= \cL_{\cK_{\alpha^*}}^{1/2}\sum_{k=1}^{s} \left(\lambda \cI + \cT_{\alpha^*,\bx}\right)^{-k}\lambda^{k-1}\frac{1}{N}\cG_{{\alpha^*},\bx}^{*}\by.
\end{align*}
Therefore, we write
\begin{align*}
	\hat{\beta}^{ITR,s}_{S,\alpha^*,\lambda,s} = \cL_{\cK_{\alpha^*}}^{1/2}\Psi_\lambda^{ITR,s}(\cT_{\alpha^*,\bx})\frac{1}{N}\cG_{{\alpha^*},\bx}^{*}\by,
\end{align*}
where 
\begin{align*}
	\Psi_\lambda^{ITR,s}(t) = \sum_{k=1}^{s}(\lambda + t)^{-k}\lambda^{k-1} = \frac{(\lambda + t)^{s} -\lambda^{s}}{t(\lambda+t)^s}, \quad \forall t\in [0,\infty), \forall \lambda \in (0,1).
\end{align*}
One can verify that $\left\{\Psi_\lambda^{ITR,s} \mid \lambda \in (0,1)\right\}$ satisfies Definition \ref{definition: filter functions} with $\nu_{\Psi} = s$, $F_\nu = 1$ for any $0\leq \nu\leq \nu_{\Psi} = s$, and $B= D= s$.
\end{example}
\begin{example}[gradient flow]\label{example: filter function3}
	Let 
	\begin{align*}
		\mathcal{E}_{S}(\beta) = \frac{1}{2N}\sum_{i=1}^N \left(Y_i-\sum_{k=1}^{m}\left(r_{k+1}-r_{k}\right)\beta(r_k)X_i(r_k)\right)^2, \quad \forall \beta\in \sW^{\alpha^*,2}(\sT),
	\end{align*}
be the empirical loss.

The gradient flow algorithm constructs estimators by solving the gradient flow equation:
\begin{align*}
	\frac{d\hat{\beta}^{GF}_r}{dt} = -\nabla\mathcal{E}_{S}\left(\hat{\beta}^{GF}_{r}\right), \forall r\geq 0, \qquad \hat{\beta}^{GF}_0 = 0,
\end{align*}
where $\nabla\mathcal{E}_{S}\left(\hat{\beta}^{GF}_{r}\right)$ denotes the gradient of $\mathcal{E}_{S}(\beta)$ for $\beta = \hat{\beta}^{GF}_r$ (see, e.g., \cite{yao2007early}). Following from the isometric isomorphic property of $\cL_{\cK_{\alpha^*}}^{1/2}$, we write $\hat{\beta}^{GF}_r = \cL_{\cK_{\alpha^*}}^{1/2}\hat{f}^{GF}_r, \forall r\geq 0$ with $\hat{f}^{GF}_0 =0$. Then imitating the proof of Proposition 2.2 of \cite{yao2007early} and using the reproducing property of $\sW^{\alpha^*,2}(\sT)$, we have
\begin{align*}
	\nabla\mathcal{E}_{S}\left(\hat{\beta}^{GF}_r\right) 
	= \cL_{\cK_{\alpha^*}}^{1/2}\frac{1}{N}\cG_{{\alpha^*},\bx}^{*}\left(\cG_{{\alpha^*},\bx}\left(\hat{f}^{GF}_r\right)-\by\right), \quad \forall t>0.
\end{align*}
Let $\lambda = 1/r$ be the regularization parameter. Then noting that $\cT_{\alpha^*,\bx} = \frac{1}{N}\cG_{{\alpha^*},\bx}^{*}\cG_{{\alpha^*},\bx}$, we can solve the gradient flow equation in closed-form as
\begin{align*}
	\hat{\beta}^{GF}_r = \cL_{\cK_{\alpha^*}}^{1/2} \hat{f}^{GF}_r = \cL_{\cK_{\alpha^*}}^{1/2}\Psi_\lambda^{GF}(\cT_{\alpha^*,\bx})\frac{1}{N}\cG_{{\alpha^*},\bx}^{*}\by, \quad \forall r>1,
\end{align*}
where 
\[
\Psi_\lambda^{GF}(t) =\left\{
\begin{aligned}
   &\frac{1- e^{-rt}}{t} = \frac{1- e^{-t/\lambda}}{t}, \qquad&\forall t>0,\forall\lambda \in (0,1),\\
	& \frac{1}{\lambda},  &t=0,\forall\lambda \in (0,1).
\end{aligned} \right.\]
One can verify that $\left\{\Psi_\lambda^{GF}\mid \lambda\in(0,1)\right\}$ satisfies Definition \ref{definition: filter functions} with $\nu_{\Psi} = \infty$, $F_\nu = (\nu/e)^{\nu}$ for any $0\leq \nu<\infty$, and $B=D = 1$.
\end{example}

\subsection{Distributed Spectral Regularization Algorithms}

In this subsection, we introduce notations used in the distributed spectral regularization algorithms. 

We denote the sample set \[\bx:=\left\{\left(X_i(r_1),X_i(r_2),\cdots,X_i(r_m),X_i(r_{m+1})\right)\right\}_{i=1}^N\]
which consists of the discrete samples of $X$ in $S$.
Then we define an empirical operator $\cG_{\alpha^*,\bx}: \sL^2(\sT)\rightarrow \mathbb{R}^N$ based on $\bx$ as
\begin{align*}
	&\cG_{{\alpha^*},\bx}(f)\\
	:=&\left(\sum_{k=1}^{m}(r_{k+1}-r_k)\left\langle f,\cK^{1/2}_{{\alpha^*}}(r_k)\right\rangle_{\sL^2}X_1(r_k),\cdots, \sum_{k=1}^{m}(r_{k+1}-r_k)\left\langle f,\cK^{1/2}_{{\alpha^*}}(r_k)\right\rangle_{\sL^2}X_N(r_k)\right)^{T}
\end{align*}
for any $f\in \sL^{2}(\sT)$. The adjoint operator of $\cG_{{\alpha^*},\bx}$, denote by $\cG_{{\alpha^*},\bx}^*:\mathbb{R}^N\rightarrow \sL^2(\sT)$, is defined as
\[\cG_{{\alpha^*},\bx}^*(a):= \sum_{i=1}^{N}\sum_{k=1}^{m}a_i(r_{k+1}-r_k)\cK_{{\alpha^*}}^{1/2}(\cdot,r_k)X_i(r_k), \quad \forall a\in \mathbb{R}^N.\]
And we define the empirical operator $\cT_{{\alpha^*},\bx}:\sL^2(\sT)\rightarrow\sL^2(\sT)$ as
\begin{align*}
	\cT_{{\alpha^*},\bx}&:= \frac{1}{N}\cG_{{\alpha^*},\bx}^*\cG_{{\alpha^*},\bx}\\
	&= \frac{1}{N}\sum_{i=1}^N\left[\sum_{k=1}^{m}(r_{k+1}-r_k)\cK_{{\alpha^*}}^{1/2}(\cdot,r_k)X_i(r_k)\right]\otimes\left[\sum_{k=1}^{m}(r_{k+1}-r_k)\cK_{{\alpha^*}}^{1/2}(\cdot,r_k)X_i(r_k)\right].
\end{align*}
For the sake of simplicity, we define
\[\mathcal{S}_i(\cK_{\alpha^*}, \bx):= \sum_{k=1}^{m}(r_{k+1}-r_k)\cK_{{\alpha^*}}^{1/2}(\cdot,r_k)X_i(r_k), \quad\forall i=1,2,\cdots,N.\]
Thus, we express $\cT_{{\alpha^*},\bx}$ as
\[\cT_{{\alpha^*},\bx} = \frac{1}{N}\sum_{i=1}^N\mathcal{S}_i(\cK_{\alpha^*},\bx)\otimes\mathcal{S}_i(\cK_{\alpha^*},\bx).\]
The spectral regularization estimator based on the unanchored Sobolev space $\sW^{\alpha^*,2}(\sT)$ (with Sobolev kernel $\cK_{\alpha^*}$) and a filter function $\left\{\Psi_\lambda:[0,\infty) \rightarrow \mathbb{R}| \lambda\in (0,1)\right\}$ is defined as
\[\hat{\beta}_{S,\alpha^*,\Psi_\lambda} = \cL_{\cK_{\alpha^*}}^{1/2}\hat{f}_{S,\alpha^*,\Psi_\lambda}= \cL_{\cK_{\alpha^*}}^{1/2}\Psi_\lambda\left(T_{\alpha^*,\bx}\right)\frac{1}{N}G_{\alpha^*,\bx}^{*}\by,\]
where $\by= (Y_1,Y_2,\cdots,Y_N)^{T}\in \mathbb{R}^{N}$.

Recall that $S=\cup_{j=1}^M S_j$ with $S_j \cap S_k = \emptyset$ for $j\neq k$ and $|S_j|=\frac{N}{M}$. For any $1\leq j\leq M$, we denote the local sample sets as
\begin{align*}
	\bx_j:=\left\{\left(X_i(r_1),\cdots,X_i(r_{m+1})\right): \left(X_i(r_1),\cdots,X_i(r_{m+1}),Y_i\right)\in S_j\right\}
\end{align*}

and 
\begin{align*}
	\by_j:=\left\{Y_i: \left(X_i(r_1),\cdots,X_i(r_{m+1}),Y_i\right)\in S_j\right\}.
\end{align*}
Using these notations, the local spectral regularization estimator on each subset $S_j$ can be computed as $\hat{\beta}_{S_j,{\alpha^*},\Psi_\lambda}= \cL_{\cK_{\alpha^*}}^{1/2}\hat{f}_{S_j,\alpha^*,\Psi_\lambda}$ with 
\begin{align}\label{fhatlocal}
	\hat{f}_{S_j,\alpha^*,\Psi_\lambda}=\Psi_\lambda\left(\cT_{{\alpha^*},\bx_j}\right)\frac{1}{|S_j|}\cG_{{\alpha^*},\bx_j}^* \by_j,
\end{align} 
where the local empirical operator $\cT_{{\alpha^*},\bx_j}:\sL^2(\sT)\rightarrow \sL^2(\sT)$ is defined as
\[
\cT_{{\alpha^*},\bx_j}:= \frac{1}{|S_j|}\sum_{i:\left(X_i(r_1),\cdots,X_i(r_{m+1})\right)\in\bx_j}\mathcal{S}_{i}(\cK_{\alpha^*}, \bx)\otimes\mathcal{S}_{i}(\cK_{\alpha^*},\bx),
\]
and the local empirical operator $\cG_{{\alpha^*},\bx_j}^*:\mathbb{R}^{|S_j|}\rightarrow\sL^2(\sT)$ is defined as
\begin{align*}
	\cG_{{\alpha^*},\bx_j}^*(a)&:=\sum_{i:\left(X_i(r_1),\cdots,X_i(r_{m+1})\right)\in\bx_j}a_i\mathcal{S}_{i}(\cK_{\alpha^*},\bx), \quad\forall a\in \mathbb{R}^{|\mathcal{S}_j|}.
\end{align*}
Then the distributed spectral regularization estimator $\overline{\beta}_{S,{\alpha^*},\Psi_\lambda}$ can be computed as $\overline{\beta}_{S,{\alpha^*},\Psi_\lambda}=\cL^{1/2}_{\cK_{\alpha^*}}\overline{f}_{S,\alpha^*,\Psi_\lambda}$ with 
\begin{align}\label{fhataveraged}
	\overline{f}_{S,\alpha^*,\Psi_\lambda}=\frac{1}{M}\sum_{j=1}^M \hat{f}_{S_j,\alpha^*,\Psi_\lambda}.
\end{align}

\section{Main Results}\label{section: main results}

In this section, we first introduce main assumptions of our paper. Then based on these assumptions, we present our theoretical lower and upper bounds on the estimation error of the distributed spectral regularization estimator \eqref{finalestimator}.

\subsection{Assumptions}\label{subsection: assumptions}

 To establish the optimal upper and lower bounds for the estimation errors, we need to impose some mild assumptions on the slope function $\beta_0$, the functional covariate $X$, the random noise $\epsilon$ and the sampling scheme. We begin with the regularity condition of the slope function $\beta_0$. To this end, we define operators $\cT_{{\alpha^*}}:=\cL_{\cK_{\alpha^*}}^{1/2} \cL_\cC \cL_{\cK_{\alpha^*}}^{1/2}$ and $\cT_{{\alpha^*},\dagger}:=\cL_\cC^{1/2}\cL_{\cK_{\alpha^*}}\cL_\cC^{1/2}$. Note that 
\begin{equation*}
	\cT_{{\alpha^*}}= \cL_{\cK_{\alpha^*}}^{1/2}\cL_\cC^{1/2}\left(\cL_{\cK_{\alpha^*}}^{1/2}\cL_\cC^{1/2}\right)^{*} \mbox{ and } \cT_{{\alpha^*},\dagger} = \left(\cL_{\cK_{\alpha^*}}^{1/2}\cL_\cC^{1/2}\right)^{*}\cL_{\cK_{\alpha^*}}^{1/2}\cL_\cC^{1/2}.
\end{equation*} 
It is obvious that $\cL_{\cK_{\alpha^*}}^{1/2}\cL_\cC^{1/2}$, $\cL_\cC^{1/2}\cL_{\cK_{\alpha^*}}^{1/2}$, $\cT_{{\alpha^*}}$ and $\cT_{{\alpha^*},\dagger}$ are all compact. Then according to the singular value decomposition theorem (see, e.g., Theorem 4.3.1 in \cite{hsing2015theoretical}), we have the following expansions:
\begin{equation}\label{singular value decomposition}
	\begin{aligned}
		\cT_{{\alpha^*}}&= \sum_{j\geq 1}\mu_{\alpha^*,j} \phi_{\alpha^*,j} \otimes \phi_{\alpha^*,j},\\
		\cT_{{\alpha^*},\dagger}&= \sum_{j\geq 1} \mu_{\alpha^*,j} \varphi_{\alpha^*,j} \otimes \varphi_{\alpha^*,j},\\
		\cL_{\cK_{\alpha^*}}^{1/2}\cL_\cC^{1/2} &= \sum_{j\geq 1}\sqrt{\mu_{\alpha^*,j}} \varphi_{\alpha^*,j} \otimes \phi_{\alpha^*,j},\\
		\cL_\cC^{1/2}\cL_{\cK_{\alpha^*}}^{1/2} &= \sum_{j\geq 1}\sqrt{\mu_{\alpha^*,j}} \phi_{\alpha^*,j} \otimes \varphi_{\alpha^*,j},
	\end{aligned}
\end{equation} 
where $\{\mu_{\alpha^*,j}\}_{j\geq 1}$ is a positive and decreasing sequence, $\{\phi_{\alpha^*,j}\}_{j\geq 1}$ and $\{\varphi_{\alpha^*,j}\}_{j\geq 1}$ are two orthonormal sets of $\sL^2(\mathscr{T})$. Without loss of generality, we assume that $\mbox{Ker}( \cT_{\alpha^*}) = \mbox{Ker}( \cT_{\alpha^*, \dagger}) = \{0\}$. Under this assumption, $\{\phi_{\alpha^*,j}\}_{j\geq 1}$ and $\{\varphi_{\alpha^*,j}\}_{j\geq 1}$ are two orthonormal bases of $\sL^2(\mathscr{T})$. In other cases, similar results can be obtained by following the same proof procedure as outlined in our paper, only more tedious. We additionally assume that the sequence $\{\mu_{\alpha^*,j}\}_{j\geq 1}$ is summable.
\begin{assumption}[regularity condition of slope function]\label{assumption: regularity condition}
The slope function $\beta_0$ in functional linear regression model \eqref{LFRmodel} satisfies $\beta_0=\cL_{\cK_{\alpha^*}}^{1/2}f_0$ with
	\begin{equation}\label{equation: regularity condition}
		f_0=\cT_{{\alpha^*}}^{\theta}g_0 \mbox{ for some $0\leq\theta<\infty$ and $g_0\in\sL^2(\sT)$}.
	\end{equation} 
\end{assumption}
According to the isometric isomorphism property of $\cL_{\cK_{\alpha^*}}^{1/2}$, Assumption \ref{assumption: regularity condition} implies that $\beta_0\in \sW^{{\alpha^*},2}(\sT)$ for whatever $0\leq\theta\leq 1/2$. Furthermore, denote by $\cL_{\cK_{\alpha^*}}^{-1/2}$ the inverse operator of $\cL_{\cK_{\alpha^*}}^{1/2}$, Assumption \ref{assumption: regularity condition} implies that $\cL_{\cK_{\alpha^*}}^{-1/2}\beta_0$ belongs to the range space of $\cT_{{\alpha^*}}^{\theta}$ expressed as
\begin{align*}
	\mathrm{ran}\cT_{{\alpha^*}}^{\theta}=:\left\{f\in \sL^{2}(\sT): \sum_{j\geq 1}\frac{\langle f,\phi_{\alpha^*,j}\rangle_{\sL^2}^2}{\mu_{\alpha^*,j}^{2\theta}}<\infty\right\},
\end{align*}
where $\{\left(\mu_{\alpha^*,j},\phi_{\alpha^*,j}\right)\}_{j=1}^\infty$ is given by the singular value decomposition of $\cT_{{\alpha^*}}$ in \eqref{singular value decomposition}. Then there holds $\mathrm{ran}\cT_{{\alpha^*}}^{\theta_1} \subseteq \mathrm{ran}\cT_{{\alpha^*}}^{\theta_2}$ as $\theta_1 \geq \theta_2$. The regularity of functions in $\mathrm{ran}\cT_{{\alpha^*}}^{\theta}$ is measured by the decay rate of its expansion coefficients in terms of $\{\phi_{\alpha^*,j}\}_{j \geq 1}$. Condition \eqref{equation: regularity condition} means that $\langle \cL_{\cK_{\alpha^*}}^{-1/2}\beta_0, \phi_{\alpha^*,j}\rangle^2_{{\sL}^2}$ decays faster than the $2\theta-$th power of the eigenvalues of $\cT_{{\alpha^*}}$. Larger parameter $\theta$ will result in faster decay rates, and thus indicate higher regularity of $\beta_0$. We will discuss Assumption \ref{assumption: regularity condition} further in Section \ref{section: comparison}.

We impose the following assumption on random noise.
\begin{assumption}[noise condition]\label{assumption: epsilon2}
	The random noise $\epsilon$ in the functional linear regression model \eqref{LFRmodel} is independent of $X$ satisfying $\EE[\epsilon]=0$ and $\EE[\epsilon^2]\leq \sigma^2$.
\end{assumption}
We also need to impose the following assumption on the algorithm setting.
\begin{assumption}[sampling scheme]\label{assumption: sampling scheme} The discrete sample points $\left\{r_k\right\}_{k=1}^{m+1}$ in algorithms \eqref{spectral regularization estimator} and \eqref{finalestimator} satisfy $r_1<\cdots<r_{m+1}$, $r_1=0$ and $r_{m+1}=1$ for some integer $m\geq 1$, and there exists a constant $C_d$ such that $r_{k+1}-r_{k}\leq \frac{C_d}{m}$ for any $1\leq k\leq m$.
\end{assumption}
Assumption \ref{assumption: sampling scheme} ensures our sampling scheme closely approximates equally-spaced sampling, accommodating both random and fixed-point sampling schemes. A noteworthy example that satisfies this assumption is as follows: suppose the sample points are generated randomly from a distribution with a density function $\omega:[0,1]\rightarrow \mathbb{R}$ such that $\min_{s\in[0,1]}\omega(s)>0$. In this case, Assumption \ref{assumption: sampling scheme} holds with high probability (see, for instance, \cite{wang2014falling}).

The different theoretical upper bounds for the estimation error given by \eqref{estimation error} in our paper are based on the following two different regularity conditions of the functional covariate $X$, respectively.
\begin{assumption}[regularity condition of functional covariate \uppercase\expandafter{\romannumeral1}]\label{assumption: X1}
	There exists a constant $\rho>0$, such that for any $f\in \sL^2(\sT)$,
	\begin{align}\label{equation: X1}
		\EE\left[\left\langle X,f\right\rangle^4_{\sL^2}\right]\leq \rho\left[\EE\left\langle X,f
		\right\rangle_{\sL^2}^2\right]^2.
	\end{align}
	Moreover, there exists a constant $\kappa>0$ such that
	\begin{align}\label{equation: X1 additional}
		\EE\left[\|X\|_{\sW^{{\alpha^*},2}}^2\right]\leq\kappa^2
	\end{align}
\end{assumption}
Condition \eqref{equation: X1} has been introduced in \cite{cai2012minimax,yuan2010reproducing} showing that the linear functionals of $X$ have bounded kurtosis. In particular, one can verify that condition \eqref{equation: X1} is satisfied with $\rho =3$ when $X$ is a Gaussian random variable in $\sW^{\alpha^*,2}(\sT)$.

\begin{assumption}[regularity condition of functional covariate \uppercase\expandafter{\romannumeral2}]\label{assumption: X2}
	$X$ is a centered Gaussian random variable in $\sW^{\alpha^*,2}(\sT)$ satisfying condition \eqref{equation: X1 additional} with $\kappa>0$.
\end{assumption}

As stated before, Assumption \ref{assumption: X2} is covered by Assumption \ref{assumption: X1}, thus an enhanced version. One could relax Assumption \ref{assumption: X2} to the case that $X$ is a sub-Gaussian random variable without essentially changing the proof in this paper. Our proof only requires that the linear functionals of $X$ have bounded arbitrary-order moments. 

\subsection{Mini-max Lower Bounds}

Assumption \ref{assumption: regularity condition} and \ref{assumption: epsilon2} are sufficient to establish mini-max lower bounds for the estimation error. However, it is also necessary to assume that the eigenvalues $\{\mu_{\alpha^*,j}\}_{j\geq 1}$ of $\cT_{{\alpha^*}}$ (and $\cT_{{\alpha^*},\dagger}$) satisfy a polynomial decay condition.  To this end, for two positive sequences $\{a_j\}_{j\geq 1}$ and $\{b_j\}_{j\geq 1}$, we write $a_j \lesssim b_j$ if there exists a constant $c>0$ independent of $j$ such that $a_j \leq c b_j, \forall j\geq 1$. Additionally, we write $a_j \asymp b_j$ if and only if  $a_j \lesssim b_j$ and $b_j \lesssim a_j$. For convenience, we write $\beta_0 \in \mathrm{ran}\left\{\cL_{\cK_{\alpha^*}}^{1/2}\cT_{{\alpha^*}}^{\theta}\right\}$ in the scenarios where $\beta_0$ satisfies regularity condition \eqref{equation: regularity condition}. Similar lower bounds have been established by Theorem 4.4 of \cite{gupta2025optimal}.
\begin{theorem}[mini-max lower bound]\label{theorem: lower bound}
	Suppose that Assumption \ref{assumption: regularity condition} is satisfied with $0\leq \theta<\infty$, Assumption \ref{assumption: epsilon2} is satisfied with $\sigma>0$ and the eigenvalues $\{\mu_{\alpha^*,j}\}_{j=1}^\infty$ satisfy $\mu_{\alpha^*,j}\asymp j^{-1/p}$ for some $0< p \leq 1$. Then there holds
	\begin{equation}\label{equation: lower bound}
		\lim_{\gamma \to 0}\mathop{\lim\inf}_{N\to \infty} \inf_{\hat{\beta}_{\widetilde{S}}} \sup_{\beta_0} \mathbb{P} \left\{\left\|\hat{\beta}_{\widetilde{S}}-\beta_0\right\|_{\sW^{{\alpha^*},2}}^2\geq \gamma N^{-\frac{2\theta}{1+2\theta+p}}\right\} = 1,
	\end{equation} where the supremum is taken over all $\beta_0 \in \sW^{{\alpha^*},2}(\sT)$ satisfying $\beta_0 \in \mathrm{ran}\cL_{\cK_{\alpha^*}}^{1/2}\cT_{{\alpha^*}}^{\theta}$  and the infimum is taken over all possible predictors $\hat{\beta}_S \in \sW^{{\alpha^*},2}(\sT)$ based on the fully observed sample set $\widetilde{S} = \{(X_i,Y_i)\}_{i=1}^N$ consisting of $N$ independent copies of $(X,Y)$.
\end{theorem}

\subsection{Upper Bounds}

We next consider the upper bounds of estimation errors and show that the rate of lower bound established in Theorem \ref{theorem: lower bound} can be achieved by the spectral regularization estimator $\overline{\beta}_{S,{\alpha^*},\lambda}$ in \eqref{finalestimator}. 

  Under Assumption \ref{assumption: regularity condition}, \ref{assumption: epsilon2}, \ref{assumption: sampling scheme} and \ref{assumption: X1} and a polynomial decay condition of the eigenvalues $\{\mu_{\alpha^*,j}\}_{j\geq 1}$, we can establish the following theorem which provides the upper bound for the convergence rate of estimation error \eqref{estimation error}. To this end, we denote by $o(a_j)$ a little-o sequence of non-negative $\{a_j\}_{j\geq 1}$ as $\lim_{j\to \infty} o(a_j)/a_j=0$.
\begin{theorem}[upper bound \uppercase\expandafter{\romannumeral1}]\label{theorem: upper bound1}
	Let $\left\{\Psi_\lambda:[0,\infty) \rightarrow \mathbb{R}| \lambda\in (0,1)\right\}$ be a filter function satisfying Definition \ref{definition: filter functions} with qualification $\nu_{\Psi}\geq 1$. Suppose that Assumption \ref{assumption: regularity condition} is satisfied with $0\leq \theta\leq \nu_{\Psi}$, Assumption \ref{assumption: epsilon2}, \ref{assumption: sampling scheme} and \ref{assumption: X1} are satisfied and the eigenvalues $\{\mu_{\alpha^*,j}\}_{j=1}^\infty$ satisfy $\mu_{\alpha^*,j}\lesssim j^{-1/p}$ for some $0< p \leq1$. Then there holds
	\begin{align}
		\lim_{\Gamma \to 0}\mathop{\lim\sup}_{N\to \infty} \sup_{\beta_0} \mathbb{P} \left\{\left\|\overline{\beta}_{S,{\alpha^*},\Psi_\lambda}-\beta_0\right\|_{\sW^{{\alpha^*},2}}^2\geq \Gamma N^{-\frac{2\theta}{1+2\theta+p}}\right\} = 0,
	\end{align}
	provided that $\lambda = N^{-\frac{1}{1+2\theta+p}}$, $m\geq N^{\frac{2+2\theta}{(2{\alpha^*}-1)(1+2\theta+p)}}$ and $M\leq o\left(\min\left\{N^{\frac{2\theta}{1+2\theta+p}},N^{\frac{1+2\theta-p}{2(1+2\theta+p)}}\right\}\right)$, where the supremum is taken over all $\beta_0 \in \sW^{{\alpha^*},2}(\sT)$ satisfying $\beta_0 \in \mathrm{ran}\cL_{\cK_{\alpha^*}}^{1/2}\cT_{{\alpha^*}}^{\theta}$ with $0\leq\theta\leq 1/2$.
\end{theorem}

Under Assumption \ref{assumption: X2}, an enhanced version of Assumption \ref{assumption: X1}, and the other assumptions of Theorem \ref{theorem: upper bound1}, we can establish the following strong upper bound in expectation for the estimation error given by \eqref{estimation error}. 
\begin{theorem}[upper bound \uppercase\expandafter{\romannumeral2}]\label{theorem: upper bound2}
		Let $\left\{\Psi_\lambda:[0,\infty) \rightarrow \mathbb{R}| \lambda\in (0,1)\right\}$ be a filter function satisfying Definition \ref{definition: filter functions} with qualification $\nu_{\Psi}\geq 1$. Suppose that Assumption \ref{assumption: regularity condition} is satisfied with $0\leq \theta\leq \nu_{\Psi}$, Assumption \ref{assumption: epsilon2}, \ref{assumption: sampling scheme} and \ref{assumption: X2} are satisfied and the eigenvalues $\{\mu_{\alpha^*,j}\}_{j=1}^\infty$ satisfy $\mu_{\alpha^*,j}\lesssim j^{-1/p}$ for some $0< p \leq1$. Then there holds
	\begin{equation}\label{equation: upper bound2}
		\EE\left[\left\|\overline{\beta}_{S,{\alpha^*},\Psi_\lambda}-\beta_0\right\|_{\sW^{{\alpha^*},2}}^2\right]\lesssim N^{-\frac{2\theta}{1+2\theta+p}},
	\end{equation}
	provided that $\lambda=N^{-\frac{1}{1+2\theta+p}}$, $1/m\leq o\left(N^{-\frac{2+2\theta}{(2{\alpha^*}-1)(1+2\theta+p)}}\log^{-\frac{2}{2\alpha^*-1}}N\right)$ and $M\leq o\left(N^{\frac{1+2\theta-p}{1+2\theta+p}}\log^{-1}N\right)$.
\end{theorem}

\section{Discussions and Comparisons}\label{section: comparison}

In this section, we will first discuss Assumption \ref{assumption: regularity condition}, then compare our analysis with some related results, and finally present several directions for future research.

\subsection{Discussions on Assumption \ref{assumption: regularity condition}}
For any two bounded self-adjoint operators $A_1$ and $A_2$ on $\sL^2(\sT)$, we write $A_1 \preceq A_2$, if $A_2 - A_1$ is nonnegative, and $A_1 \succeq A_2$, if $A_1 - A_2$ is nonnegative.  Suppose that $\beta_0$ satisfies Assumption \ref{assumption: regularity condition} with $0\leq \theta<\infty$ and $ \cL_\cC \preceq \delta_1 \cL_{\cK_{\alpha^*}}^{r_1}$ for some $\delta_1>0$ and $r_1\geq 0$. Then according to Theorem 3 in \cite{chen2022online}, we have $\cT_{\alpha^*}^{\theta} = \left(\cL_{\cK_{\alpha^*}}^{1/2}\cL_\cC\cL_{\cK_{\alpha^*}}^{1/2}\right)^{\theta}\preceq \delta_1^{\theta}\cL_{\cK_{\alpha^*}}^{(1+r_1)\theta}$, and thus there exists $g_0^{*}\in \sL^2(\sT)$, such that
$\beta_0 = \cL_{\cK_{\alpha^*}}^{(1+r_1)\theta}(g_0^*)$. In reverse, suppose that $\cL_\cC \succeq \delta_2 \cL_{\cK_{\alpha^*}}^{r_2}$ and $\beta_0 = \delta_2^{\tau}\cL_{\cK_{\alpha^*}}^{(1+r_2)\tau}(g_0^*)$ for some $\delta_2>0$, $r_2\geq 0$ and $\tau\geq 0$. Then also from Theorem 3 in \cite{chen2022online}, we have $\cT_{\alpha^*}^{\tau}=\left(\cL_{\cK_{\alpha^*}}^{1/2}\cL_\cC\cL_{\cK_{\alpha^*}}^{1/2}\right)^{\tau}\succeq \delta_2^{\tau}\cL_{\cK_{\alpha^*}}^{(1+r_2)\tau}$, and thus there exists $g_0\in \sL^2(\sT)$ such that $\beta_0$ satisfies Assumption \ref{assumption: regularity condition} with $\theta = \tau$.

Our further discussion on Assumption \ref{assumption: regularity condition} relies on the interpolation space (or power space). Following from the singular value decomposition theorem, the compact and symmetric operator $\cL_{\cK_{\alpha^*}}$ can be expressed as
\begin{align*}
	\cL_{\cK_{\alpha^*}} = \sum_{j\geq 1} \lambda_{\alpha^*,j} e_{\alpha^*,j}\otimes e_{\alpha^*,j},
\end{align*} 
where $\left\{\lambda_{\alpha^*,j}\right\}_{j\geq 1}$ and $\left\{e_{\alpha^*,j}\right\}_{j\geq 1}$ are the eigenvalues and eigenfunctions.

We define the interpolation space (or power space) $\left[\sW^{\alpha^*,2}(\sT)\right]^{r}$ for any $0\leq r\leq 1$ as
\begin{align*}
	\left[\sW^{\alpha^*,2}(\sT)\right]^{r} := \mbox{Ran}\left\{\cL_{\cK_{\alpha^*}}^{r/2}\right\} = \left\{\sum_{j\geq 1} a_j\lambda_{\alpha^*,j}^{r/2}e_{\alpha^*,j}\Bigg| \sum_{j\geq 1}a_j^2<\infty \right\}.
\end{align*}
One can verify that for any $0 \leq r_1\leq r_2\leq 1$, the embedding $\left[\sW^{\alpha^*,2}(\sT)\right]^{r_2}\hookrightarrow \left[\sW^{\alpha^*,2}(\sT)\right]^{r_1}$ exists and is compact. Noting that $\sW^{\alpha^*,2}(\sT)$ is dense in $\sL^2(\sT)$ and recalling the isometric isomorphic property of $\cL_{\cK_{\alpha^*}}^{1/2}$, we have $\left[\sW^{\alpha^*,2}(\sT)\right]^{0} = \sL^2(\sT)$ and $\left[\sW^{\alpha^*,2}(\sT)\right]^{1} = \sW^{\alpha^*,2}(\sT)$. Besides, Theorem 4.6 of \cite{steinwart2012mercer} shows that for any $0<r<1$,
\begin{align*}
	\left[\sW^{\alpha^*,2}(\sT)\right]^{r} = \left(\sL^2(\sT),\sW^{\alpha^*,2}(\sT)\right)_{r,2} = \sW^{\alpha^*r,2}(\sT).
\end{align*}
Then Combined with the previous discussion, we can draw the conclusions:
\begin{enumerate}
	\item Suppose that $\beta_0$ satisfies Assumption \ref{assumption: regularity condition} with $0\leq \theta<\infty$ and $ \cL_\cC \preceq \delta_1 \cL_{\cK_{\alpha^*}}^{r_1}$ for some $\delta_1>0$ and $r_1\geq 0$. Then we have $\beta_0 \in \left[\sW^{\alpha^*,2}(\sT)\right]^{(1+r_1)\theta} = \sW^{\alpha^*(1+r_1)\theta,2}(\sT)$.
	\item Suppose that $\cL_\cC \succeq \delta_2 \cL_{\cK_{\alpha^*}}^{r_2}$ and $\beta_0 \in \left[\sW^{\alpha^*,2}(\sT)\right]^{(1+r_2)\tau} = \sW^{\alpha^*(1+r_2)\tau,2}(\sT)$ for some $\delta_2>0$, $r_2\geq 0$ and $\tau\geq 0$. Then we have Assumption \ref{assumption: regularity condition} is satisfied with $\theta =\tau$.
\end{enumerate}

\subsection{Comparisons with Related Results}\label{subsection: comparisons}
There are only a few works studying functional linear regression with discretely observed data, among which the most recent paper \cite{wang2022functional} is notable. The authors of \cite{wang2022functional} establish finite sample upper bounds for the prediction error of constrained least squares estimator within a highly flexible model that accommodates functional responses, both functional and vector covariates, and discrete sampling. While our paper focuses on the scalar response case, we believe that the techniques developed here can be readily extended to handle functional responses without significant modifications. For further comparison, we first give the following proposition which can be derived in the context our paper using the techniques of \cite{wang2022functional}.
\begin{proposition}\label{propositon: upper bound1}
		Suppose that the slope function $\beta_0\in \sW^{\alpha^*,2}(\sT)$, Assumption \ref{assumption: epsilon2} is satisfied and additionally $\epsilon$ is a Gaussian random variable, Assumption \ref{assumption: sampling scheme} and \ref{assumption: X1} are satisfied, and the eigenvalues $\{\mu_{\alpha^*,j}\}_{j=1}^\infty$ satisfy $\mu_{\alpha^*,j}\asymp j^{-1/p}$ for some $0<p \leq1$. The constrained least squares estimator based on the training sample set $S$ is given by
		\begin{align}\label{cls estimator}
			\widetilde{\beta}_{S,\alpha^*,C_\beta}:=\mathop{\mathrm{argmin}}_{\beta\in \sW^{{\alpha^*},2}(\sT),\left\|\beta\right\|_{\sW^{{\alpha^*},2}}\leq C_\beta}\left\{\frac{1}{N}\sum_{i=1}^N \left(Y_i-\sum_{k=1}^{m}\left(r_{k+1}-r_{k}\right)\beta(r_k)X_i(r_k)\right)^2\right\},
		\end{align}
	where $C_\beta>0$ is a parameter to be chosen. If we take $C_{\beta}=C\sqrt{\log(N) N^{-1}}$ for some sufficiently large constant $C$, then the following upper bound holds with probability at least $1-N^{-4}$:
	\begin{align}\label{equation: upper bound3}
		\mathcal{R}(\widetilde{\beta}_{S,\alpha^*,C_\beta}) - \mathcal{R}(\beta_0) \lesssim \log(N)\left(m^{-{\alpha^*}+1/2}+N^{-\frac{1}{1+p}}\right),
	\end{align}
where $\mathcal{R}(\widetilde{\beta}_{S,\alpha^*,C_\beta}) - \mathcal{R}(\beta_0)$ represents the excess prediction risk of $\widetilde{\beta}_{S,\alpha^*,C_\beta}$, defined as 
\begin{align*}
	\mathcal{R}(\widetilde{\beta}_{S,\alpha^*,C_\beta}) - \mathcal{R}(\beta_0):= \mathbb{E}\left[Y-\left\langle X,\widetilde{\beta}_{S,\alpha^*,C_\beta}\right\rangle_{{\sL}^2}^2\right] - \mathbb{E}\left[Y-\left\langle X,\beta_0\right\rangle_{{\sL}^2}^2\right].
\end{align*}
\end{proposition}
According to the well-know equivalence between Tikhonov regularization algorithm and constrained least squares algorithm (see, e.g., \cite{hastie2009elements}), one can verify that the algorithm \eqref{cls estimator} with parameter $C_\beta>0$ is equivalent to the algorithm \eqref{Tikhonov regularization algorithm} with some $\lambda>0$. The techniques of \cite{wang2022functional} are not applicable for deriving upper bounds for the estimation error, while the techniques developed by our paper are sufficient to establish upper bounds for the excess prediction risk. Noting that \[\EE\left[\mathcal{R}(\overline{\beta}_{S,{\alpha^*},\lambda})-\mathcal{R}(\beta_0)\right]= \EE\left[\left\langle X, \overline{\beta}_{S,{\alpha^*},\lambda}-\beta_0\right\rangle_{\sL^2}^2\right]= \left\|\cL_\cC^{1/2}\left(\overline{\beta}_{S,{\alpha^*},\lambda}-\beta_0\right)\right\|_{\sL^2}^2,\] 
and that \[\left\|\cL_\cC^{1/2}\cL_{\cK_{{\alpha^*}}}^{1/2}(\lambda \cI + \cT_{{\alpha^*}})^{-1/2}\right\|^2=\left\|(\lambda \cI + \cT_{{\alpha^*}})^{-1/2}\cL_{\cK_{{\alpha^*}}}^{1/2}\cL_\cC\cL_{\cK_{{\alpha^*}}}^{1/2}(\lambda \cI + \cT_{{\alpha^*}})^{-1/2}\right\|\leq 1,\] the following corollary can be obtained by using the same arguments as in Theorem \ref{theorem: upper bound2} with the filter function taken as $\left\{\Psi_\lambda^{TR}\mid \lambda \in (0,1) \right\}$, which is given in Example \ref{example: filter function1} (Tikhonov regularization).
\begin{corollary}\label{corollary: upper bound2}
	Suppose that Assumption \ref{assumption: regularity condition} is satisfied with $0\leq\theta\leq 1$, Assumption \ref{assumption: epsilon2}, \ref{assumption: sampling scheme} and \ref{assumption: X2} are satisfied and the eigenvalues $\{\mu_{\alpha^*,j}\}_{j=1}^\infty$ satisfy $\mu_{\alpha^*,j}\lesssim j^{-1/p}$ for some $0<p \leq1$. Then there holds
	\begin{equation}\label{equation: upper bound4}
		\EE\left[\mathcal{R}(\overline{\beta}_{S,{\alpha^*},\Psi_\lambda^{TR}})-\mathcal{R}(\beta_0)\right]\lesssim N^{-\frac{1+2\theta}{1+2\theta+p}},
	\end{equation}
	provided that $\lambda=N^{-\frac{1}{1+2\theta+p}}$, $1/m\leq o\left(N^{-\frac{2+2\theta}{(2{{\alpha^*}}-1)(1+2\theta+p)}}\log^{-\frac{2}{2{\alpha^*}-1}}N\right)$ and $M\leq o\left(N^{\frac{1+2\theta-p}{1+2\theta+p}}\log^{-1}N\right)$.
\end{corollary}
Using the same arguments as in Theorem \ref{theorem: lower bound}, we can establish the following lower bound for the excess prediction risk. Similar results have been established by Theorem 4.4 of \cite{gupta2025optimal}.
\begin{corollary}\label{corollary: lower bound}
		Suppose that Assumption \ref{assumption: regularity condition} is satisfied with $0\leq \theta \leq 1$, Assumption \ref{assumption: epsilon2} is  satisfied and the eigenvalues $\{\mu_{\alpha^*,j}\}_{j=1}^\infty$ satisfy $\mu_{\alpha^*,j}\asymp j^{-1/p}$ for some $0<p \leq1$. Then there holds
	\begin{equation}\label{equation: lower bound2}
		\lim_{\gamma \to 0}\mathop{\lim\inf}_{N\to \infty} \inf_{\hat{\beta}_{S}} \sup_{\beta_0} \mathbb{P} \left\{\mathcal{R}\left(\hat{\beta}_S\right)-\mathcal{R}\left(\beta_0\right)\geq \gamma N^{-\frac{1+2\theta}{1+2\theta+p}}\right\} = 1,
	\end{equation} where the supremum is taken over all $\beta_0 \in \sW^{{\alpha^*},2}(\sT)$ satisfying $\beta_0 \in \mathrm{ran}\cL_{\cK_{{\alpha^*}}}^{1/2}\cT_{{\alpha^*}}^{\theta}$  and the infimum is taken over all possible predictors $\hat{\beta}_S \in \sW^{{\alpha^*},2}(\sT)$ based on the fully observed sample set $\{(X_i,Y_i)\}_{i=1}^N$.
\end{corollary}
The upper bound of \eqref{equation: upper bound4} attains the rate of lower bound given by \eqref{equation: lower bound2} and is thus optimal. Proposition \ref{propositon: upper bound1} and Corollary \ref{corollary: upper bound2} are compared as follows: First, \eqref{equation: upper bound4} in Corollary \ref{corollary: lower bound} establishes upper bounds for distributed estimators under the most common noise assumption that the noise has zero mean and bounded variance, whereas \eqref{equation: upper bound3} in Proposition \ref{propositon: upper bound1} establishes upper bounds for non-distributed estimators under a stricter assumption that the noise is a Gaussian random variable. Second, \eqref{equation: upper bound4} in Corollary \ref{corollary: upper bound2} establishes upper bounds under various regularity conditions of $\beta_0$ characterized by Assumption \ref{assumption: regularity condition}, whereas \eqref{equation: upper bound3} in Proposition \ref{propositon: upper bound1} only provides upper bounds under condition $\beta_0\in \sW^{\alpha^*,2}(\sT)$ which corresponds to Assumption \ref{assumption: regularity condition} with $\theta=0$. Third, under the condition $\beta_0\in \sW^{\alpha^*,2}(\sT)$, the upper bound of \eqref{equation: upper bound3} in Proposition \ref{propositon: upper bound1} achieves the optimal convergence rates up to a logarithmic factor, provided that $m\geq N^{\frac{2}{(2{{\alpha^*}}-1)(1+p)}}$, whereas the upper bound of \eqref{equation: upper bound4} in Corollary \ref{corollary: upper bound2} achieves the optimal convergence rates with a slightly stricter requirement that $1/m\leq o\left(N^{-\frac{2}{(2{{\alpha^*}}-1)(1+p)}}\log^{-\frac{2}{2{{\alpha^*}}-1}}N\right)$.

\subsection{Further Discussion}
There are several directions for future research.

\noindent\textbf{ 1. Extending our results to the functional response case.}

 As we assert in Section \ref{subsection: comparisons}, our results can be extended to this setting. 

In particular, the functional response regression model can be expressed as
\begin{align*}
	Y(t) = \int_{\sT} \beta_0(s, t) X(s) ds + \epsilon(t), \quad\mbox{ for all $t\in \sT$ },
\end{align*}
where $X\in \sL^2(\sT)$ is the functional covariate, $Y\in \sL^2(\sT)$ is the functional response, $\beta_0 \in \sL^2(\sT\times \sT)$ is the target function, $\epsilon\in \sL^2(\sT)$ is the random noise independent of $X$ satisfying $\EE\left[\epsilon\right] = 0$ and $\EE\left[\left\|\epsilon\right\|_{\sL^2(\sT)}^2\right]<\infty$.

The training sample set is given by
\begin{align*}
	S:= \Big\{X_i(s_j), Y_i(t_k)\Big\}_{i=1, j=1, k=1}^{N, m+1, m+1},
\end{align*}
where $\left\{X_i,Y_i\right\}_{i=1}^N$ are $N$ independent copies of $(X,Y)$, the functional covariates $\left\{X_i\right\}_{i=1}^N$ and the functional outputs $\left\{Y_i\right\}_{i=1}^N$ are respectively observed at the discrete points $\left\{s_j\right\}_{j=1}^{m+1}\in \sT$ and $\left\{t_k\right\}_{k=1}^{m+1}\in \sT$, where $m\geq 1$ is an integer.

Following a similar approach as in our paper, we can define the spectral regularization estimators $\hat{\beta}_{S,\Psi_{\lambda}}$ for the functional outputs case. For brevity, we omit the explicit expression of the spectral regularization estimators $\hat{\beta}_{S,\Psi_{\lambda}}$.  The key requirement for ensuring that $\hat{\beta}_{S,\Psi_{\lambda}}$ approximates the target function $\beta_0$ is that the Riemann sums \begin{align*}
	\mbox{$\sum_{j=1}^{m}\left(s_{j+1}-s_j\right)X_i(s_j)\beta_0(s_j,t)$ and $\sum_{j=1}^{m}\sum_{k=1}^{m}\left(s_{j+1}-s_j\right)\left(t_{k+1}-t_k\right)X_i(s_j)Y_i(t_k)$}
\end{align*} provide good approximations of the integrals \begin{align*}
	\mbox{$\int_{\sT}X_i(s)\beta_0(s,t)ds$ \quad and \quad $\int_{\sT}\int_{\sT}X_i(s)Y_i(t)dsdt$,}
\end{align*} respectively. This requires certain smoothness conditions on $\beta_0$, $X$ and $\epsilon$, as well as appropriate constraints on the sampling points $\left\{s_j\right\}_{j=1}^{m+1}$ and $\left\{t_k\right\}_{k=1}^{m+1}$. For instance, we can introduce the following conditions:
\begin{enumerate}
	\item [(\romannumeral1)]\textbf{ Smoothness of $\beta_0$:} \begin{align*}
		\mbox{$\beta_0 \in \sW^{\alpha}(\sT)\otimes \sW^{\alpha}(\sT)$ for some $\alpha>1/2$ and $g_0\in \sL^2(\sT)\otimes \sL^2(\sT)$.}
	\end{align*}
	\item [(\romannumeral2)] \textbf{Moment and regularity conditions on 
		$X$:} There exists a constant $\kappa>0$, such that
	\begin{align*}
		\EE\left[\left\langle X,f\right\rangle^4_{\sL^2(\sT)}\right]\leq \kappa^2\left[\EE\left\langle X,f
		\right\rangle_{\cL^2(\sT)}^2\right]^2,\quad \mbox{ for any $f\in \sL^2(\sT)$}.
	\end{align*} 
	Moreover, there exists a constant $\rho>0$ such that
	\begin{align*}
		\EE\left[\left\|X\right\|_{\sW^{\alpha}(\sT)}^2\right] \leq \rho^2, \quad\mbox{ for some $\alpha>1/2$}.
	\end{align*}
	\item [(\romannumeral3)] \textbf{Noise regularity:} The random noise $\epsilon$ is independent of $X$, and satisfies \begin{align*}
		\mbox{$\EE[\epsilon]=0$ and $\EE\left[\left\|\epsilon\right\|_{\sW^{\alpha}(\sT)}^2\right]\leq \sigma^2$ for some $\alpha>1/2$.}
	\end{align*}
	\item [(\romannumeral4)] \textbf{Sampling scheme:} The discrete sample points $\left\{s_j\right\}_{j=1}^{m+1}$ and $\left\{t_k\right\}_{k=1}^{m+1}$ satisfy 
	\begin{align*}
		\mbox{$s_1<\cdots<s_{m+1}$ ,$t_1<\cdots<t_{m+1}$, $s_1=t_1 =0$ and $s_{m+1}=t_{m+1}=1$,}
	\end{align*}
	for some integer $m\geq 1$. Additionally, there exists a constant $C_d$ such that \begin{align*}
		\mbox{$t_{k+1}-t_{k}\leq \frac{C_d}{m}$ for any $1\leq k\leq m$, and $s_{j+1}-s_{j}\leq \frac{C_d}{m}$ for any $1\leq j\leq m$.}
	\end{align*}
\end{enumerate}
Under these assumptions, the approximation properties of the spectral regularization estimators in the functional output setting can be effectively analyzed using the techniques developed in our work.

\noindent  \textbf{2. Extending our results to accommodate more general source and capacity assumptions, for example, in \cite{bauer2007regularization}.}

If we define a new source condition as $\beta_0 = \cL_{\cK_{\alpha^*}}^{1/2}\phi(T_{\alpha^{*}})g_0 $ for some operator monotone index function $\phi$ and $g_0\in \sL^2(\sT)$, following the condition (11) in \cite{bauer2007regularization}, then under some mild assumptions on the index function $\phi$ ( e.g., $\left\|T_{\alpha^*}^{-\theta}\phi(T_{\alpha^{*}})\right\|\leq R$ for some $0<\theta\leq 1$ and $R>0$), our analysis can be extended to this setting, yielding similar results.
 
\noindent \textbf{3. Extending our results to the higher-dimensional case.}

For higher-dimensional case, we assume that the functional covariates $X_i$ and the target function $\beta_0$ take values in a bounded domain $\mathcal{X}\in \mathbb{R}^d$ for some integer $d\geq 2$. The functional covariates $X_i$ are observed on discrete sample points $\left\{r_k\right\}_{k=1}^m \in \mathcal{X}$. The primary challenge is to construct  approximate quadrature weights $\left\{\Omega_{k}^\bx\right\}_{k=1}^m$ based on the sample set \begin{align*}
	\bx = \left\{\left(X_i(r_1),\cdots, X_i(r_m)\right) \right\}_{i=1}^N,
\end{align*} to ensure that the Riemann sums $\sum_{k=1}^{m}\Omega_{k}^{\bx}X_i(r_k)\beta_0(r_k)$ provide good approximations of the integral $\int_{\mathcal{X}}X(s)\beta_0(s)ds$. This requires a more detailed analysis of the sampling scheme in higher-dimensional spaces and the construction of appropriate quadrature weights for Riemann summation. Addressing this issue necessitates further investigation, as it cannot be directly handled using the techniques developed in this paper. We have studied such issues and are preparing a paper for future publication.

\noindent \textbf{4. Extending our results to the polynomial regression case.}

Linear regression is  is a particular case of polynomial regression. Recently, the authors of \cite{holzleitner2024regularized} analyzed the polynomial regression in the functional data setting and established meaningful results. 

In this setting, the polynomial regression model of order $p\geq 1$ can be expressed as
\begin{align*}
	Y = \sum_{\ell =1 }^{p}\int_{\sT^{\ell}} \beta_{\ell}(s_1,\cdots,s_{\ell})\prod_{j=1}^{\ell}X(s_{\ell})d(s_{\ell}) + \epsilon,
\end{align*}
where $\beta_{\ell} \in \underbrace{\sL^2(\sT)\otimes\cdots\otimes \sL^2(\sT)}_{\ell-times}$, $X\in \sL^2(\sT)$ is the functional covariate, $Y\in \mathbb{R}$ is the scalar response and $\epsilon$ is the random noise independent of $X$. The training sample set is given by
\[S:=\left\{\left(X_i(r_1),X_i(r_2),\cdots,X_i(r_m),X_i(r_{m+1}),Y_i\right)\right\}_{i=1}^N,\] where $\left\{(X_i,Y_i)\right\}_{i=1}^N$ are $N$ independent copies of the random variable $(X,Y)$, and functional covariates $\left\{X_i\right\}_{i=1}^N$ are observed at discrete points $\left\{r_k\right\}_{k=1}^{m+1}$, with $m\geq1$ and $0\leq r_1<\cdots<r_m<r_{m+1}\leq 1$.

Based on the approach presented in our work and that of \cite{holzleitner2024regularized}, one can give the spectral regularization estimators $\hat{\beta}_{S,\Psi_{\lambda}}$. The key requirement for ensuring that $\hat{\beta}_{S,\Psi_{\lambda}}$ approximates $(\beta_1,\cdots,\beta_p)$ is that the Riemann sums
\begin{align*}
	\sum_{k_1=1}^{m}\cdots\sum_{k_\ell=1}^{m}\left(r_{k_1+1} - r_{k_1}\right)X_i(r_{k_1})\cdots\left(r_{k_\ell+1} - r_{k_\ell}\right)X_i(r_{k_\ell})\beta_{\ell}(r_{k_1},\cdots,r_{k_j})
\end{align*}
provide good approximations of the integrals 
\begin{align*}
	\int_{\sT^{\ell}} \beta_{\ell}(s_1,\cdots,s_{\ell})\prod_{j=1}^{\ell}X_i(s_{\ell})d(s_{\ell}).
\end{align*}
This requires appropriate smoothness conditions on $\beta_{\ell}$ and $X_i$, as well as suitable constraints on the sampling points $\left\{r_k\right\}_{k=1}^{m+1}$. If we impose Assumption 2 (noise condition), Assumption 3 (sampling scheme), Assumption 4 (regularity condition of functional covariate \uppercase\expandafter{\romannumeral1}) and Assumption 5 (regularity condition of functional covariate \uppercase\expandafter{\romannumeral2}) from our work, along with the following smoothness condition on $(\beta_1,\cdots,\beta_p)$:
\begin{align*}
	\beta_\ell\in \underbrace{\sW^{\alpha,2}(\sT)\otimes\cdots\otimes \sW^{\alpha,2}(\sT)}_{\ell-times}, \quad\mbox{ for some $\alpha>1/2$ and any $1\leq \ell \leq p$},
\end{align*}
then we believe that our results can be extended to the polynomial regression setting.

\section{Convergence Analysis}\label{section: convergence analysis}

In this section, we first derive the upper bounds presented in Theorem \ref{theorem: upper bound1} and \ref{theorem: upper bound2}. Then we establish the mini-max lower bound in Theorem \ref{theorem: lower bound}.

\subsection{Deriving Upper Bounds}\label{subsection: upper bounds}

For any $j=1,2,\dots,M$, define the event \[\sU_j = \left\{\bx_j: \left\|\left(\lambda \cI + \cT_{{\alpha^*}}\right)^{-1/2}\left(\cT_{{\alpha^*},\bx_j}- \cT_{{\alpha^*}}\right)\left(\lambda \cI + \cT_{{\alpha^*}}\right)^{-1/2}\right\| \geq 1/2\right\},\]
and denote its complement by $\sU_j^c$. Let $\sU = \cup_{j=1}^M\sU_j$ be the union of the above events. Then the complement of $\sU$ is given by $\sU^c=\cap_{j=1}^{M} \sU^c_j$. Hereafter,  let $\II_\mathcal{E}$ denote the indicator function of the event $\mathcal{E}$ and $\PP(\mathcal{E})=\mathbb{E} \left[\II_\mathcal{E}\right]$.  We first give the following estimation
\begin{equation}\label{equation: inverse norm under event}
	\begin{split}
		&\left\|(\lambda \cI + \cT_{{\alpha^*}})^{1/2}(\lambda \cI + \cT_{{\alpha^*},\bx_j})^{-1}(\lambda \cI + \cT_{{\alpha^*}})^{1/2}\right\|\II_{\sU_j^c}\\
		=& \left\|\left(\cI- (\lambda \cI + \cT_{{\alpha^*}})^{-1/2}(\cT_{{\alpha^*}}-\cT_{{\alpha^*},\bx_j})(\lambda \cI + \cT_{{\alpha^*}})^{-1/2}\right)^{-1}\right\|\II_{\sU_j^c}\\
		\overset{(*)}{\leq}& 1+\sum_{j=1}^{\infty}\left\|(\lambda \cI + \cT_{{\alpha^*}})^{-1/2}(\cT_{{\alpha^*}}-\cT_{{\alpha^*},\bx_j})(\lambda \cI + \cT_{{\alpha^*}})^{-1/2}\right\|^{k}\II_{\sU_j^c}\leq 1+\sum_{j=1}^{\infty} \frac{1}{2^j}=2,
	\end{split}
\end{equation}  where inequality $(*)$ follows from expanding the inverse operator in Neumann series.

The following lemma provides an upper bound for the expectation of the estimation error $\left\|\overline{\beta}_{S,\alpha^*,\Psi_\lambda}-\beta_0\right\|_{\sW^{\alpha^*,2}}^2$ on the event $\sU^c$.
\begin{lemma}\label{lemma: rates upper bound1}
	Suppose that Assumption \ref{assumption: regularity condition} is satisfied. Then for any partition number $M\geq 1$, there holds
	\begin{equation}\label{equation: rates upper bound1 1} 
		\begin{aligned}
			&\mathbb{E}\left[\left\|\overline{\beta}_{S,\alpha^*,\Psi_\lambda}-\beta_0\right\|_{\sW^{\alpha^*,2}}^2\II_{\sU^c}\right]\\
			\leq& \frac{1}{M}\mathbb{E}\left[\left\|\hat{f}_{S_1,\alpha^*,\Psi_\lambda}-f_0\right\|^2_{\sL^2}\II_{\sU_1^c}\right]+ \left\|\mathbb{E}\left[\left(\hat{f}_{S_1,\alpha^*,\Psi_\lambda}-f_0\right)\II_{\sU_1^c}\right]\right\|^2_{\sL^2},
		\end{aligned}
	\end{equation}
where $f_0\in \sL^2(\sT)$ is given by Assumption \ref{assumption: regularity condition}. 
\end{lemma}
\begin{proof}
	Recall that $\overline{\beta}_{S,{\alpha^*},\Psi_\lambda}=\cL^{1/2}_{\cK_{\alpha^*}}\overline{f}_{S,\alpha^*,\Psi_\lambda}$ with $
		\overline{f}_{S,\alpha^*,\Psi_\lambda}=\frac{1}{M}\sum_{j=1}^M \hat{f}_{S_j,\alpha^*,\Psi_\lambda}.$ Then under Assumption \ref{assumption: regularity condition} and following from the isometric isomorphism property \eqref{equation: norm and inner product relationship}, we write
	\begin{align*}
		\mathbb{E}\left[\left\|\overline{\beta}_{S,\alpha^*,\Psi_\lambda}-\beta_0\right\|_{\sW^{\alpha^*,2}}^2\II_{\sU^c}\right] = \mathbb{E}\left[\left\|\hat{f}_{S,\alpha^*,\Psi_\lambda}-f_0\right\|_{\sL^2}^2\II_{\sU^c}\right].
	\end{align*}	
	When $M\geq 2$, we write
	\begin{align*}
		&\mathbb{E}\left[\left\|\hat{f}_{S,\alpha^*,\Psi_\lambda}-f_0\right\|_{\sL^2}^2\II_{\sU^c}\right]\\
		=& \EE\left[\left\|\frac{1}{M}\sum_{j=1}^{M}\left(\hat{f}_{S_j,\alpha^*,\Psi_\lambda}-f_0\right)\right\|^2_{\sL^2}\II_{\sU^c}\right]\\
		\overset{(\romannumeral1)}{=}&\frac{1}{M^2}\sum_{j=1}^{M}\EE\left[\left\|\hat{f}_{S_j,\alpha^*,\Psi_\lambda}-f_0\right\|^2_{\sL^2}\II_{\sU^c}\right]+ \frac{1}{M^2}\sum_{j\neq k} \EE\left[\left\langle \hat{f}_{S_j,\alpha^*,\Psi_\lambda}-f_0,\hat{f}_{S_k,\alpha^*,\Psi_\lambda}-f_0\right\rangle_{\sL^2} \II_{\sU^c}\right]\\
		\overset{(\romannumeral2)}{=}& \frac{1}{M}\EE\left[\left\|\hat{f}_{S_1,\alpha^*,\Psi_\lambda}-f_0\right\|^2_{\sL^2}\II_{\sU_1^c}\right]\PP(\cap_{j=2}^M\sU_j^c)\\
		&\quad+\frac{M(M-1)}{M^2}\EE\left[\left\langle \hat{f}_{S_1,\alpha^*,\Psi_\lambda}-f_0,\hat{f}_{S_2,\alpha^*,\Psi_\lambda}-f_0\right\rangle_{\sL^2} \II_{\sU_1^c}\II_{\sU_2^c}\right]\PP(\cap_{j=3}^M \sU_j^c)\\
		\overset{(\romannumeral3)}{\leq}& \frac{1}{M}\mathbb{E}\left[\left\|\hat{f}_{S_1,\alpha^*,\Psi_\lambda}-f_0\right\|^2_{\sL^2}\II_{\sU_1^c}\right]+ \left\|\mathbb{E}\left[\left(\hat{f}_{S_1,\alpha^*,\Psi_\lambda}-f_0\right)\II_{\sU_1^c}\right]\right\|^2_{\sL^2},
	\end{align*}
	where equality (\romannumeral1) follows from the binomial expansion, equality (\romannumeral2) uses the fact that $\II_{\sU^c}=\II_{\sU_1^c}\II_{\sU_2^c}\cdots\II_{\sU_M^c}$ and that for $1\leq j \neq k\leq M$, $(\hat{f}_{S_j,\alpha^*,\Psi_\lambda}-f_0)\II_{\sU_j^c}$ is independent of $(\hat{f}_{S_k,\alpha^*,\Psi_\lambda}-f_0)\II_{\sU_k^c}$, and inequality (\romannumeral3) is from \[\EE\left[\left\langle \hat{f}_{S_1,\alpha^*,\Psi_\lambda}-f_0,\hat{f}_{S_2,\alpha^*,\Psi_\lambda}-f_0\right\rangle_{\sL^2} \II_{\sU_1^c}\II_{\sU_2^c}\right] = \left\|\mathbb{E}\left[\left(\hat{f}_{S_1,\alpha^*,\Psi_\lambda}-f_0\right)\II_{\sU_1^c}\right]\right\|^2_{\sL^2}.\]
	When $M=1$, \eqref{equation: rates upper bound1 1} is obvious.
	
	We have completed the proof of Lemma \ref{lemma: rates upper bound1}.
\end{proof}
In the rest part of the proof, our main task is to estimate the two terms on the right hand side of \eqref{equation: rates upper bound1 1}. For simplicity of notations, we denote
\begin{align}\label{simple notation}
	n:= |S_1|= \frac{N}{M}  \qquad and \qquad \{\left(X_{1,i}(r_1),\cdots, X_{1,i}(r_{m+1}), Y_{1,i}\right)\}_{i=1}^n:= S_1.
\end{align}
Therefore, we write
\begin{align*}
	\cT_{{\alpha^*},\bx_1}=\frac{1}{n}\sum_{i=1}^n \mathcal{S}_{1,i}(\cK_{{\alpha^*}},\bx_1)\otimes \mathcal{S}_{1,i}(\cK_{{\alpha^*}},\bx_1)
\end{align*}
and 
\begin{align*}
	\cG_{{\alpha^*},\bx_1}^*(a) = \sum_{i=1}^n a_i\mathcal{S}_{1,i}(\cK_{{\alpha^*}},\bx_1), \quad \forall a\in \mathbb{R}^n,
\end{align*}
where for any $1\leq i\leq n$, we define $\mathcal{S}_{1,i}(\cK_{{\alpha^*}},\bx_1):= \sum_{k=1}^{m}(r_{k+1}-r_k)\cK_{{\alpha^*}}^{1/2}(\cdot,r_k)X_{1,i}(r_k)$. Then for the first term on the right hand side of \eqref{equation: rates upper bound1 1}, recalling the expressions of $\hat{f}_{S_1,\alpha^*,\Psi_\lambda}$ and using the triangular inequality, we write
\begin{equation}\label{equation: decomposition 1}
	\begin{split}
		&\frac{1}{M}\EE\left[\left\|\hat{f}_{S_1,\alpha^*,\Psi_\lambda}-f_0\right\|_{\sL^2}^2\II_{\sU_1^c}\right]=\frac{1}{M}\EE\left[\left\|\Psi_\lambda\left(\cT_{{\alpha^*},\bx_1}\right)\frac{1}{n}\cG_{{\alpha^*},\bx_1}^*\by_1-f_0\right\|_{\sL^2}^2\II_{\sU_1^c}\right]\\
	=&\frac{1}{M}\EE\left[\left\|\Psi_\lambda\left(\cT_{{\alpha^*},\bx_1}\right)\frac{1}{n}\sum_{i=1}^n\mathcal{S}_{1,i}(\cK_{\alpha^*},\bx_1)\Big(\left\langle X_{1,i},\beta_0\right\rangle_{\sL^2} +\epsilon_{1,i} \Big)-f_0\right\|_{\sL^2}^2\II_{\sU_1^c}\right]\\
	\leq&\frac{2}{M}\EE\left[\left\|\Psi_\lambda\left(\cT_{{\alpha^*},\bx_1}\right)\frac{1}{n}\sum_{i=1}^n\mathcal{S}_{1,i}(\cK_{\alpha^*},\bx_1)\left\langle \cL_{\cK_{{\alpha^*}}}^{1/2}X_{1,i},f_0\right\rangle_{\sL^2} -f_0\right\|_{\sL^2}^2\II_{\sU_1^c}\right]\\
	&+ \frac{2}{M}\EE\left[\left\|\Psi_\lambda\left(\cT_{{\alpha^*},\bx_1}\right)\frac{1}{n}\sum_{i=1}^n\mathcal{S}_{1,i}(\cK_{\alpha^*},\bx_1)\epsilon_{1,i}\right\|_{\sL^2}^2\II_{\sU_1^c}\right]\leq\frac{4}{M}\EE\Big[\mathcal{F}_1\left(\bx_1,\alpha^*,\Psi_\lambda\right)\Big]\\
	&+ \frac{4}{M}\EE\Big[\mathcal{F}_2\left(\bx_1,\alpha^*,\Psi_\lambda\right)\II_{\sU_1^c}\Big]+ \frac{4}{M}\EE\Big[\mathcal{F}_3\left(S_1,\alpha^*,\Psi_\lambda\right)\Big]+\frac{4}{M}\EE\Big[\mathcal{F}_4\left(S_1,\alpha^*,\Psi_\lambda\right)\II_{\sU_1^c}\Big],
	\end{split}
\end{equation}
where we denote $\epsilon_{1,i}:=Y_{1,i}-\langle X_{1,i},\beta_0\rangle_{\sL^2}$ for any $1\leq i\leq n$ and we define
\begin{align*}
	\mathcal{F}_1\left(\bx_1,\alpha^*,\Psi_\lambda\right)&:= \left\|\Psi_\lambda\left(\cT_{{\alpha^*},\bx_1}\right)\frac{1}{n}\sum_{i=1}^n\mathcal{S}_{1,i}(\cK_{\alpha^*},\bx_1)\left\langle \cL_{\cK_{{\alpha^*}}}^{1/2}X_{1,i} - \mathcal{S}_{1,i}(\cK_{\alpha^*},\bx_1),f_0\right\rangle_{\sL^2} \right\|_{\sL^2}^2;\\
	\mathcal{F}_2\left(\bx_1,\alpha^*,\Psi_\lambda\right)&:=\left\|\Psi_\lambda\left(\cT_{{\alpha^*},\bx_1}\right)\frac{1}{n}\sum_{i=1}^n\mathcal{S}_{1,i}(\cK_{\alpha^*},\bx_1)\left\langle \mathcal{S}_{1,i}(\cK_{\alpha^*},\bx_1),f_0\right\rangle_{\sL^2} - f_0\right\|_{\sL^2}^2;\\
	\mathcal{F}_3\left(S_1,\alpha^*,\Psi_\lambda\right)&:=\left\|\Psi_\lambda\left(\cT_{{\alpha^*},\bx_1}\right)\frac{1}{n}\sum_{i=1}^n\Big(\mathcal{S}_{1,i}(\cK_{\alpha^*},\bx_1) - \cL_{\cK_{{\alpha^*}}}^{1/2}X_{1,i}\Big)\epsilon_{1,i}\right\|_{\sL^2}^2;\\
	\mathcal{F}_4\left(S_1,\alpha^*,\Psi_\lambda\right)&:=\left\|\Psi_\lambda\left(\cT_{{\alpha^*},\bx_1}\right)\frac{1}{n}\sum_{i=1}^n\cL_{\cK_{{\alpha^*}}}^{1/2}X_{1,i}\epsilon_{1,i}\right\|_{\sL^2}^2
\end{align*}
While for the second term on the right hand side of \eqref{equation: rates upper bound1 1}, noting that $\epsilon$ is a zero-mean random variable independent of $X$ and that the event $\sU_1$ is only related to $\bx_1$, using the triangular inequality, we write 
\begin{equation}\label{equation: decomposition 2}
	\begin{split}
		&\left\|\mathbb{E}\left[\left(\hat{f}_{S_1,\alpha^*,\Psi_\lambda}-f_0\right)\II_{\sU_1^c}\right]\right\|^2_{\sL^2}\\
	=& \left\|\EE\left[\left\{\Psi_\lambda\left(\cT_{{\alpha^*},\bx_1}\right)\frac{1}{n}\sum_{i=1}^n\mathcal{S}_{1,i}(\cK_{\alpha^*},\bx_1)\Big(\left\langle \cL_{\cK_{{\alpha^*}}}^{1/2}X_{1,i},f_0\right\rangle_{\sL^2} +\epsilon_{1,i} \Big)-f_0\right\}\II_{\sU_1^c}\right]\right\|^2_{\sL^2}\\
	=&\left\|\EE\left[\left\{\Psi_\lambda\left(\cT_{{\alpha^*},\bx_1}\right)\frac{1}{n}\sum_{i=1}^n\mathcal{S}_{1,i}(\cK_{\alpha^*},\bx_1)\left\langle \cL_{\cK_{{\alpha^*}}}^{1/2}X_{1,i},f_0\right\rangle_{\sL^2} -f_0\right\}\II_{\sU_1^c}\right]\right\|^2_{\sL^2}\\
	\overset{(\dagger)}{\leq}& \EE\left[\left\|\Psi_\lambda\left(\cT_{{\alpha^*},\bx_1}\right)\frac{1}{n}\sum_{i=1}^n\mathcal{S}_{1,i}(\cK_{\alpha^*},\bx_1)\left\langle \cL_{\cK_{{\alpha^*}}}^{1/2}X_{1,i},f_0\right\rangle_{\sL^2} -f_0\right\|_{\sL^2}^2\II_{\sU_1^c}\right]\\
	\leq& 2\EE\Big[\mathcal{F}_1\left(\bx_1,\alpha^*,\Psi_\lambda\right)\Big] + 2\EE\Big[\mathcal{F}_2\left(\bx_1,\alpha^*,\Psi_\lambda\right)\II_{\sU_1^c}\Big],
	\end{split}
\end{equation}
where terms $\mathcal{F}_1\left(\bx_1,\alpha^*,\Psi_\lambda\right)$ and $\mathcal{F}_2\left(\bx_1,\alpha^*,\Psi_\lambda\right)$ are defined by \eqref{equation: decomposition 1} and inequality $(\dagger)$ uses Jensen's inequality.

We next provide upper bounds for the terms $\EE\left[\mathcal{F}_1\left(\bx_1,\alpha^*,\Psi_\lambda\right)\right]$, $\EE\left[\mathcal{F}_2\left(\bx_1,\alpha^*,\Psi_\lambda\right)\II_{\sU_1^c}\right]$, $\EE\left[\mathcal{F}_3\left(S_1,\alpha^*,\Psi_\lambda\right)\right]$ and $\EE\left[\mathcal{F}_4\left(S_1,\alpha^*,\Psi_\lambda\right)\II_{\sU_1^c}\right]$ by proving the following lemma. To this end, define the effective dimension with respect to $\alpha^*$ as
\begin{align}\label{effectivedimension}
	\sN_{\alpha^*}(\lambda) := \sum_{j=1}^{\infty}\frac{\mu_{\alpha^*,j}}{\lambda+\mu_{\alpha^*,j}}.
\end{align}
where $\lambda >0$ and $\{\mu_{\alpha^*,j}\}_{j\geq 1}$ are positive eigenvalues of $\cT_{{\alpha^*}}$ (with geometric multiplicities) arranged in decreasing order. The effective dimension is widely used in the convergence analysis of kernel ridge regression (see, \cite{caponnetto2007optimal,fischer2020sobolev,lin2017distributed,zhang2015divide}).
\begin{lemma}\label{lemma: rates upper bound2}
	Let $\left\{\Psi_\lambda:[0,\infty) \rightarrow \mathbb{R}| \lambda\in (0,1)\right\}$ be a filter function satisfying Definition \ref{definition: filter functions} with qualification $\nu_{\Psi}\geq 1$. Suppose that Assumption \ref{assumption: regularity condition} is satisfied with $0\leq \theta\leq \nu_{\Psi}$ and $g_0 \in \sL^2(\sT)$, Assumptions \ref{assumption: epsilon2} is satisfied with $\sigma>0$, \ref{assumption: sampling scheme} with $C_d>0$ and \ref{assumption: X1} is satisfied with $\rho,\kappa>0$. Then for any $\lambda \in (0,1)$, there holds
	\begin{align}\label{equation: rates upper bound21}
		\EE\Big[\mathcal{F}_1\left(\bx_1,\alpha^*,\Psi_\lambda\right)\Big]\leq 2\rho\kappa^4E^2C_1^2C_{\alpha^*}^2\left(C_2^2+C_1C_{\alpha^*}\right)^2\rho_{\alpha^*}^{\theta}\left\|g_0\right\|_{\sL^2}^2\lambda^{-2}m^{-2\alpha^*+1},
	\end{align}
	\begin{align}\label{equation: rates upper bound22}
		\EE\Big[\mathcal{F}_2\left(\bx_1,\alpha^*,\Psi_\lambda\right)\II_{\sU_1^c}\Big]\leq F_{\theta}^2\lambda^{2\theta}\left\|g_0\right\|_{\sL^2}^2\lambda^{2\theta},
	\end{align}
\begin{align}\label{equation: rates upper bound23}
	\EE\Big[\mathcal{F}_3\left(S_1,\alpha^*,\Psi_\lambda\right)\Big]\leq \kappa^2\sigma^2E^2C_1^2C_{\alpha^*}^2\lambda^{-2}m^{-2\alpha^*+1},
\end{align}
and 
\begin{align}\label{equation: rates upper bound24}
	\EE\Big[\mathcal{F}_4\left(S_1,\alpha^*,\Psi_\lambda\right)\II_{\sU_1^c}\Big]\leq 2E^2\sigma^2\lambda^{-1}\frac{M\sN_{\alpha^*}(\lambda)}{N},
\end{align}
where $\rho_{\alpha^*}$ is a constant given by \eqref{equation: bound of expectation of L_KX}, $C_{\alpha^*}= C_{\alpha^*}C_d^{\alpha-1/2}$ is a constant depending on $\alpha^*$ defined by Lemma \ref{lemma: integral approximation}, $C_1$ and $C_2$ are constants given by Lemma \ref{lemma: Sobolev inequalities}.
\end{lemma}
\begin{proof}
	We begin with the first inequality \eqref{equation: rates upper bound21}. Recalling the expression of $\mathcal{F}_1\left(\bx_1,\alpha^*,\Psi_\lambda\right)$, we write
	\begin{align*}
		&\EE\left[\mathcal{F}_1\left(\bx_1,\alpha^*,\Psi_\lambda\right)\right]\nonumber\\
		=& \EE\left[\left\|\Psi_\lambda\left(\cT_{{\alpha^*},\bx_1}\right)\frac{1}{n}\sum_{i=1}^n\mathcal{S}_{1,i}(\cK_{\alpha^*},\bx_1)\left\langle \cL_{\cK_{{\alpha^*}}}^{1/2}X_{1,i} - \mathcal{S}_{1,i}(\cK_{\alpha^*},\bx_1),f_0\right\rangle_{\sL^2} \right\|_{\sL^2}^2\right]\nonumber\\
		\leq& \lambda^{-2}\EE\Bigg[\left\|\left(\lambda \cI+\cT_{\alpha^*,\bx_1}\right)^{1/2}\Psi_\lambda\left(\cT_{\alpha^*,\bx_1}\right)\left(\lambda \cI+\cT_{\alpha^*,\bx_1}\right)^{1/2}\right\|^2\nonumber\\
		&\times \left\|\frac{1}{n}\sum_{i=1}^n\mathcal{S}_{1,i}(\cK_{\alpha^*},\bx_1)\left\langle \cL_{\cK_{{\alpha^*}}}^{1/2}X_{1,i} - \mathcal{S}_{1,i}(\cK_{\alpha^*},\bx_1),f_0\right\rangle_{\sL^2}\right\|_{\sL^2}^2\Bigg].
	\end{align*}
While for the term $\left\|\left(\lambda \cI+\cT_{\alpha^*,\bx_1}\right)^{1/2}\Psi_\lambda\left(\cT_{\alpha^*,\bx_1}\right)\left(\lambda \cI+\cT_{\alpha^*,\bx_1}\right)^{1/2}\right\|^2$, noting that $\cT_{\alpha^*,\bx_1}$ is a finite-rank non-negative operator (with rank at most $n = \left|S_1\right|$), we write 
	\begin{align*}
		\cT_{\alpha^*,\bx_1} = \sum_{j=1}^\infty \hat{\mu}_{\bx_1,\alpha^*,j} \hat{\phi}_{\bx_1,\alpha^*,j}\otimes \hat{\phi}_{\bx_1,\alpha^*,j},
	\end{align*}
	where $\left\{\hat{\mu}_{\bx_1,\alpha^*,j}\right\}_{j=1}^\infty$ is a non-negative and decreasing sequence in which at most $n$ numbers are positive and $\left\{\hat{\phi}_{\bx_1,\alpha^*,j}\right\}_{j=1}^\infty$ is an orthornormal basis of $\sL^2(\sT)$. As $\left\{\Psi_\lambda:[0,\infty) \rightarrow \mathbb{R}| \lambda\in (0,1)\right\}$ is a filter function satisfying Definition \ref{definition: filter functions}, there holds
	\begin{equation}\label{equation: upper bound of F_1 1}
	\begin{split}
			&\left\|\left(\lambda \cI+\cT_{\alpha^*,\bx_1}\right)^{1/2}\Psi_\lambda\left(\cT_{\alpha^*,\bx_1}\right)\left(\lambda \cI+\cT_{\alpha^*,\bx_1}\right)^{1/2}\right\|^2\\
		=&\sup_{\|f\|_{\sL^2}=1}\left\{\left\|\left(\lambda \cI+\cT_{\alpha^*,\bx_1}\right)^{1/2}\Psi_\lambda\left(\cT_{\alpha^*,\bx_1}\right)\left(\lambda \cI+\cT_{\alpha^*,\bx_1}\right)^{1/2}f\right\|_{\sL^2}^2\right\}\\ 
		=& \sup_{\|f\|_{\sL^2}=1}\left\{\sum_{j=1}^\infty \left\langle f,\hat{\phi}_{\bx_1,\alpha^*,j} \right\rangle_{\sL^2}^2\Psi_\lambda^2(\hat{\mu}_{\bx_1,\alpha^*,j})\left(\lambda + \hat{\mu}_{\bx_1,\alpha^*,j}\right)^2\right\}\\
		 \overset{(\dagger)}{\leq}&   E^2\sup_{\|f\|_{\sL^2}=1}\left\{\sum_{j=1}^\infty \left\langle f,\hat{\phi}_{\bx_1,\alpha^*,j} \right\rangle_{\sL^2}^2\right\}=E^2,
	\end{split}
	\end{equation}
	where inequality $(\dagger)$ is from Definition \ref{definition: filter functions}.
	
Then using the triangular inequality, we have
\begin{align*}
	&\EE\left[\mathcal{F}_1\left(\bx_1,\alpha^*,\Psi_\lambda\right)\right]\nonumber\\
	\leq& 2E^2\lambda^{-2}\EE\left[\left\|\frac{1}{n}\sum_{i=1}^n\mathcal{S}_{1,i}(\cK_{\alpha^*},\bx_1)\left\langle \cL_{\cK_{{\alpha^*}}}^{1/2}X_{1,i} - \mathcal{S}_{1,i}(\cK_{\alpha^*},\bx_1),f_0\right\rangle_{\sL^2}\right\|_{\sL^2}^2\right]\nonumber\\
	\leq&2E^2\lambda^{-2}\frac{1}{n}\sum_{i=1}^n\EE\left[\Big\|\mathcal{S}_{1,i}(\cK_{\alpha^*},\bx_1)\Big\|_{\sL^2}^2\left\|\cL_{\cK_{{\alpha^*}}}^{1/2}X_{1,i} - \mathcal{S}_{1,i}(\cK_{\alpha^*},\bx_1)\right\|_{\sL^2}^2\Big\|f_0\Big\|_{\sL^2}^2\right]\nonumber\\
	\overset{(\dagger)}{\leq}&2E^2\rho_{\alpha^*}^{\theta}\left\|g_0\right\|_{\sL^2}^2\lambda^{-2}\frac{1}{n}\sum_{i=1}^n\EE\left[\Big\|\mathcal{S}_{1,i}(\cK_{\alpha^*},\bx_1)\Big\|_{\sL^2}^2\left\|\cL_{\cK_{{\alpha^*}}}^{1/2}X_{1,i} - \mathcal{S}_{1,i}(\cK_{\alpha^*},\bx_1)\right\|_{\sL^2}^2\right].
\end{align*}
Here inequality $(\dagger)$ applies Assumption \ref{assumption: regularity condition} and the fact that
\begin{align}\label{equation: operator norm of T_alpha^*}
	\left\|\cT_{\alpha^*}\right\|^2\leq \left\|\cT_{\alpha^*}\right\|^2_{\sF} = \sum_{j=1}^{\infty}\EE\left[\left\langle \cL_{\cK_{\alpha^*}}^{1/2}X, e_j\right\rangle_{\sL^2}^2\right]=\EE\left[\left\|\cL_{\cK_{\alpha^*}}^{1/2}X\right\|_{\sL^2}^2\right]\leq \rho_{\alpha^*}^2,
\end{align} 
where $\left\{e_j\right\}_{j=1}^\infty$ is an (any) orthonormal basis of $\sL^2(\sT)$.

While for any $1\leq i\leq n$, recalling that $\mathcal{S}_{1,i}(\cK_{{\alpha^*}},\bx_1) = \sum_{k=1}^m(r_{k+1}-r_k)\cK_{{\alpha^*}}^{1/2}(\cdot,r_k)X_{1,i}(r_k)$, we write
\begin{equation}\label{prooflemma 1}
\begin{split}
		&\left\|\cL_{\cK_{{\alpha^*}}}^{1/2}X_{1,i}-\mathcal{S}_{1,i}(\cK_{{\alpha^*}},\bx_1)\right\|_{\sL^2}=\sup_{\|f\|_{\sL^2}=1}\left\{\left|\left\langle\cL_{\cK_{{\alpha^*}}}^{1/2}X_{1,i}-\mathcal{S}_{1,i}(\cK_{{\alpha^*}},\bx_1),f\right\rangle_{\sL^2}\right|\right\}\\
	=&\sup_{\|f\|_{\sL^2}=1}\left\{\left|\left\langle\sum_{k=1}^m(r_{k+1}-r_k)\cK_{{\alpha^*}}^{1/2}(\cdot,r_k)X_{1,i}(r_k)-\cL_{\cK_{{\alpha^*}}}^{1/2}X_{1,i},f\right\rangle_{\sL^2}\right|\right\}\\
	\overset{(\romannumeral1)}{=}&\sup_{\|f\|_{\sL^2}=1} \left\{\left|\sum_{k=1}^{m}(r_{k+1}-r_k)\cL_{\cK_{{\alpha^*}}}^{1/2}f(r_k)X_{1,i}(r_k)- \int_{\sT} \cL_{\cK_{{\alpha^*}}}^{1/2}f(t)X_{1,i}(t)dt\right|\right\}\\ 
	\overset{(\romannumeral2)}{\leq}& C_1C_{\alpha^*}m^{-{\alpha^*}+1/2}\left\|X_{1,i}\right\|_{\sW^{{\alpha^*},2}},
\end{split}
\end{equation}
where equality $(\romannumeral1)$ follows from the fact that $\langle \cK_{{\alpha^*}}^{1/2}(\cdot,r_k), f \rangle_{\sL^2}= \cL_{\cK_{{\alpha^*}}^{1/2}}f(r_k)= \cL_{\cK_{{\alpha^*}}}^{1/2}f(r_k)$, inequality $(\romannumeral2)$ uses Assumption \ref{assumption: sampling scheme}, Lemma \ref{lemma: integral approximation} and that for any $f\in \sL^2(\sT)$ satisfying $\left\|f\right\|_{\sL^2}=1$, there holds \[\|\cL_{\cK_{{\alpha^*}}}^{1/2}f(\cdot)X_{1,i}(\cdot)\|_{\sW^{{\alpha^*},2}}\overset{(\romannumeral1)}{\leq} C_1\|\cL_{\cK_{{\alpha^*}}}^{1/2}f\|_{\sW^{{\alpha^*},2}}\|X_{1,i}\|_{\sW^{{\alpha^*},2}} \overset{(\romannumeral2)}{=} C_1\|f\|_{\sL^2}\|X_{1,i}\|_{\sW^{{\alpha^*},2}}= C_1\|X_{1,i}\|_{\sW^{{\alpha^*},2}},\]
where inequality $(\romannumeral1)$ follows from \eqref{equation: Sobolev inequalities} in Lemma \ref{lemma: Sobolev inequalities} and equality $(\romannumeral2)$ uses \eqref{equation: norm and inner product relationship}.

And then for any $1\leq i\leq n$, we write
\begin{align*}
	\Big\|\mathcal{S}_{1,i}(\cK_{\alpha^*},\bx_1)\Big\|_{\sL^2}
	\leq& \left\|\cL_{\cK_{{\alpha^*}}}^{1/2}X_{1,i}-\mathcal{S}_{1,i}(\cK_{{\alpha^*}},\bx_1)\right\|_{\sL^2} + \left\|\cL_{\cK_{{\alpha^*}}}^{1/2}X_{1,i}\right\|_{\sL^2}\\
	\overset{(\romannumeral1)}{\leq}& \left\|\cL_{\cK_{{\alpha^*}}}^{1/2}X_{1,i}-\mathcal{S}_{1,i}(\cK_{{\alpha^*}},\bx_1)\right\|_{\sL^2} + C_2\left\|\cL_{\cK_{{\alpha^*}}}^{1/2}X_{1,i}\right\|_{\sW^{\alpha^*,2}}\\
	\overset{(\romannumeral2)}{\leq}&  C_1C_{\alpha^*}m^{-{\alpha^*}+1/2}\left\|X_{1,i}\right\|_{\sW^{{\alpha^*},2}} + C_2\left\|X_{1,i}\right\|_{\sL^2}\overset{(\romannumeral3)}{\leq} \left(C_2^2+C_1C_{\alpha^*}\right)\left\|X_{1,i}\right\|_{\sW^{{\alpha^*},2}},
\end{align*}
where inequalities $(\romannumeral1)$ and $(\romannumeral3)$ follow from \eqref{equation: Sobolev inequalities} in Lemma \ref{lemma: Sobolev inequalities}, inequality $(\romannumeral1)$ uses \eqref{equation: norm and inner product relationship}.

Combining the above estimates, we write
\begin{align*}
	\EE\left[\mathcal{F}_1(\bx_1,\alpha^*,\lambda)\right]\leq&2E^2C_1^2C_{\alpha^*}^2\left(C_2^2+C_1C_{\alpha^*}\right)^2\rho_{\alpha^*}^{\theta}\left\|g_0\right\|_{\sL^2}^2\lambda^{-2}m^{-2\alpha^*+1} \frac{1}{n}\sum_{i=1}^{n}\EE\left[\left\|X_{1,i}\right\|_{\sW^{{\alpha^*},2}}^4\right].
\end{align*}

While for any $1\leq i\leq n$, we have
\begin{align}\label{equation: 4-th Walpha,2-norm of X}
	\EE\left[\left\|X_{1,i}\right\|_{\sW^{{\alpha^*},2}}^{4}\right]&\overset{(\romannumeral1)}{=}\EE\left[\left\|\cL_{\cK_{{\alpha^*}}}^{-1/2}X_{1,i}\right\|_{\sL^2}^{4}\right]= \EE\left[\left(\sum_{j=1}^\infty\left\langle \cL_{\cK_{{\alpha^*}}}^{-1/2}X_{1,i},\phi_{\alpha^*,j}\right\rangle_{\sL^2}^2\right)^{2}\right]\nonumber\\
	&=\sum_{j_1=1}^\infty\sum_{j_2=1}^\infty\EE\left[\left\langle \cL_{\cK_{{\alpha^*}}}^{-1/2}X_{1,i},\phi_{\alpha^*, j_1}\right\rangle_{\sL^2}^2\left\langle \cL_{\cK_{{\alpha^*}}}^{-1/2}X_{1,i},\phi_{\alpha^*, j_2}\right\rangle_{\sL^2}^2\right]\nonumber\\
	&\overset{(\romannumeral2)}{\leq} \sum_{j_1=1}^\infty\sum_{j_2=1}^\infty\left[\EE\left\langle \cL_{\cK_{{\alpha^*}}}^{-1/2}X_{1,i},\phi_{\alpha^*, j_1}\right\rangle_{\sL^2}^4\right]^{\frac{1}{2}}\left[\EE\left\langle \cL_{\cK_{{\alpha^*}}}^{-1/2}X_{1,i},\phi_{\alpha^*, j_2}\right\rangle_{\sL^2}^4\right]^{\frac{1}{2}}\nonumber\\
	&\overset{(\romannumeral3)}{\leq}\rho\sum_{j_1=1}^\infty\sum_{j_2=1}^\infty\EE\left[\left\langle \cL_{\cK_{{\alpha^*}}}^{-1/2}X_{1,i},\phi_{\alpha^*, j_1}\right\rangle_{\sL^2}^2\right]\EE\left[\left\langle \cL_{\cK_{{\alpha^*}}}^{-1/2}X_{1,i},\phi_{\alpha^*, j_2}\right\rangle_{\sL^2}^2\right]\nonumber\\
	&=\rho\left[\EE\sum_{j=1}^\infty\left\langle \cL_{\cK_{{\alpha^*}}}^{-1/2}X_{1,i},\phi_{\alpha^*,j}\right\rangle_{\sL^2}^2\right]^2=\rho\left[\EE\left\|\cL_{\cK_{{\alpha^*}}}^{-1/2}X_{1,i}\right\|_{\sL^2}^2\right]^2\nonumber\\
	&\overset{(\romannumeral4)}{=}\rho\left[\EE\|X_{1,i}\|_{\sW^{{\alpha^*},2}}^2\right]^2 \overset{(\romannumeral5)}{\leq} \rho\kappa^4,
\end{align}
where $\{\left(\mu_{\alpha^*,j},\phi_{\alpha^*,j}\right)\}_{j=1}^\infty$ is given by the singular value decomposition of $\cT_{{\alpha^*}}$ in \eqref{singular value decomposition} and $\cL_{\cK_{{\alpha^*}}}^{-1/2}$ denotes the inverse operator of $\cL_{\cK_{{\alpha^*}}}^{1/2}$, equality $(\romannumeral1)$ is from \eqref{equation: norm and inner product relationship}, inequality $(\romannumeral2)$ uses H\"older inequality, inequality $(\romannumeral3)$ follows from \eqref{equation: X1} in Assumption \ref{assumption: X1}, equality $(\romannumeral4)$ is also from \eqref{equation: norm and inner product relationship} and inequality $(\romannumeral5)$ is due to \eqref{equation: X1 additional} in Assumption \ref{assumption: X1}.

Therefore, we write
\begin{align*}
	\EE\left[\mathcal{F}_1(\bx_1,\alpha^*,\lambda)\II_{\sU_1^c}\right]
	\leq2\rho\kappa^4E^2C_1^2C_{\alpha^*}^2\left(C_2^2+C_1C_{\alpha^*}\right)^2\rho_{\alpha^*}^{\theta}\left\|g_0\right\|_{\sL^2}^2\lambda^{-2}m^{-2\alpha^*+1},
\end{align*}
This completes the proof of inequality \eqref{equation: rates upper bound21}.

We next turn to prove the second inequality \eqref{equation: rates upper bound22}. Recalling the expression of $\mathcal{F}_2(\bx_1,\alpha^*,\Psi_\lambda)$ and noting that $\cT_{\alpha^*,\bx_1} = \frac{1}{n}\sum_{i=1}^n\mathcal{S}_{1,i}(\cK_{\alpha^*},\bx_1)\otimes \mathcal{S}_{1,i}(\cK_{\alpha^*},\bx_1)$, we write
\begin{align*}
	&\EE\left[\mathcal{F}_2\left(\bx_1,\alpha^*,\Psi_\lambda\right)\II_{\sU_1^c}\right]\\
	=&\EE\left[\left\|\Psi_\lambda\left(\cT_{{\alpha^*},\bx_1}\right)\frac{1}{n}\sum_{i=1}^n\mathcal{S}_{1,i}(\cK_{\alpha^*},\bx_1)\left\langle \mathcal{S}_{1,i}(\cK_{\alpha^*},\bx_1),f_0\right\rangle_{\sL^2} - f_0\right\|_{\sL^2}^2\II_{\sU_1^c}\right]\\
	=&\EE\left[\Big\|\Psi_\lambda\left(\cT_{{\alpha^*},\bx_1}\right)\cT_{\alpha^*,\bx_1}f_0- f_0\Big\|_{\sL^2}^2\II_{\sU_1^c}\right]\overset{(*)}{\leq}\Big\|g_0\Big\|_{\sL^2}^2\EE\left[\Big\|\Big(\Psi_\lambda\left(\cT_{{\alpha^*},\bx_1}\right)\cT_{\alpha^*,\bx_1}-\cI\Big)\cT_{\alpha^*}^{\theta}\Big\|^2\II_{\sU_1^c}\right],
\end{align*}
where inequality $(*)$ follows from Assumption \ref{assumption: regularity condition}. 

While for the term $\Big\|\Big(\Psi_\lambda\left(\cT_{{\alpha^*},\bx_1}\right)\cT_{{\alpha^*},\bx_1} - \cI\Big)\cT_{\alpha^*}^{\theta}\Big\|^2\II_{\sU^c}$, if we write 
\begin{align*}
	\cT_{\alpha^*,\bx_1} = \sum_{j=1}^\infty \hat{\mu}_{\bx_1,\alpha^*,j} \hat{\phi}_{\bx_1,\alpha^*,j}\otimes \hat{\phi}_{\bx_1,\alpha^*,j}
\end{align*} 
same as in \eqref{equation: upper bound of F_1 1}, then we have
\begin{align*}
&\Big\|\Big(\Psi_\lambda\left(\cT_{{\alpha^*},\bx_1}\right)\cT_{{\alpha^*},\bx_1} - \cI\Big)\cT_{\alpha^*}^{\theta}\Big\|^2\II_{\sU^c}\nonumber\\
=& \sup_{\|f\|_{\sL^2}=1} \left\{\sum_{j=1}^{\infty}\Big(\Psi_\lambda(\hat{\mu}_{\bx_1,\alpha^*,j})\hat{\mu}_{\bx_1,\alpha^*,j}-1\Big)^2\hat{\mu}_{\bx_1,\alpha^*,j}^{2\theta}\left\langle f, \hat{\phi}_{\bx_1,\alpha^*,j}\right\rangle_{\sL^2}^2\right\}\II_{\sU_1^c}\nonumber\\
\overset{(\dagger)}{\leq}&2\sup_{\|f\|_{\sL^2}=1} \left\{\sum_{j=1}^{\infty}F_{\theta}^2\lambda^{2\theta}\left\langle f, \hat{\phi}_{\bx_1,\alpha^*,j}\right\rangle_{\sL^2}^2\right\}= F_{\theta}^2\lambda^{2\theta},
\end{align*}
where inequality $(\dagger)$ follows from Definition \ref{definition: filter functions} and the fact that on the event $\sU_1^c$ (recalling the expression of the event $\sU_1^c$), for any $\lambda\in (0,1)$ and any integer $j\geq 1$, there holds
\begin{align}\label{equation: estimation of norm of T_x,alpha}
	\hat{\mu}_{\bx_1,\alpha^*,j}\leq& \left\|\cT_{\alpha^*,\bx_1}\right\| \leq \left\|\cT_{\alpha^*}\right\| + \left\|\cT_{\alpha^*,\bx_1} - \cT_{\alpha^*}\right\|\nonumber\\
	\leq& \left\|\cT_{\alpha^*}\right\| +  \left\|(\lambda \cI + \cT_{\alpha^*})^{1/2}\right\|^2\left\|(\lambda \cI + \cT_{\alpha^*})^{-1/2}\left(\cT_{\alpha^*,\bx_1} - \cT_{\alpha^*}\right)(\lambda \cI + \cT_{\alpha^*})^{-1/2}\right\|\nonumber\\
	\leq& \rho_{\alpha^*}+ \frac{\lambda + \rho_{\alpha^*}}{2}\leq \frac{3\rho_{\alpha^*}}{2} + \frac{1}{2}.
\end{align}
Therefore, we have
\begin{align*}
	\EE\left[\mathcal{F}_2\left(\bx_1,\alpha^*,\Psi_\lambda\right)\II_{\sU_1^c}\right]\leq F_{\theta}^2\lambda^{2\theta}\left\|g_0\right\|_{\sL^2}^2\lambda^{2\theta},
\end{align*}
We have completed the proof of inequality \eqref{equation: rates upper bound22}.

We next prove the third inequality \eqref{equation: rates upper bound23}. Recalling the expression of $\mathcal{F}_3\left(S_1,\alpha^*,\Psi_\lambda\right)$ and that $n= N/M$, noting that $\epsilon$ is a mean-zero random variable independent of $X$, we write
\begin{align*}
	&\EE\left[\mathcal{F}_3\left(S_1,\alpha^*,\Psi_\lambda\right)\right]\nonumber\\
	=& \EE\left[\left\|\Psi_\lambda\left(\cT_{{\alpha^*},\bx_1}\right)\frac{1}{n}\sum_{i=1}^n\Big(\mathcal{S}_{1,i}(\cK_{\alpha^*},\bx_1) - \cL_{\cK_{{\alpha^*}}}^{1/2}X_{1,i}\Big)\epsilon_{1,i}\right\|_{\sL^2}^2\right]\\
	=&\frac{1}{n^2}\sum_{i=1}^n\EE\left[\left\|\Psi_\lambda\left(\cT_{{\alpha^*},\bx_1}\right)\Big(\mathcal{S}_{1,i}(\cK_{\alpha^*},\bx_1) - \cL_{\cK_{{\alpha^*}}}^{1/2}X_{1,i}\Big)\right\|_{\sL^2}^2\right]\EE\Big[\epsilon_{1,i}^2\Big]\\
	\overset{(\romannumeral1)}{\leq}& \sigma^2\lambda^{-2}\frac{1}{n^2}\sum_{i=1}^n\EE\left[\left\|\left(\lambda \cI+\cT_{\alpha^*,\bx_1}\right)^{1/2}\Psi_\lambda\left(\cT_{\alpha^*,\bx_1}\right)\left(\lambda \cI+\cT_{\alpha^*,\bx_1}\right)^{1/2}\right\|^2\left\|\mathcal{S}_{1,i}(\cK_{\alpha^*},\bx_1) - \cL_{\cK_{{\alpha^*}}}^{1/2}X_{1,i}\right\|_{\sL^2}^2\right]\\
	\overset{(\romannumeral2)}{\leq}& \sigma^2E^2C_1^2C_{\alpha^*}^2\lambda^{-2}m^{-2\alpha^*+1}\frac{1}{n^2}\sum_{i=1}^n\EE\Big[\left\|X_{1,i}\right\|_{\sW^{{\alpha^*},2}}^2\Big]\overset{(\romannumeral3)}{\leq}\kappa^2\sigma^2E^2C_1^2C_{\alpha^*}^2\lambda^{-2}m^{-2\alpha^*+1}\frac{M}{N},
\end{align*}
where inequality $(\romannumeral1)$ is due to Assumption \eqref{assumption: epsilon2}, inequality $(\romannumeral2)$ follows from \eqref{equation: upper bound of F_1 1} and \eqref{prooflemma 1}, inequality $(\romannumeral3)$ uses \eqref{equation: X1 additional} in Assumption \ref{assumption: X1}.  

We have obtained the inequality \eqref{equation: rates upper bound23}.

Finally, we prove the fourth inequality \eqref{equation: rates upper bound24}. Recalling the expression of $\mathcal{F}_4\left(S_1,\alpha^*,\Psi_\lambda\right)$ and that $n=N/M$, also noting that $\epsilon$ is a mean-zero random variable independent of $X$, we write
\begin{align*}
	&\EE\left[\mathcal{F}_4\left(S_1,\alpha^*,\Psi_\lambda\right)\II_{\sU_1^c}\right]\nonumber\\
	=&\EE\left[\left\|\Psi_\lambda\left(\cT_{{\alpha^*},\bx_1}\right)\frac{1}{n}\sum_{i=1}^n\cL_{\cK_{{\alpha^*}}}^{1/2}X_{1,i}\epsilon_{1,i}\right\|_{\sL^2}^2\II_{\sU_1^c}\right]=\frac{1}{n^2}\sum_{i=1}^n\EE\left[\left\|\Psi_\lambda\left(\cT_{{\alpha^*},\bx_1}\right)\cL_{\cK_{{\alpha^*}}}^{1/2}X_{1,i}\right\|_{\sL^2}^2\II_{\sU_1^c}\right]\EE\Big[\epsilon_{1,i}^2\Big]\\
	\overset{(\romannumeral1)}{\leq}&\sigma^2\lambda^{-1}\frac{1}{n^2}\sum_{i=1}^n\EE\Bigg[\left\|\left(\lambda \cI+\cT_{\alpha^*,\bx_1}\right)^{1/2}\Psi_\lambda\left(\cT_{\alpha^*,\bx_1}\right)\left(\lambda \cI+\cT_{\alpha^*,\bx_1}\right)^{1/2}\right\|^2\left\|\left(\lambda \cI+\cT_{\alpha^*,\bx_1}\right)^{-1/2}\left(\lambda \cI+\cT_{\alpha^*}\right)^{1/2}\right\|^2\\
	&\times \left\|\left(\lambda \cI+ \cT_{\alpha^*} \right)^{-1/2}\cL_{\cK_{{\alpha^*}}}^{1/2}X_{1,i}\right\|_{\sL^2}^2\II_{\sU_1^c}\Bigg]\overset{(\romannumeral2)}{\leq} 2E^2\sigma^2\lambda^{-1}\frac{1}{n^2}\sum_{i=1}^n\EE\left[\left\|\left(\lambda \cI+ \cT_{\alpha^*} \right)^{-1/2}\cL_{\cK_{{\alpha^*}}}^{1/2}X_{1,i}\right\|_{\sL^2}^2\right]\\
	=&2E^2\sigma^2\lambda^{-1}\frac{1}{n^2}\sum_{i=1}^n\sum_{j=1}^\infty\frac{1}{\lambda + \mu_{\alpha^*,j}}\EE\left[\left\langle \cL_{\cK_{{\alpha^*}}}^{1/2}X_{1,i}, \phi_{\alpha^*,j}\right\rangle_{\sL^2}^2\right]\\
	=&2E^2\sigma^2\lambda^{-1}\frac{1}{n}\sum_{j=1}^\infty\frac{\mu_{\alpha^*,j}}{\lambda + \mu_{\alpha^*,j}}=2E^2\sigma^2\lambda^{-1}\frac{M\sN_{\alpha^*}(\lambda)}{N},
\end{align*}
where $\{\left(\mu_{\alpha^*,j},\phi_{\alpha^*,j}\right)\}_{j=1}^\infty$ is given by the singular value decomposition of $\cT_{{\alpha^*}}$ in \eqref{singular value decomposition}, $\sN_{\alpha^*}(\lambda)$ is the effective dimension given by \eqref{effectivedimension}, inequality $(\romannumeral1)$ is due to Assumption \ref{assumption: epsilon2}, inequality $(\romannumeral2)$ uses \eqref{equation: upper bound of F_1 1}
and the fact that following from \eqref{equation: inverse norm under event}, there holds
\begin{align*}
	\left\|\left(\lambda \cI+\cT_{\alpha^*,\bx_1}\right)^{-1/2}\left(\lambda \cI+\cT_{\alpha^*}\right)^{1/2}\right\|^2\II_{\sU_1^c}= \left\|\left(\lambda \cI+\cT_{\alpha^*}\right)^{1/2}\left(\lambda \cI+\cT_{\alpha^*,\bx_1}\right)^{-1}\left(\lambda \cI+\cT_{\alpha^*}\right)^{1/2}\right\|\II_{\sU_1^c}\leq 2
\end{align*}
We have gotten inequality \eqref{equation: rates upper bound24}. The proof of Lemma \ref{lemma: rates upper bound2} is then completed.
\end{proof}
The following lemma estimates the probability of event $\sU_1$. Recall that $\sU_1$ is defined as
\[\sU_1 = \left\{\bx_1: \left\|(\lambda \cI + \cT_{{\alpha^*}})^{-1/2}(\cT_{{\alpha^*},\bx_1}- \cT_{{\alpha^*}})(\lambda \cI + \cT_{{\alpha^*}})^{-1/2}\right\| \geq 1/2\right\}.\]
\begin{lemma}\label{lemma: basic probability estimation of U1}
	Suppose that Assumption \ref{assumption: X1} is satisfied. Then for any $\lambda>0$, there holds
	\begin{align}\label{equation: basic probability estimation of U1}
		\PP(\sU_1)\leq c_3\left(\lambda^{-2}m^{-2{{\alpha^*}}+1} + \frac{M\sN_{\alpha^*}^2(\lambda)}{N}\right),
	\end{align}
	where $c_3$ is a constant that will be specified in the proof.
\end{lemma}
\begin{proof}
	Recalling the notation \eqref{simple notation}, using the triangular inequality, we write
	\begin{align}\label{equation: basic probability estimation of U1 1}
		&\EE\left[\left\|(\lambda \cI + \cT_{{\alpha^*}})^{-1/2}\left(\cT_{{\alpha^*}}-\cT_{{\alpha^*},\bx_1}\right)(\lambda \cI + \cT_{{\alpha^*}})^{-1/2}\right\|^2\right]\nonumber\\
		\leq& 2\EE\Big[\mathcal{D}_1^2(\bx_1,\alpha^*,\lambda)\Big] + 2\EE\Big[\mathcal{D}_2^2(\bx_1,\alpha^*,\lambda)\Big],
	\end{align}
	where we define
	\begin{align*}
		\mathcal{D}_1(\bx_1,\alpha^*,\lambda)&:= \left\|(\lambda \cI + \cT_{{\alpha^*}})^{-1/2}\left( \cT_{{\alpha^*},\bx_1} - \frac{1}{n}\sum_{i=1}^{n}\cL_{\cK_{{\alpha^*}}}^{1/2}X_{1,i}\otimes \cL_{\cK_{{\alpha^*}}}^{1/2}X_{1,i}\right)(\lambda \cI + \cT_{{\alpha^*}})^{-1/2}\right\|;\\
		\mathcal{D}_2(\bx_1,\alpha^*,\lambda)&:= \left\|(\lambda \cI + \cT_{{\alpha^*}})^{-1/2}\left( \frac{1}{n}\sum_{i=1}^{n}\cL_{\cK_{{\alpha^*}}}^{1/2}X_{1,i}\otimes \cL_{\cK_{{\alpha^*}}}^{1/2}X_{1,i} - \cT_{{\alpha^*}} \right)(\lambda \cI + \cT_{{\alpha^*}})^{-1/2}\right\|.
	\end{align*}
	For the term $\EE\left[\mathcal{D}_1^2(\bx_1,\alpha^*,\lambda)\right]$, recalling that 
	\begin{align*}
		\cT_{{\alpha^*},\bx_1}=\frac{1}{n}\sum_{i=1}^n \mathcal{S}_{1,i}(\cK_{{\alpha^*}},\bx_1)\otimes \mathcal{S}_{1,i}(\cK_{{\alpha^*}},\bx_1),
	\end{align*} 
	and using the triangular inequality, we write
		\begin{align*}
		&\EE\Big[\mathcal{D}_1^2(\bx_1,\alpha^*,\lambda)\Big]\\
		\overset{(\romannumeral1)}{\leq}&\EE\left[\left\|(\lambda \cI + \cT_{{\alpha^*}})^{-1/2}\left( \cT_{{\alpha^*},\bx_1} - \frac{1}{n}\sum_{i=1}^{n}\cL_{\cK_{{\alpha^*}}}^{1/2}X_{1,i}\otimes \cL_{\cK_{{\alpha^*}}}^{1/2}X_{1,i}\right)(\lambda \cI + \cT_{{\alpha^*}})^{-1/2}\right\|_{\sF}^2\right]\\
		\overset{(\romannumeral2)}{\leq}& \lambda^{-2}\EE\left[\left\|\frac{1}{n}\sum_{i=1}^n \mathcal{S}_{1,i}(\cK_{{\alpha^*}},\bx_1)\otimes \mathcal{S}_{1,i}(\cK_{{\alpha^*}},\bx_1) - \frac{1}{n}\sum_{i=1}^{n}\cL_{\cK_{{\alpha^*}}}^{1/2}X_{1,i}\otimes \cL_{\cK_{{\alpha^*}}}^{1/2}X_{1,i}\right\|_{\sF}^2\right]\\
		\leq& \lambda^{-2}\frac{1}{n}\sum_{i=1}^{n} \EE\left[\left\|\mathcal{S}_{1,i}(\cK_{{\alpha^*}},\bx_1)\otimes \mathcal{S}_{1,i}(\cK_{{\alpha^*}},\bx_1) - \cL_{\cK_{{\alpha^*}}}^{1/2}X_{1,i}\otimes \cL_{\cK_{{\alpha^*}}}^{1/2}X_{1,i}\right\|_{\sF}^2\right]\\
		\leq&\lambda^{-2}\frac{2}{n}\sum_{i=1}^{n}\EE\left[\left\|\left(\mathcal{S}_{1,i}(\cK_{{\alpha^*}},\bx_1) - \cL_{\cK_{{\alpha^*}}}^{1/2}X_{1,i}\right)\otimes\mathcal{S}_{1,i}(\cK_{{\alpha^*}},\bx_1)\right\|_{\sF}^2\right]\\
		&+ \lambda^{-2}\frac{2}{n}\sum_{i=1}^{n}\EE\left[\left\|\cL_{\cK_{\alpha^*}}^{1/2}X_{1,i}\otimes\left(\mathcal{S}_{1,i}(\cK_{{\alpha^*}},\bx_1) - \cL_{\cK_{{\alpha^*}}}^{1/2}X_{1,i}\right)\right\|_{\sF}^2\right]\\
		\overset{(\romannumeral3)}{\leq}& \lambda^{-2}\frac{2}{n}\sum_{i=1}^{n}\EE\left[\left\|\mathcal{S}_{1,i}(\cK_{{\alpha^*}},\bx_1) - \cL_{\cK_{{\alpha^*}}}^{1/2}X_{1,i}\right\|_{\cL^2}^2\left\{\Big\|\mathcal{S}_{1,i}(\cK_{{\alpha^*}},\bx_1)\Big\|_{\sL^2}^2+\left\|\cL_{\cK_{\alpha^*}}^{1/2}X_{1,i}\right\|_{\sL^2}^2\right\}\right],
	\end{align*}
	where inequality $(\romannumeral2)$ uses the relationship \eqref{relationship between L2 and HS norm}, inequality $(\romannumeral2)$ is due to \eqref{equation: HS norm of product operator} and inequality $(\romannumeral3)$ follows from \eqref{equation: HS norm of rank-one operator}.
	
	While for any $1\leq i\leq n$, we have estimated $\left\|\mathcal{S}_{1,i}(\cK_{{\alpha^*}},\bx_1) - \cL_{\cK_{{\alpha^*}}}^{1/2}X_{1,i}\right\|_{\cL^2}$ in \eqref{prooflemma 1} as
	\begin{align*}
		\left\|\mathcal{S}_{1,i}(\cK_{{\alpha^*}},\bx_1)-\cL_{\cK_{{\alpha^*}}}^{1/2}X_{1,i}\right\|_{\sL^2}\leq C_1C_{\alpha^*}m^{-\alpha^*+1/2}\left\|X_{1,i}\right\|_{\sW^{\alpha^*,2}},
	\end{align*}
    and we can estimate the terms $\left\|\cL_{\cK_{\alpha^*}}^{1/2}X_{1,i}\right\|_{\sL^2}$ and $\Big\|\mathcal{S}_{1,i}(\cK_{{\alpha^*}},\bx_1)\Big\|_{\sL^2}$ as
    \begin{align*}
    	\left\|\cL_{\cK_{\alpha^*}}^{1/2}X_{1,i}\right\|_{\sL^2}\overset{(\romannumeral1)}{\leq} C_2\left\|\cL_{\cK_{\alpha^*}}^{1/2}X_{1,i}\right\|_{\sW^{\alpha^*,2}}\overset{(\romannumeral2)}{=} C_2\left\|X_{1,i}\right\|_{\sL^2}\overset{(\romannumeral3)}{\leq} C_2^2\left\|X_{1,i}\right\|_{\sW^{\alpha^*,2}},
    \end{align*}
    and 
    \begin{align*}
    	\Big\|\mathcal{S}_{1,i}(\cK_{{\alpha^*}},\bx_1)\Big\|_{\sL^2}\leq& \left\|\mathcal{S}_{1,i}(\cK_{{\alpha^*}},\bx_1) - \cL_{\cK_{{\alpha^*}}}^{1/2}X_{1,i}\right\|_{\cL^2} + \left\|\cL_{\cK_{\alpha^*}}^{1/2}X_{1,i}\right\|_{\sL^2}\\
    	\leq& \left(C_2+ C_1^2C_{\alpha^*}m^{-\alpha^*+1/2}\right)\left\|X_{1,i}\right\|_{\sW^{\alpha^*,2}}\leq \left(C_2^2+C_1C_{\alpha^*}\right)\left\|X_{1,i}\right\|_{\sW^{\alpha^*,2}},
    \end{align*}
    where inequalities $(\romannumeral1)$ and $(\romannumeral3)$ follow from \eqref{equation: Sobolev inequalities} in Lemma \ref{lemma: Sobolev inequalities}, equality $(\romannumeral2)$ uses \eqref{equation: norm and inner product relationship}.
    
    Therefore we write
	\begin{align}\label{equation: basic probability estimation of U1 2}
		\EE\left[\mathcal{D}_1^2(\bx_1,\alpha^*,\lambda)\right]
		\leq& 4C_1^2C_{\alpha^*}^2\left(C_2^2+C_1C_{\alpha^*}\right)^2\lambda^{-2}m^{-2\alpha^*+1}\frac{1}{n}\sum_{i=1}^{n}\EE\left[\left\|X_{1,i}\right\|_{\sW^{\alpha^*,2}}^4\right]\nonumber\\
		\overset{(*)}{\leq}& 4\rho\kappa^4C_1^2C_{\alpha^*}^2\left(C_2^2+C_1C_{\alpha^*}\right)^2\lambda^{-2}m^{-2\alpha^*+1},
	\end{align}
	where inequality $(*)$ follows from \eqref{equation: 4-th Walpha,2-norm of X}.
	
	For the term $\EE\left[\mathcal{D}_2^2(\bx_1,\alpha^*,\lambda)\right]$, noting that for any $1\leq i\leq n$, 
	\begin{align*}
		\EE\left[(\lambda \cI + \cT_{{\alpha^*}})^{-1/2}\left( \cL_{\cK_{{\alpha^*}}}^{1/2}X_{1,i}\otimes \cL_{\cK_{{\alpha^*}}}^{1/2}X_{1,i} - \cT_{{\alpha^*}} \right)(\lambda \cI + \cT_{{\alpha^*}})^{-1/2}\right]=0,
	\end{align*} 
	and using the relationship \eqref{relationship between L2 and HS norm}, we write
	\begin{align*}
			&\EE\left[\mathcal{D}_2^2(\bx_1,\alpha^*,\lambda)\right]\\
		\leq& \EE\left[\left\|(\lambda \cI + \cT_{{\alpha^*}})^{-1/2}\left( \frac{1}{n}\sum_{i=1}^{n}\cL_{\cK_{{\alpha^*}}}^{1/2}X_{1,i}\otimes \cL_{\cK_{{\alpha^*}}}^{1/2}X_{1,i} - \cT_{{\alpha^*}} \right)(\lambda \cI + \cT_{{\alpha^*}})^{-1/2}\right\|_{\sF}^2\right]\\
		=&\frac{1}{n^2}\sum_{i=1}^{n}\EE\left[\left\|(\lambda \cI + \cT_{{\alpha^*}})^{-1/2}\left( \cL_{\cK_{{\alpha^*}}}^{1/2}X_{1,i}\otimes \cL_{\cK_{{\alpha^*}}}^{1/2}X_{1,i} - \cT_{{\alpha^*}} \right)(\lambda \cI + \cT_{{\alpha^*}})^{-1/2}\right\|_{\sF}^2\right].
	\end{align*}
	While for any $1\leq i\leq n$, we have
	\begin{align*}
		&\EE\left[\left\|(\lambda \cI + \cT_{{\alpha^*}})^{-1/2}\left( \cL_{\cK_{{\alpha^*}}}^{1/2}X_{1,i}\otimes \cL_{\cK_{{\alpha^*}}}^{1/2}X_{1,i} - \cT_{{\alpha^*}} \right)(\lambda \cI + \cT_{{\alpha^*}})^{-1/2}\right\|_{\sF}^2\right]\nonumber\\
		=& \sum_{j=1}^{\infty}\sum_{k=1}^{\infty}\EE\left[\left\langle(\lambda \cI + \cT_{{\alpha^*}})^{-1/2}(\cL_{\cK_{{\alpha^*}}}^{1/2}X_{1,i}\otimes \cL_{\cK_{{\alpha^*}}}^{1/2}X_{1,i} - \cT_{{\alpha^*}})(\lambda \cI + \cT_{{\alpha^*}})^{-1/2}\phi_{\alpha^*,j},\phi_{\alpha^*,k}\right\rangle^2_{\sL^2} \right]\nonumber\\
		\overset{(\romannumeral1)}{\leq}& \sum_{j=1}^{\infty}\sum_{k=1}^{\infty}\frac{1}{\lambda + \mu_{\alpha^*,j}}\frac{1}{\lambda + \mu_{\alpha^*,k}}\EE\left[\left\langle \cL_{\cK_{{\alpha^*}}}^{1/2}X_{1,i},\phi_{\alpha^*,j}\right\rangle^2_{\sL^2}\left\langle \cL_{\cK_{{\alpha^*}}}^{1/2}X_{1,i},\phi_{\alpha^*,k}\right\rangle^2_{\sL^2}\right]\nonumber\\
		\overset{(\romannumeral2)}{\leq}& \sum_{j=1}^{\infty}\sum_{k=1}^{\infty}\frac{1}{\lambda + \mu_{\alpha^*,j}}\frac{1}{\lambda + \mu_{\alpha^*,k}}\left[\EE\left\langle \cL_{\cK_{{\alpha^*}}}^{1/2}X_{1,i},\phi_{\alpha^*,j}\right\rangle^4_{\sL^2}\right]^{\frac{1}{2}}\left[\EE\left\langle \cL_{\cK_{{\alpha^*}}}^{1/2}X_{1,i},\phi_{\alpha^*,k}\right\rangle^4_{\sL^2}\right]^{\frac{1}{2}}\nonumber\\
		\overset{(\romannumeral3)}{\leq}& \rho\sum_{j=1}^{\infty}\sum_{k=1}^{\infty}\frac{1}{\lambda + \mu_{\alpha^*,j}}\frac{1}{\lambda + \mu_{\alpha^*,k}}\EE\left[\left\langle \cL_{\cK_{{\alpha^*}}}^{1/2}X_{1,i},\phi_{\alpha^*,j}\right\rangle^2_{\sL^2}\right]\EE\left[\left\langle \cL_{\cK_{{\alpha^*}}}^{1/2}X_{1,i},\phi_{\alpha^*,k}\right\rangle^2_{\sL^2}\right]\nonumber\\
		=&\rho\sum_{j=1}^{\infty}\sum_{k=1}^{\infty}\frac{1}{\lambda + \mu_{\alpha^*,j}}\frac{1}{\lambda + \mu_{\alpha^*,k}}\left\langle \cT_{{\alpha^*}}\phi_{\alpha^*,j},\phi_{\alpha^*,j} \right\rangle_{\sL^2}\left\langle \cT_{{\alpha^*}}\phi_{\alpha^*,k},\phi_{\alpha^*,k} \right\rangle_{\sL^2}\nonumber\\
		=&\rho\sum_{j=1}^{\infty}\sum_{k=1}^{\infty}\frac{\mu_{\alpha^*,j}}{\lambda + \mu_{\alpha^*,j}}\frac{\mu_{\alpha^*,k}}{\lambda + \mu_{\alpha^*,k}}=\rho\sN_{\alpha^*}^2(\lambda),
	\end{align*}
	where $\{\left(\mu_{\alpha^*,j},\phi_{\alpha^*,j}\right)\}_{j=1}^\infty$ is given by the singular value decomposition of $\cT_{{\alpha^*}}$ in \eqref{singular value decomposition}, inequality $(\romannumeral1)$ is from the fact that for any $1\leq i\leq n$, $L_K^{1/2}X_{1,i}\otimes L_K^{1/2}X_{1,i} - \cT_{{\alpha^*}}$ is a zero-mean random variable, inequality $(\romannumeral2)$ uses Cauchy-Schwartz inequality and inequality $(\romannumeral3)$ applies \eqref{equation: X1} in Assumption \ref{assumption: X1}.
	
Therefore, we write
	\begin{align}\label{equation: basic probability estimation of U1 3}
		\EE\left[\mathcal{D}_2^2(\bx_1,\alpha^*,\lambda)\right]\leq \rho\frac{\sN_{\alpha^*}^2(\lambda)}{n}.
	\end{align}
	Recalling that $n=N/M$ and combining \eqref{equation: basic probability estimation of U1 1}, \eqref{equation: basic probability estimation of U1 2} and \eqref{equation: basic probability estimation of U1 3}, we have
	\begin{align}\label{equation: basic probability estimation of U1 4}
		&\EE\left[\left\|(\lambda \cI + \cT_{{\alpha^*}})^{-1/2}\left(\cT_{{\alpha^*}}-\cT_{{\alpha^*},\bx_1}\right)(\lambda \cI + \cT_{{\alpha^*}})^{-1/2}\right\|^2\right]\nonumber\\
		\leq&2\EE\left[\mathcal{D}_1^2(\bx_1,\alpha^*,\lambda)\right] + 2\EE\left[\mathcal{D}_2^2(\bx_1,\alpha^*,\lambda)\right]\nonumber\\
		\leq&8\rho\kappa^4C_1^2C_{\alpha^*}^2\left(C_2^2+C_1C_{\alpha^*}\right)^2\lambda^{-2}m^{-2\alpha^*+1} + 2\rho\frac{M\sN_{\alpha^*}^2(\lambda)}{N}.
	\end{align}
	Then using Chebyshev's inequality, we have
	\begin{align*}
		\PP(\sU_1)=&\PP\left(\left\{\bx_1: \left\|(\lambda \cI + \cT_{{\alpha^*}})^{-1/2}(\cT_{{\alpha^*},\bx_1}- \cT_{{\alpha^*}})(\lambda \cI + \cT_{{\alpha^*}})^{-1/2}\right\|\geq  1/2\right\}\right)\\
		\leq& 4\EE\left[\left\|(\lambda \cI + \cT_{{\alpha^*}})^{-1/2}\left(\cT_{{\alpha^*}}-\cT_{{\alpha^*},\bx_1}\right)(\lambda \cI + \cT_{{\alpha^*}})^{-1/2}\right\|^2\right]\\
		\leq& 32\rho\kappa^4C_1^2C_{\alpha^*}^2\left(C_2^2+C_1C_{\alpha^*}\right)^2\lambda^{-2}m^{-2\alpha^*+1} + 8\rho\frac{M\sN_{\alpha^*}^2(\lambda)}{N}\leq c_3\left(\lambda^{-2}m^{-2{{\alpha^*}}+1} + \frac{M\sN_{\alpha^*}^2(\lambda)}{N}\right),
	\end{align*}
	where we define $c_3:=32\rho\kappa^4C_1^2C_{\alpha^*}^2\left(C_2^2+C_1C_{\alpha^*}\right)^2+8\rho$. That's the desired result.
\end{proof}

The following lemma provides an estimate of $\sN_{\alpha^*}(\lambda)$ under the polynomial decaying condition of the eigenvalues.
\begin{lemma}\label{lemma: estimation of N(lambda)}
	Suppose that $\{\mu_{\alpha^*,j}\}_{j\geq1}$ satisfy $\mu_{\alpha^*,j}\lesssim j^{-1/p}$ for some $0<p \leq1$, then
	\begin{equation}\label{estimation of N(lambda)}
		\sN_{\alpha^*}(\lambda) \lesssim \lambda^{-p}, \quad \forall 0<\lambda\leq 1.
	\end{equation} 
\end{lemma}
The proof of Lemma \ref{lemma: estimation of N(lambda)} can be found in \cite{guo2017learning,guo2019optimal,lin2017distributed}.

We have established all necessary preliminaries to prove Theorem \ref{theorem: upper bound1}. Before proceeding with the proof, we will introduce the notations $o_{_\PP}(\cdot)$ and $\mathcal{O}_{\PP}(\cdot)$ for the sake of simplicity. For a sequence of random variables $\{\xi_j\}_{j=1}^\infty$, we write $\xi_j\leq o_{_\PP}(1)$ if \begin{align*}
	\lim_{j\rightarrow\infty}\PP\left(\left|\xi_k\right|\geq d\right)=0, \forall d>0.
\end{align*} We write $\xi_j\leq \mathcal{O}_{\PP}(1)$ if \begin{align*}
	\lim_{D\rightarrow \infty}\sup_{j\geq 1}\PP\left(\left|\xi_j\right|\geq D\right)=0. 
\end{align*}
Additionally, suppose that there exists a positive sequence $\{a_j\}_{j=1}^\infty$. Then we write
$\xi_j\leq o_{_\PP}(a_j)$ if $\xi_j/a_j \leq o_{_\PP}(1)$, and $\xi_j\leq \mathcal{O}_{\PP}(a_j)$ if $\xi_j/a_j \leq \mathcal{O}_{\PP}(1)$.

\noindent
\emph{Proof of Theorem \ref{theorem: upper bound1}}.
In the proof, we let $\lambda\in (0,1)$ and $0\leq \theta\leq \nu_{\Psi}$. We first decompose the estimation error $\left\|\overline{\beta}_{S,\alpha^*,\Psi_\lambda}-\beta_0\right\|_{\sW^{\alpha^*,2}}^2$ as 
\begin{align}\label{prooftheorem upper bound 1}
	\left\|\overline{\beta}_{S,\alpha^*,\Psi_\lambda}-\beta_0\right\|_{\sW^{\alpha^*,2}}^2 = \left\|\overline{\beta}_{S,\alpha^*,\Psi_\lambda}-\beta_0\right\|_{\sW^{\alpha^*,2}}^2\II_{\sU} + \left\|\overline{\beta}_{S,\alpha^*,\Psi_\lambda}-\beta_0\right\|_{\sW^{\alpha^*,2}}^2\II_{\sU^c}.
\end{align}
For the term $\left\|\overline{\beta}_{S,\alpha^*,\Psi_\lambda}-\beta_0\right\|_{\sW^{\alpha^*,2}}^2\II_{\sU}$, following from \eqref{equation: basic probability estimation of U1} in Lemma \ref{lemma: basic probability estimation of U1} and \eqref{estimation of N(lambda)} in Lemma \ref{lemma: estimation of N(lambda)}, we write
\begin{align*}
	\EE\left[\II_{\sU}\right]=\PP(\sU)\leq \sum_{j=1}^M\PP(\sU_j) = M\PP(\sU_1)\lesssim M\lambda^{-2}m^{-2{{\alpha^*}}+1}+ \frac{M^2\lambda^{-2p}}{N}.
\end{align*}
Then using Markov's inequality, we write
\begin{align}\label{prooftheorem upper bound 2}
	\left\|\overline{\beta}_{S,\alpha^*,\Psi_\lambda}-\beta_0\right\|_{\sW^{\alpha^*,2}}^2\II_{\sU}\leq& \mathcal{O}_{\PP}\left(M\lambda^{-2}m^{-2{{\alpha^*}}+1}+ \frac{M^2\lambda^{-2p}}{N}\right)\left\|\overline{\beta}_{S,\alpha^*,\Psi_\lambda}-\beta_0\right\|_{\sW^{\alpha^*,2}}^2.
\end{align}

For the term $\left\|\overline{\beta}_{S,\alpha^*,\Psi_\lambda}-\beta_0\right\|_{\sW^{\alpha^*,2}}^2\II_{\sU^c}$, following from \eqref{equation: rates upper bound1 1}, \eqref{equation: decomposition 1} and \eqref{equation: decomposition 2}, we write
\begin{align*}
	\EE\left[\left\|\overline{\beta}_{S,\alpha^*,\Psi_\lambda}-\beta_0\right\|_{\sW^{\alpha^*,2}}^2\II_{\sU^c}\right]\leq& \left(2+\frac{4}{M}\right)\Big(\EE\left[\mathcal{F}_1(\bx_1,\alpha^*,\Psi_\lambda)\II_{\sU_1^c}\right]+\EE\left[\mathcal{F}_2(\bx_1,\alpha^*,\Psi_\lambda)\II_{\sU_1^c}\right] \Big)\\
	&+\frac{4}{M}\left(\EE\Big[\mathcal{F}_3(S_1,\alpha^*,\Psi_\lambda)\II_{\sU_1^c}\Big]+\EE\Big[\mathcal{F}_4(S_1,\alpha^*,\Psi_\lambda)\II_{\sU_1^c}\Big]\right).
\end{align*}
Then using \eqref{equation: rates upper bound21}, \eqref{equation: rates upper bound22}, \eqref{equation: rates upper bound23} and \eqref{equation: rates upper bound24} in Lemma \ref{lemma: rates upper bound2} and \eqref{estimation of N(lambda)} in Lemma \ref{lemma: estimation of N(lambda)}, we have
\begin{align*}
\EE\left[\left\|\overline{\beta}_{S,\alpha^*,\Psi_\lambda}-\beta_0\right\|_{\sW^{\alpha^*,2}}^2\II_{\sU^c}\right] \lesssim\lambda^{2\theta} +\lambda^{-2}m^{-2\alpha^*+1} + \lambda^{-1}\frac{\sN_{\alpha^*}(\lambda)}{N}\lesssim \lambda^{2\theta} +\lambda^{-2}m^{-2\alpha^*+1} + \frac{\lambda^{-1-p}}{N}
\end{align*}
Combining the above estimate with Markov's inequality, we write
\begin{align}\label{prooftheorem upper bound 3}
	\left\|\overline{\beta}_{S,\alpha^*,\Psi_\lambda}-\beta_0\right\|_{\sW^{\alpha^*,2}}^2\II_{\sU^c}\leq \mathcal{O}_{\PP}\left(\lambda^{2\theta}+ \lambda^{-2}m^{-2\alpha^*+1} + \frac{\lambda^{-1-p}}{N}\right).
\end{align}
Therefore, combining \eqref{prooftheorem upper bound 1}, \eqref{prooftheorem upper bound 2} and \eqref{prooftheorem upper bound 3} yields
\begin{align*}
	\left[1-\mathcal{O}_{\PP}\left(M\lambda^{-2}m^{-2{{\alpha^*}}+1}+ \frac{M^2\lambda^{-2p}}{N}\right)\right]\left\|\overline{\beta}_{S,\alpha^*,\Psi_\lambda}-\beta_0\right\|_{\sW^{\alpha^*,2}}^2\leq \mathcal{O}_{\PP}\left(\lambda^{2\theta}+ \lambda^{-2}m^{-2\alpha^*+1} + \frac{\lambda^{-1-p}}{N}\right).
\end{align*}
Then taking $\lambda = N^{-\frac{1}{1+2\theta+p}}$, $m\geq N^{\frac{2+2\theta}{(2{{\alpha^*}}-1)(1+2\theta+p)}}$ and $M\leq o\left(\min\left\{N^{\frac{2\theta}{1+2\theta+p}},N^{\frac{1-p+2\theta}{2(1+2\theta+p)}}\right\}\right)$, we have
\begin{align*}
	M\lambda^{-2}m^{-2{{\alpha^*}}+1}+ \frac{M^2\lambda^{-2p}}{N}\leq o(1),
\end{align*}
and
\begin{align*}
\lambda^{2\theta}+ \lambda^{-2}m^{-2\alpha^*+1} + \frac{\lambda^{-1-p}}{N}\lesssim N^{-\frac{2\theta}{1+2\theta+p}},
\end{align*}
and thus
\begin{align*}
	\left\|\overline{\beta}_{S,\alpha^*,\Psi_\lambda}-\beta_0\right\|_{\sW^{\alpha^*,2}}^2\leq \mathcal{O}_{\PP}\left(N^{-\frac{2\theta}{1+2\theta+p}}\right).
\end{align*}
This is equivalent to
\begin{align*}
		\lim_{\Gamma \to 0}\mathop{\lim\sup}_{N\to \infty} \sup_{\beta_0} \mathbb{P} \left\{\left\|\overline{\beta}_{S,{\alpha^*},\Psi_\lambda}-\beta_0\right\|_{\sW^{{\alpha^*},2}}^2\geq \Gamma N^{-\frac{2\theta}{1+2\theta+p}}\right\} = 0,
\end{align*}
provided that $\lambda = N^{-\frac{1}{1+2\theta+p}}$, $m\geq N^{\frac{2+2\theta}{(2{{\alpha^*}}-1)(1+2\theta+p)}}$ and $M\leq o\left(\min\left\{N^{\frac{2\theta}{1+2\theta+p}},N^{\frac{1-p+2\theta}{2(1+2\theta+p)}}\right\}\right)$. 

We have completed the proof of Theorem \ref{theorem: upper bound1}.\qed

We next turn to prove Theorem \ref{theorem: upper bound2}. We first propose the following lemma to bound the expectation of estimation error $\left\|\overline{\beta}_{S,\alpha^*,\Psi_\lambda}-\beta_0\right\|_{\sW^{\alpha^*,2}}^2$. The proof of this lemma can be obtained by imitating the proof of Lemma \ref{lemma: rates upper bound1}.
\begin{lemma}\label{lemma: expecatation upper bound 1}
	Suppose that Assumption \ref{assumption: regularity condition} is satisfied. Then for any partition number $M\geq 1$, there holds
	\begin{equation}\label{equation: expecatation upper bound 1} 
		\mathbb{E}\left[\left\|\overline{\beta}_{S,\alpha^*,\Psi_\lambda}-\beta_0\right\|_{\sW^{\alpha^*,2}}^2\right]\leq \frac{1}{M}\mathbb{E}\left[\left\|\hat{f}_{S_1,\alpha^*,\Psi_\lambda}-f_0\right\|^2_{\sL^2}\right]+ \left\|\mathbb{E}\left[\hat{f}_{S_1,\alpha^*,\Psi_\lambda}-f_0\right]\right\|^2_{\sL^2},
	\end{equation}
where $f_0\in \sL^2(\sT)$ is given by Assumption \ref{assumption: regularity condition}. 
\end{lemma} 
Following from the same arguments of \eqref{equation: decomposition 1} and \eqref{equation: decomposition 2}, we can bound the two terms on the right hand side of \eqref{equation: expecatation upper bound 1} as
\begin{align}\label{equation: upper bound of expectation1}
	&\frac{1}{M}\mathbb{E}\left[\left\|\hat{f}_{S_1,\alpha^*,\Psi_\lambda}-f_0\right\|^2_{\sL^2}\right]\\
	\leq&\frac{4}{M}\EE\Big[\mathcal{F}_1\left(\bx_1,\alpha^*,\Psi_\lambda\right)\Big] + \frac{4}{M}\EE\Big[\mathcal{F}_2\left(\bx_1,\alpha^*,\Psi_\lambda\right)\Big]+ \frac{4}{M}\EE\Big[\mathcal{F}_3\left(S_1,\alpha^*,\Psi_\lambda\right)\Big]+\frac{4}{M}\EE\Big[\mathcal{F}_4\left(S_1,\alpha^*,\Psi_\lambda\right)\Big]\nonumber\\
	=& \frac{4}{M}\EE\Big[\mathcal{F}_1\left(\bx_1,\alpha^*,\Psi_\lambda\right)\Big] + \frac{4}{M}\EE\Big[\mathcal{F}_3\left(S_1,\alpha^*,\Psi_\lambda\right)\Big] + \frac{4}{M}\EE\Big[\mathcal{F}_2\left(\bx_1,\alpha^*,\Psi_\lambda\right)\II_{\sU_1^c}\Big]\nonumber\\
	 &+ \frac{4}{M}\EE\Big[\mathcal{F}_2\left(\bx_1,\alpha^*,\Psi_\lambda\right)\II_{\sU_1}\Big] + \frac{4}{M}\EE\Big[\mathcal{F}_4\left(S_1,\alpha^*,\Psi_\lambda\right)\II_{\sU_1^c}\Big] + \frac{4}{M}\EE\Big[\mathcal{F}_4\left(S_1,\alpha^*,\Psi_\lambda\right)\II_{\sU_1}\Big]\nonumber
\end{align}
and
\begin{equation}\label{equation: upper bound of expectation2}
	\begin{split}
		&\left\|\mathbb{E}\left[\hat{f}_{S_1,\alpha^*,\Psi_\lambda}-f_0\right]\right\|^2_{\sL^2} \leq 2\EE\Big[\mathcal{F}_1\left(\bx_1,\alpha^*,\Psi_\lambda\right)\Big] + 2\EE\Big[\mathcal{F}_2\left(\bx_1,\alpha^*,\Psi_\lambda\right)\Big]\\
	=& 2\EE\Big[\mathcal{F}_1\left(\bx_1,\alpha^*,\Psi_\lambda\right)\Big] +2\EE\Big[\mathcal{F}_2\left(\bx_1,\alpha^*,\Psi_\lambda\right)\II_{\sU_1^c}\Big] + 2\EE\Big[\mathcal{F}_2\left(\bx_1,\alpha^*,\Psi_\lambda\right)\II_{\sU_1}\Big],
	\end{split}
\end{equation}
where the terms $\mathcal{F}_1\left(\bx_1,\alpha^*,\Psi_\lambda\right)$, $\mathcal{F}_2\left(\bx_1,\alpha^*,\Psi_\lambda\right)$, $\mathcal{F}_3\left(S_1,\alpha^*,\Psi_\lambda\right)$ and $\mathcal{F}_4\left(S_1,\alpha^*,\Psi_\lambda\right)$ are defined by \eqref{equation: decomposition 1}.

We have estimated $\EE\Big[\mathcal{F}_1\left(\bx_1,\alpha^*,\Psi_\lambda\right)\Big]$, $\EE\Big[\mathcal{F}_2\left(\bx_1,\alpha^*,\Psi_\lambda\right)\II_{\sU_1^c}\Big]$, $\EE\Big[\mathcal{F}_3\left(S_1,\alpha^*,\Psi_\lambda\right)\Big]$ and $\EE\Big[\mathcal{F}_4\left(S_1,\alpha^*,\Psi_\lambda\right)\II_{\sU_1^c}\Big]$ in Lemma \ref{lemma: rates upper bound2} under Assumption \ref{assumption: regularity condition}, \ref{assumption: epsilon2}, \ref{assumption: sampling scheme} and \ref{assumption: X1}. As previously stated, Assumption \ref{assumption: X2} is an enhanced version of Assumption \ref{assumption: X1}. Consequently, Lemma \ref{lemma: rates upper bound2} also establishes the upper bounds for these terms when Assumption \ref{assumption: X1} is enhanced to Assumption \ref{assumption: X2}. The following lemma provide upper bounds for the remaining two terms $\EE\Big[\mathcal{F}_2\left(\bx_1,\alpha^*,\Psi_\lambda\right)\II_{\sU_1}\Big]$ and $\EE\Big[\mathcal{F}_4\left(S_1,\alpha^*,\Psi_\lambda\right)\II_{\sU_1}\Big]$.
\begin{lemma}\label{lemma: expecatation upper bound 2}
	Let $\left\{\Psi_\lambda:[0,\infty) \rightarrow \mathbb{R}| \lambda\in (0,1)\right\}$ be a filter function satisfying Definition \ref{definition: filter functions} with qualification $\nu_{\Psi}\geq 1$. Suppose that Assumption \ref{assumption: regularity condition} is satisfied with $0\leq \theta\leq \nu_{\Psi}$ and $g_0 \in \sL^2(\sT)$, Assumptions \ref{assumption: epsilon2} is satisfied with $\sigma>0$, \ref{assumption: sampling scheme} with $C_d>0$ and \ref{assumption: X1} is satisfied with $\kappa>0$. Then for any $\lambda\in (0,1)$, there holds
	\begin{align}\label{equation: expecatation upper bound 2 1}
		\EE\Big[\mathcal{F}_2\left(\bx_1,\alpha^*,\Psi_\lambda\right)\II_{\sU_1}\Big]\leq E^2\rho_{\alpha^*}^{2\theta}\left\|g_0\right\|_{\sL^2}\PP(\sU_1)
	\end{align}
	and
	\begin{align}\label{equation: expecatation upper bound 2 2}
		\EE\Big[\mathcal{F}_4\left(S_1,\alpha^*,\Psi_\lambda\right)\II_{\sU_1}\Big]\leq E^2\sigma^2\mbox{Tr}^2(\cT_{\alpha^*})\lambda^{-1}\frac{M}{N}\PP^{\frac{1}{2}}(\sU_1),
	\end{align}
where $\rho_{\alpha^*}$ is a constant given by \eqref{equation: bound of expectation of L_KX} and $\mbox{Tr}(\cT_{\alpha^*})=\sum_{j=1}^{\infty}\mu_{\alpha^*,j}$ denotes the trace of $\cT_{\alpha^*}$.
\end{lemma}
\begin{proof}
	We start with the first inequality \eqref{equation: expecatation upper bound 2 1}. Recalling the expression of $\mathcal{F}_2\left(\bx_1,\alpha^*,\Psi_\lambda\right)$ that
	\begin{align*}
		\cT_{{\alpha^*},\bx_1}=\frac{1}{n}\sum_{i=1}^n \mathcal{S}_{1,i}(\cK_{{\alpha^*}},\bx_1)\otimes \mathcal{S}_{1,i}(\cK_{{\alpha^*}},\bx_1),
	\end{align*} 
we write
\begin{align*}
	&\EE\Big[\mathcal{F}_2\left(\bx_1,\alpha^*,\Psi_\lambda\right)\II_{\sU_1}\Big]\\
	=&\EE\left[\left\|\Psi_\lambda\left(\cT_{{\alpha^*},\bx_1}\right)\frac{1}{n}\sum_{i=1}^n\mathcal{S}_{1,i}(\cK_{\alpha^*},\bx_1)\left\langle \mathcal{S}_{1,i}(\cK_{\alpha^*},\bx_1),f_0\right\rangle_{\sL^2} - f_0\right\|_{\sL^2}^2\II_{\sU_1}\right]\\
	\overset{(\romannumeral1)}{=}& \EE\left[\left\|\Big(\Psi_\lambda\left(\cT_{{\alpha^*},\bx_1}\right)\cT_{\alpha^*,\bx_1}- \cI\Big)\cT_{\alpha^*}^{\theta}(g_0)\right\|_{\sL^2}^2\II_{\sU_1}\right]\\
	\leq&\EE\left[\Big\|\Psi_\lambda\left(\cT_{{\alpha^*},\bx_1}\right)\cT_{\alpha^*,\bx_1}- \cI\Big\|^2\Big\|\cT_{\alpha^*}^{\theta}(g_0)\Big\|_{\sL^2}^2\II_{\sU_1}\right]
	\overset{(\romannumeral2)}{\leq} E^2\rho_{\alpha^*}^{2\theta}\left\|g_0\right\|_{\sL^2}\PP(\sU_1),
\end{align*}
where inequality $(\romannumeral1)$ follows from Assumption \ref{assumption: regularity condition}, inequality $(\romannumeral2)$ uses \eqref{equation: upper bound of F_1 1} and \eqref{equation: operator norm of T_alpha^*}.

For the second inequality \eqref{equation: expecatation upper bound 2 2}, recalling the expression of $\mathcal{F}_4\left(S_1,\alpha^*,\Psi_\lambda\right)$ and noting that $\epsilon$ is a mean-zero random variable independent of $X$, we write
\begin{align*}
	&\EE\left[\mathcal{F}_4\left(S_1,\alpha^*,\Psi_\lambda\right)\II_{\sU_1^c}\right]\nonumber\\
	=&\EE\left[\left\|\Psi_\lambda\left(\cT_{{\alpha^*},\bx_1}\right)\frac{1}{n}\sum_{i=1}^n\cL_{\cK_{{\alpha^*}}}^{1/2}X_{1,i}\epsilon_{1,i}\right\|_{\sL^2}^2\II_{\sU_1}\right]\\
	=&\frac{1}{n^2}\sum_{i=1}^n\EE\left[\left\|\Psi_\lambda\left(\cT_{{\alpha^*},\bx_1}\right)\cL_{\cK_{{\alpha^*}}}^{1/2}X_{1,i}\right\|_{\sL^2}^2\II_{\sU_1}\right]\EE\Big[\epsilon_{1,i}^2\Big]\\
	\overset{(\romannumeral1)}{\leq}&\sigma^2\lambda^{-2}\frac{1}{n^2}\sum_{i=1}^n\EE\Bigg[\left\|\left(\lambda \cI+\cT_{\alpha^*,\bx_1}\right)^{1/2}\Psi_\lambda\left(\cT_{\alpha^*,\bx_1}\right)\left(\lambda \cI+\cT_{\alpha^*,\bx_1}\right)^{1/2}\right\|^2\left\|\cL_{\cK_{{\alpha^*}}}^{1/2}X_{1,i}\right\|_{\sL^2}^2\II_{\sU_1}\Bigg]\\
	\overset{(\romannumeral2)}{\leq}& E^2\sigma^2\lambda^{-1}\frac{1}{n^2}\sum_{i=1}^n\EE\left[\left\|\cL_{\cK_{{\alpha^*}}}^{1/2}X_{1,i}\right\|_{\sL^2}^2\II_{\sU_1}\right]\\
	\overset{(\romannumeral3)}{\leq}&E^2\sigma^2\lambda^{-1}\frac{1}{n^2}\sum_{i=1}^n\left[\EE\left\|\cL_{\cK_{{\alpha^*}}}^{1/2}X_{1,i}\right\|_{\sL^2}^4\right]^2\PP^{\frac{1}{2}}(\sU_1),
\end{align*}
where inequality $(\romannumeral1)$ is due to Assumption \ref{assumption: epsilon2}, inequality $(\romannumeral2)$ follows from \eqref{equation: upper bound of F_1 1} and inequality $(\romannumeral3)$ uses Cauchy-Schwartz inequality.

While for any $1\leq i\leq n$, we write
\begin{align*}
	&\EE\left[\left\|\cL_{\cK_{{\alpha^*}}}^{1/2}X_{1,i}\right\|_{\sL^2}^4\right]=\EE\left[\left(\sum_{j=1}^\infty\left\langle \cL_{\cK_{{\alpha^*}}}^{1/2}X_{1,i},\phi_{\alpha^*,j}\right\rangle_{\sL^2}^2\right)^{2}\right]\\
	&=\sum_{j_1=1}^\infty\sum_{j_2=1}^\infty\EE\left[\left\langle \cL_{\cK_{{\alpha^*}}}^{1/2}X_{1,i},\phi_{\alpha^*, j_1}\right\rangle_{\sL^2}^2\left\langle \cL_{\cK_{{\alpha^*}}}^{1/2}X_{1,i},\phi_{\alpha^*, j_2}\right\rangle_{\sL^2}^2\right]\\
	&\overset{(\romannumeral1)}{\leq} \sum_{j_1=1}^\infty\sum_{j_2=1}^\infty\left[\EE\left\langle \cL_{\cK_{{\alpha^*}}}^{1/2}X_{1,i},\phi_{\alpha^*, j_1}\right\rangle_{\sL^2}^4\right]^{\frac{1}{2}}\EE\left[\left\langle \cL_{\cK_{{\alpha^*}}}^{1/2}X_{1,i},\phi_{\alpha^*, j_2}\right\rangle_{\sL^2}^4\right]^{\frac{1}{2}}\\
	&\overset{(\romannumeral2)}{\leq}3\sum_{j_1=1}^\infty\sum_{j_2=1}^\infty\EE\left[\left\langle \cL_{\cK_{{\alpha^*}}}^{1/2}X_{1,i},\phi_{\alpha^*, j_1}\right\rangle_{\sL^2}^2\right]\EE\left[\left\langle \cL_{\cK_{{\alpha^*}}}^{1/2}X_{1,i},\phi_{\alpha^*, j_2}\right\rangle_{\sL^2}^2\right]\\
	=&3\sum_{j_1=1}^\infty\sum_{j_2=1}^\infty\mu_{\alpha^*, j_1}\mu_{\alpha^*, j_2}= \mbox{Tr}^2(\cT_{\alpha^*}),
\end{align*}
where $\{\left(\mu_{\alpha^*,j},\phi_{\alpha^*,j}\right)\}_{j=1}^\infty$ is given by the singular value decomposition of $\cT_{{\alpha^*}}$ in \eqref{singular value decomposition}, inequality $(\romannumeral1)$ uses Cauchy-Schwartz inequality, inequality $(\romannumeral2)$ is due to the facts that for any mean-zero Gaussian random variable $\omega$, we have $\EE\left[\omega^{4}\right] = 3\left[\EE\omega^2\right]^2$ and that following from Assumption \ref{assumption: X2}, $\left\{\left\langle \cL_{\cK_{{\alpha^*}}}^{1/2}X_{1,i},\phi_{\alpha^*, j}\right\rangle_{\sL^2}\right\}_{j=1}^\infty$ are mean-zero Gaussian random variables.

Then recalling that $n=N/M$, we write
\begin{align*}
	\EE\left[\mathcal{F}_4\left(S_1,\alpha^*,\Psi_\lambda\right)\II_{\sU_1^c}\right]\leq E^2\sigma^2\mbox{Tr}^2(\cT_{\alpha^*})\lambda^{-1}\frac{M}{N}\PP^{\frac{1}{2}}(\sU_1).
\end{align*}
We have gotten \eqref{equation: expecatation upper bound 2 2}. The proof of Lemma \ref{lemma: expecatation upper bound 2} is then finished.
\end{proof}
 Our further estimation of $\PP\left(\sU_1\right)$ under Assumption \ref{assumption: X2} relies on the following lemma. Recall that $\sU_1$ is defined as
\[\sU_1 = \left\{\bx_1: \left\|(\lambda \cI + \cT_{{\alpha^*}})^{-1/2}(\cT_{{\alpha^*},\bx_1}- \cT_{{\alpha^*}})(\lambda \cI + \cT_{{\alpha^*}})^{-1/2}\right\| \geq 1/2\right\}.\]
\begin{lemma}\label{lemma: probability estimation of U1 2}
	Suppose that Assumption \ref{assumption: X2} is satisfied. Then there holds
	\begin{align}\label{equation: probability estimation of U1 2}
	\PP\left(\sU_1\right)\leq& 2\exp\left(-c_2\min\left\{\frac{N}{M\sN^2(\lambda)},\frac{N}{M\sN(\lambda)}\right\}\right)\\
	&+2\exp\left(-c_3\min\left\{\frac{1}{\lambda^{-2}m^{-2\alpha^*+1}},\frac{1}{\lambda^{-1}m^{-\alpha^*+1/2}}\right\}\right),\nonumber
	\end{align}
where $\sN(\lambda)$ is the effective dimension given by \eqref{effectivedimension}, $c_2$ and $c_3$ are universal constants.
\end{lemma}
\begin{proof}
	Recalling \eqref{simple notation} and the expression of $\cT_{{\alpha^*},\bx_1}$, we first write
	\begin{align*}
		\left\|(\lambda \cI + \cT_{{\alpha^*}})^{-1/2}\left(\cT_{{\alpha^*}}-\cT_{{\alpha^*},\bx_1}\right)(\lambda \cI + \cT_{{\alpha^*}})^{-1/2}\right\|
		\leq\mathcal{D}_1(\bx,\alpha^*,\lambda) + \mathcal{D}_2(\bx_1,\alpha^*,\lambda),
	\end{align*}
	where we define
	\begin{align*}
		\mathcal{D}_1(\bx_1,\alpha^*,\lambda)&:= \left\|(\lambda \cI + \cT_{{\alpha^*}})^{-1/2}\left( \cT_{{\alpha^*},\bx_1} - \frac{1}{n}\sum_{i=1}^{n}\cL_{\cK_{{\alpha^*}}}^{1/2}X_{1,i}\otimes \cL_{\cK_{{\alpha^*}}}^{1/2}X_{1,i}\right)(\lambda \cI + \cT_{{\alpha^*}})^{-1/2}\right\|;\\
		\mathcal{D}_2(\bx_1,\alpha^*,\lambda)&:= \left\|(\lambda \cI + \cT_{{\alpha^*}})^{-1/2}\left( \frac{1}{n}\sum_{i=1}^{n}\cL_{\cK_{{\alpha^*}}}^{1/2}X_{1,i}\otimes \cL_{\cK_{{\alpha^*}}}^{1/2}X_{1,i} - \cT_{{\alpha^*}} \right)(\lambda \cI + \cT_{{\alpha^*}})^{-1/2}\right\|.
	\end{align*}
Then we write
\begin{align}\label{prooflemma1 total1}
		\PP\left(\sU_1\right)=&\PP\left(\left\{\bx_1: \left\|(\lambda \cI + \cT_{{\alpha^*}})^{-1/2}(\cT_{{\alpha^*},\bx_1}- \cT_{{\alpha^*}})(\lambda \cI + \cT_{{\alpha^*}})^{-1/2}\right\| \geq 1/2\right\}\right)\nonumber\\
		\leq&\PP\left(\left\{\bx_1:\mathcal{D}_1(\bx_1,\alpha^*,\lambda) \geq 1/4\right\}\right) + \PP\left(\left\{\bx_1:\mathcal{D}_2(\bx_1,\alpha^*,\lambda) \geq 1/4\right\}\right),
\end{align}

For the term $\PP\left(\left\{\bx_1:\mathcal{D}_2(\bx_1,\alpha^*,\lambda) \geq 1/4\right\}\right)$, we aim to apply Lemma \ref{lemma: Bernstein's inequality for r.e.s} to give an estimation. We define
\begin{align*}
	\cQ_i:=\frac{1}{n}(\lambda \cI + \cT_{{\alpha^*}})^{-1/2}\left(\cL_{\cK_{{\alpha^*}}}^{1/2}X_{1,i}\otimes \cL_{\cK_{{\alpha^*}}}^{1/2}X_{1,i} - \cT_{{\alpha^*}}\right)(\lambda \cI + \cT_{{\alpha^*}})^{-1/2}, \quad i=1,2,\cdots,n.
\end{align*}
Then for any $1\leq i\leq n$, we have 
\begin{align}\label{prooflemma1 3}
	\EE\left[\cQ_i\right]=\frac{1}{n}(\lambda \cI + \cT_{{\alpha^*}})^{-1/2}\EE\left[\cL_{\cK_{{\alpha^*}}}^{1/2}X_{1,i}\otimes \cL_{\cK_{{\alpha^*}}}^{1/2}X_{1,i} - \cT_{{\alpha^*}}\right](\lambda \cI + \cT_{{\alpha^*}})^{-1/2}=0.
\end{align} 
And using the relationship \eqref{relationship between L2 and HS norm}, we write
\begin{align}\label{prooflemma1 5}
\mathcal{D}_2(\bx_1,\alpha^*,\lambda)=\left\|\sum_{i=1}^{n}\cQ_i\right\|\leq \left\|\sum_{i=1}^{n}\cQ_i\right\|_{\sF}.
\end{align}

For any integer $\ell\geq 2$ and any $1\leq i\leq n$, we bound $\EE\left[\|\cQ_{i}\|_{\sF}^{\ell}\right]$ as
	\begin{align}\label{prooflemma1 1}
		&\EE\left[\left\|\cQ_{i}\right\|_{\sF}^{\ell}\right]\\
		\overset{(\romannumeral1)}{\leq}&\frac{1}{n^\ell}\left[\EE\left\|(\lambda \cI + \cT_{{\alpha^*}})^{-1/2}(L_K^{1/2}X_{1,i}\otimes L_K^{1/2}X_{1,i}-\cT_{{\alpha^*}})(\lambda \cI + \cT_{{\alpha^*}})^{-1/2}\right\|_{\sF}^{2\ell}\right]^{\frac{1}{2}}\nonumber\\ 
		=& \frac{1}{n^\ell}\left[\EE\left(\sum_{j=1}^{\infty}\sum_{k=1}^{\infty}\frac{1}{\lambda + \mu_{\alpha^*,j}}\frac{1}{\lambda + \mu_{\alpha^*,k}}\left\langle \left(L_K^{1/2}X_{1,i}\otimes L_K^{1/2}X_{1,i}- \cT_{{\alpha^*}}\right)\phi_{\alpha^*,j}, \phi_{\alpha^*,k}\right\rangle_{\sL^2}^2\right)^\ell\right]^{\frac{1}{2}}\nonumber\\ 
		=&\frac{1}{n^\ell}\Bigg[\sum_{j_1=1}^{\infty } \cdots \sum_{j_\ell=1}^{\infty} \sum_{k_1=1}^{\infty} \cdots \sum_{k_\ell=1}^{\infty}\EE\Bigg\{\frac{1}{\lambda + \mu_{\alpha^*, j_1}}\frac{1}{\lambda + \mu_{\alpha^*,k_1}}\left\langle \left(L_K^{1/2}X_{1,i}\otimes L_K^{1/2}X_{1,i}- \cT_{{\alpha^*}}\right)\phi_{\alpha^*, j_1}, \phi_{\alpha^*,k_1}\right\rangle_{\sL^2}^2\nonumber\\
		&\quad\times\cdots\times
		\frac{1}{\lambda + \mu_{\alpha^*,j_\ell}}\frac{1}{\lambda + \mu_{\alpha^*,k_\ell}}\left\langle \left(L_K^{1/2}X_{1,i}\otimes L_K^{1/2}X_{1,i}- \cT_{{\alpha^*}}\right)\phi_{\alpha^*,j_\ell}, \phi_{\alpha^*,k_\ell}\right\rangle_{\sL^2}^2\Bigg\}\Bigg]^{\frac{1}{2}}\nonumber\\
		\overset{(\romannumeral2)}{\leq}&\frac{1}{n^\ell}\Bigg[\sum_{j_1=1}^{\infty }\sum_{k_1=1}^{\infty}\frac{1}{\lambda + \mu_{\alpha^*, j_1}}\frac{1}{\lambda + \mu_{\alpha^*,k_1}}\left\{\EE\left\langle \left(L_K^{1/2}X_{1,i}\otimes L_K^{1/2}X_{1,i}- \cT_{{\alpha^*}}\right)\phi_{\alpha^*, j_1}, \phi_{\alpha^*,k_1}\right\rangle_{\sL^2}^{2\ell}\right\}^{\frac{1}{\ell}}\times\cdots\times \nonumber\\
		&\sum_{j_\ell=1}^{\infty}\sum_{k_\ell=1}^{\infty}\frac{1}{\lambda + \mu_{\alpha^*,j_\ell}}\frac{1}{\lambda + \mu_{\alpha^*,k_\ell}}\left\{\EE\left\langle \left(L_K^{1/2}X_{1,i}\otimes L_K^{1/2}X_{1,i}- \cT_{{\alpha^*}}\right)\phi_{\alpha^*,j_\ell}, \phi_{\alpha^*,k_\ell}\right\rangle_{\sL^2}^{2\ell}\right\}^{\frac{1}{\ell}}\Bigg]^{\frac{1}{2}},\nonumber
	\end{align}
where $\{\left(\mu_{\alpha^*,j},\phi_{\alpha^*,j}\right)\}_{j=1}^\infty$ is given by the singular value decomposition of $\cT_{{\alpha^*}}$ in \eqref{singular value decomposition}, inequality (\romannumeral1) is from Cauchy-Schwartz inequality, inequality (\romannumeral2) uses H\"older inequality. It remains to estimate
$$\EE\left[\left\langle \left(L_K^{1/2}X_{1,i}\otimes L_K^{1/2}X_{1,i}- \cT_{{\alpha^*}}\right)\phi_{\alpha^*,j}, \phi_{\alpha^*,k}\right\rangle_{\sL^2}^{2\ell}\right],\quad \forall 1\leq i\leq n \mbox{ and }\forall 1\leq j,k<\infty.$$
To this end, we first give the following estimate for the higher-order moment of a Gaussian random variable. Suppose $\omega$ is a mean-zero Gaussian random variable. Then for any integer $t\geq 1$, we have
\begin{align}\label{prooflemma *}
	\EE\left[\omega^{4t}\right]\overset{(\romannumeral1)}{=}(4t -1)!!\left[\EE\omega^2\right]^{2t}\overset{(\romannumeral2)}{\leq} 2^{2t}(2t)!\left[\EE\omega^2\right]^{2t}\overset{(\romannumeral3)}{\leq} 2^{4t}(t!)^2\left[\EE\omega^2\right]^{2t},
\end{align}
where equality (\romannumeral1) is due to the recursive equation that $\EE\left[\omega^{k}\right] = (k-1)\EE\left[\omega^2\right] \EE\left[\omega^{k-2}\right],\forall k\geq2$, inequality (\romannumeral2) follows from the calculation that $(4t-1)!!\leq (4t)!!= 2^{2t}(2t)!$ and inequality (\romannumeral3) uses the fact that $(2t)!=(2t-1)!!(2t)!!\leq 2^{2t} (t!)^2$ which is from $(2t-1)!!\leq (2t)!!$ and $(2t)!!= 2^tt!$. 

When $j\neq k$, we write
\begin{align*}
	&\EE\left[\left\langle \left(\cL_{\cK_{{\alpha^*}}}^{1/2}X_{1,i}\otimes \cL_{\cK_{{\alpha^*}}}^{1/2}X_{1,i}- \cT_{{\alpha^*}}\right)\phi_{\alpha^*,j}, \phi_{\alpha^*,k}\right\rangle_{\sL^2}^{2\ell}\right]\\
	=& \EE\left[\left\langle \cL_{\cK_{{\alpha^*}}}^{1/2}X_{1,i},\phi_{\alpha^*,j}\right\rangle_{\sL^2}^{2\ell}\left\langle \cL_{\cK_{{\alpha^*}}}^{1/2}X_{1,i},\phi_{\alpha^*,k}\right\rangle_{\sL^2}^{2\ell}\right]\\
	\overset{(\romannumeral1)}{\leq}& \left[\EE\left\langle \cL_{\cK_{{\alpha^*}}}^{1/2}X_{1,i},\phi_{\alpha^*,j}\right\rangle_{\sL^2}^{4\ell}\right]^{\frac{1}{2}}\left[\EE\left\langle \cL_{\cK_{{\alpha^*}}}^{1/2}X_{1,i},\phi_{\alpha^*,k}\right\rangle_{\sL^2} ^{4\ell}\right]^{\frac{1}{2}}\\ \overset{(\romannumeral2)}{\leq}&2^{4\ell}(\ell!)^2\left[\EE\left\langle \cL_{\cK_{{\alpha^*}}}^{1/2}X_{1,i},\phi_{\alpha^*,j}\right\rangle_{\sL^2}^{2}\right]^{\ell}\left[\EE\left\langle \cL_{\cK_{{\alpha^*}}}^{1/2}X_{1,i},\phi_{\alpha^*,k}\right\rangle_{\sL^2}^{2}\right]^{\ell}= 2^{4\ell}(\ell!)^2\mu_{\alpha^*,j}^{\ell}\mu_{\alpha^*,k}^{\ell},
\end{align*} where inequality (\romannumeral1) uses Cauchy-Schwarz inequality and inequality (\romannumeral2) follows from \eqref{prooflemma *} with $t=\ell$ as $\left\langle \cL_{\cK_{{\alpha^*}}}^{1/2}X_{1,i}, \phi_{\alpha^*,j}\right\rangle_{\sL^2}$ and $\left\langle \cL_{\cK_{{\alpha^*}}}^{1/2}X_{1,i}, \phi_{\alpha^*,k}\right\rangle_{\sL^2}$ are mean-zero Gaussian random variables. 

When $j = k$, we write
\begin{align*}
	&\EE\left[\left\langle \left(\cL_{\cK_{{\alpha^*}}}^{1/2}X_{1,i}\otimes \cL_{\cK_{{\alpha^*}}}^{1/2}X_{1,i}- \cT_{{\alpha^*}}\right)\phi_{\alpha^*,j}, \phi_{\alpha^*,j}\right\rangle_{\sL^2}^{2\ell}\right]\\
	=& \EE\left[\left(\left\langle \cL_{\cK_{{\alpha^*}}}^{1/2}X_{1,i}\otimes \cL_{\cK_{{\alpha^*}}}^{1/2}X_{1,i} (\phi_{\alpha^*,j}),\phi_{\alpha^*,j}\right\rangle_{\sL^2}-\mu_{\alpha^*,j}\right)^{2\ell}\right]\\
	=& 2^{2\ell}\EE\left[\left(\frac{1}{2}\left\langle \cL_{\cK_{{\alpha^*}}}^{1/2}X_{1,i}\otimes \cL_{\cK_{{\alpha^*}}}^{1/2}X_{1,i} (\phi_{\alpha^*,j}),\phi_{\alpha^*,j}\right\rangle_{\sL^2}-\frac{1}{2}\mu_{\alpha^*,j}\right)^{2\ell}\right]\\
	\overset{(\romannumeral1)}{\leq}& 2^{2\ell-1} \left(\EE\left[\left\langle \cL_{\cK_{{\alpha^*}}}^{1/2}X_{1,i}\otimes \cL_{\cK_{{\alpha^*}}}^{1/2}X_{1,i} (\phi_{\alpha^*,j}),\phi_{\alpha^*,j}\right\rangle_{\sL^2}^{2\ell}\right] +\mu_{\alpha^*,j}^{2\ell} \right)\\
	=&2^{2\ell-1} \left(\EE\left[\left\langle \cL_{\cK_{{\alpha^*}}}^{1/2}X_{1,i}, \phi_{\alpha^*,j}\right\rangle_{\sL^2}^{4\ell}\right] +\mu_{\alpha^*,j}^{2\ell} \right)\\
	\overset{(\romannumeral2)}{\leq}&2^{2\ell-1} \left(2^{4\ell}(\ell!)^2\left[\EE\left\langle \cL_{\cK_{{\alpha^*}}}^{1/2}X_{1,i}, \phi_{\alpha^*,j}\right\rangle_{\sL^2}^{2}\right]^{2\ell} +\mu_{\alpha^*,j}^{2\ell} \right)\leq 2^{6\ell}(\ell!)^2\mu_{\alpha^*,j}^{2\ell},
\end{align*}
where inequality (\romannumeral1) uses Jensen's inequality and inequality (\romannumeral2) follows from \eqref{prooflemma *} with $t=\ell$ as $\left\langle L_K^{1/2}X_{1,i}, \phi_{\alpha^*,j}\right\rangle_{\sL^2}$ is a mean-zero Gaussian random variable.

Combining the above two estimates, for any $1\leq i\leq n$ and $1\leq j,k <\infty$, we have
\begin{align}\label{prooflemma1 2}
	\EE\left[\left\langle \left(\cL_{\cK_{{\alpha^*}}}^{1/2}X_{1,i}\otimes \cL_{\cK_{{\alpha^*}}}^{1/2}X_{1,i}- \cT_{{\alpha^*}}\right)\phi_{\alpha^*,j}, \phi_{\alpha^*,k}\right\rangle^{2\ell}\right]\leq 2^{6\ell}(\ell!)^2\mu_{\alpha^*,j}^\ell\mu_{\alpha^*,k}^\ell.
\end{align}
Then combining \eqref{prooflemma1 1} and \eqref{prooflemma1 2}, we have for any $1\leq i\leq n$,
\begin{align}\label{prooflemma1 4}
	\EE\left[\|\cQ_{i}\|_{\sF}^{\ell}\right]\leq& 2^{2\ell}\ell!\frac{1}{n^\ell}\left(\sum_{j_1=1}^{\infty}\sum_{k_1=1}^{\infty}\frac{\mu_{\alpha^*, j_1}}{\lambda + \mu_{\alpha^*, j_1}}\frac{\mu_{\alpha^*,k_1}}{\lambda + \mu_{\alpha^*,k_1}}\cdots\sum_{j_\ell=1}^{\infty}\sum_{k_\ell=1}^{\infty}\frac{\mu_{\alpha^*,j_\ell}}{\lambda + \mu_{\alpha^*,j_\ell}}\frac{\mu_{\alpha^*,k_\ell}}{\lambda + \mu_{\alpha^*,k_\ell}}\right)^{\frac{1}{2}}\nonumber\\
	=&2^{3\ell}\ell!\frac{\sN^{\ell}(\lambda)}{n^\ell}\leq \frac{\ell!}{2}2^{3\ell+1}\frac{\sN^{\ell}(\lambda)}{n^\ell}.
\end{align}
Recalling that $n=N/M$ and that the space of Hilbert-Schmidt operators on $\sL^2(\sT)$ is a Hilbert space, \eqref{prooflemma1 3} and \eqref{prooflemma1 4} imply that we can apply Lemma \ref{lemma: Bernstein's inequality for r.e.s} with $H=2\frac{M\sN(\lambda)}{N}$, $b_1^2=\cdots=b_n^2=2^7\frac{M^2\sN^2(\lambda)}{N^2}$, $B_n^2=b_1^2+\cdots+b_n^2=2^7\frac{M\sN^2(\lambda)}{N}$ and $x=1/4$ to get
\begin{align}\label{prooflemma1 total2}
	\PP\left(\left\{\bx_1:\mathcal{D}_2(\bx_1,\alpha^*,\lambda) \geq 1/4\right\}\right)\overset{(\dagger)}{\leq}&
	\PP\left(\left\{\bx_1:\left\|\sum_{i=1}^n\mathcal{Q}_i\right\|_{\sF}\geq 1/4\right\}\right)\nonumber\\
	\leq&2\exp\left(-c_2\min\left\{\frac{N}{M\sN^2(\lambda)},\frac{N}{M\sN(\lambda)}\right\}\right),
\end{align}
where we denote $c_2:=1/4^2\left(2^8+1.62\right)$, inequality $(\dagger)$ follows from \eqref{prooflemma1 5}.

For the term $\PP\left(\left\{\bx_1:\mathcal{D}_1(\bx_1,\alpha^*,\lambda) \geq 1/4\right\}\right)$, we will apply Lemma \ref{lemma: Bernstein's inequality for r.v.s} to give a result. Following from the same argument of the proof of \eqref{equation: basic probability estimation of U1 2}, we have
\begin{align*}
		\mathcal{D}_1(\bx_1,\alpha^*,\lambda)\leq 2C_1C_{\alpha^*}\left(C_2^2+C_1C_{\alpha^*}\right)\lambda^{-1}m^{-{\alpha^*}+1/2}\frac{1}{n}\sum_{i=1}^n\|X_{1,i}\|_{\sW^{{\alpha^*},2}}.
\end{align*}
Then for any integer $\ell\geq 2$, we have
\begin{align}\label{prooflemma1 7}
	&\EE\left[\Big|\mathcal{D}_1(\bx_1,\alpha^*,\lambda)-\EE\left[\mathcal{D}_1(\bx_1,\alpha^*,\lambda)\right]\Big|^{\ell}\right]\nonumber\\
	\leq& 2^{\ell-1}\EE\left[\mathcal{D}_1^{\ell}(\bx_1,\alpha^*,\lambda)\right]+2^{\ell-1}\left[\EE\mathcal{D}_1(\bx_1,\alpha^*,\lambda)\right]^{\ell}
	\overset{(\romannumeral1)}{\leq}2^{\ell}\left[\EE\mathcal{D}_1^{4\ell}(\bx_1,\alpha^*,\lambda)\right]^{\frac{1}{4}}\nonumber\\
	\overset{(\romannumeral2)}{\leq}&2^{\ell}C_1^{\ell}C_{\alpha^*}^{\ell}(C_2^2+C_1C_{\alpha^*})^{\ell}\lambda^{-\ell}m^{-{\alpha^*}\ell+\ell/2}\left[\frac{1}{n}\sum_{i=1}^n\EE\left\|X_{1,i}\right\|_{\sW^{{\alpha^*},2}}^{4\ell}\right]^{\frac{1}{4}},
\end{align}
where inequality $(\romannumeral1)$ is from H\"older inequality, inequality (\romannumeral2) uses Jensen's inequality.

While for any $1\leq i\leq n$, we write
\begin{align}\label{prooflemma1 8}
	\EE\left[\left\|X_{1,i}\right\|_{\sW^{{\alpha^*},2}}^{4\ell}\right]&\overset{(\romannumeral1)}{=}\EE\left[\left\|\cL_{\cK_{{\alpha^*}}}^{-1/2}X_{1,i}\right\|_{\sL^2}^{4\ell}\right]= \EE\left[\left(\sum_{j=1}^\infty\left\langle \cL_{\cK_{{\alpha^*}}}^{-1/2}X_{1,i},\phi_{\alpha^*,j}\right\rangle_{\sL^2}^2\right)^{2\ell}\right]\\
	&=\sum_{j_1=1}^\infty\cdots\sum_{j_{2\ell}=1}^\infty\EE\left[\left\langle \cL_{\cK_{{\alpha^*}}}^{-1/2}X_{1,i},\phi_{\alpha^*, j_1}\right\rangle_{\sL^2}^2\cdots\left\langle \cL_{\cK_{{\alpha^*}}}^{-1/2}X_{1,i},\phi_{\alpha^*, j_{2\ell}}\right\rangle_{\sL^2}^2\right]\nonumber\\
	&\overset{(\romannumeral2)}{\leq} \sum_{j_1=1}^\infty\cdots\sum_{j_{2\ell}=1}^\infty\left[\EE\left\langle \cL_{\cK_{{\alpha^*}}}^{-1/2}X_{1,i},\phi_{\alpha^*, j_1}\right\rangle_{\sL^2}^{4\ell}\right]^{\frac{1}{2\ell}}\cdots\left[\EE\left\langle \cL_{\cK_{{\alpha^*}}}^{-1/2}X_{1,i},\phi_{\alpha^*, j_{2\ell}}\right\rangle_{\sL^2}^{4\ell}\right]^{\frac{1}{2\ell}}\nonumber\\
	&\overset{(\romannumeral3)}{\leq}2^{4\ell}(\ell!)^2\sum_{j_1=1}^\infty\cdots\sum_{j_{2\ell}=1}^\infty\EE\left[\left\langle \cL_{\cK_{{\alpha^*}}}^{-1/2}X_{1,i},\phi_{\alpha^*, j_1}\right\rangle_{\sL^2}^2\right]\cdots\EE\left[\left\langle \cL_{\cK_{{\alpha^*}}}^{-1/2}X_{1,i},\phi_{\alpha^*, j_{2\ell}}\right\rangle_{\sL^2}^2\right]\nonumber\\
	&=2^{4\ell}(\ell!)^2\left[\EE\left(\sum_{j=1}^\infty\left\langle \cL_{\cK_{{\alpha^*}}}^{-1/2}X_{1,i},\phi_{\alpha^*,j}\right\rangle_{\sL^2}^2\right)\right]^{2\ell}=2^{4\ell}(\ell!)^2\left[\EE\left\|\cL_{\cK_{{\alpha^*}}}^{-1/2}X_{1,i}\right\|_{\sL^2}^2\right]^{2\ell}\nonumber\\
	&\overset{(\romannumeral4)}{=}2^{4\ell}(\ell!)^2\left[\EE\|X_{1,i}\|_{\sW^{{\alpha^*},2}}^2\right]^{2\ell} \overset{(\romannumeral5)}{\leq} 2^{4\ell}(\ell!)^2\kappa^{4\ell},\nonumber
\end{align}
where $\{\left(\mu_{\alpha^*,j},\phi_{\alpha^*,j}\right)\}_{j=1}^\infty$ is given by the singular value decomposition of $\cT_{{\alpha^*}}$ in \eqref{singular value decomposition} and $\cL_{\cK_{{\alpha^*}}}^{-1/2}$ denotes the inverse operator of $\cL_{\cK_{{\alpha^*}}}^{1/2}$, equality $(\romannumeral1)$ is from \eqref{equation: norm and inner product relationship}, inequality $(\romannumeral2)$ uses H\"older inequality, inequality $(\romannumeral3)$ follows from \eqref{prooflemma *} with $t=\ell$ as $\left\{\left\langle \cL_{\cK_{{\alpha^*}}}^{-1/2}X_i,\phi_{\alpha^*,j}\right\rangle_{\sL^2}\right\}_{j=1}^\infty$ are mean-zero Gaussian random variables, equality $(\romannumeral4)$ is also from \eqref{equation: norm and inner product relationship} and inequality (\romannumeral5) is due to Assumption \ref{assumption: X2}.

Combining \eqref{prooflemma1 7} and \eqref{prooflemma1 8} yields that for any integer $\ell\geq2$,
\begin{align*}
	\EE\left[\Big|\mathcal{D}_1(\bx_1,\alpha^*,\lambda)-\EE\left[\mathcal{D}_1(\bx_1,\alpha^*,\lambda)\right]\Big|^{\ell}\right]\leq \frac{\ell!}{2}2^{2\ell+1}\kappa^{\ell}C_1^{\ell}C_{\alpha^*}^{\ell}(C_2^2+C_1C_{\alpha^*})^{\ell}\lambda^{-\ell}m^{-{\alpha^*}\ell+\ell/2}.
\end{align*}
Noting $\mathcal{D}_1(\bx_1,\alpha^*,\lambda)$ is a non-negative random variable, the above estimation implies that we can apply Lemma \ref{lemma: Bernstein's inequality for r.v.s} with $b^2=2^{5}\kappa^{2}C_1^2C_{\alpha^*}^{2}(C_2^2+C_1C_{\alpha^*})^{2}\lambda^{-2}m^{-2{{\alpha^*}}+1}$, $H=4\kappa C_1C_{\alpha^*}(C_2^2+C_1C_{\alpha^*})\lambda^{-1}m^{-{\alpha^*}+1/2}$ and $x=1/4$ to get
\begin{align}\label{prooflemma1 total3}
	\PP\left(\mathcal{D}_1(\bx_1,\alpha^*,\lambda)\geq 1/4\right)\leq& 2\exp\left(-\frac{1}{2}\min\left\{\frac{x^2}{2b^2},\frac{2x}{H}\right\}\right)\nonumber\\
	\leq&2\exp\left(-c_3\min\left\{\frac{1}{\lambda^{-2}m^{-2\alpha^*+1}},\frac{1}{\lambda^{-1}m^{-\alpha^*+1/2}}\right\}\right),
\end{align}
where we define $c_3:=\min\left\{\frac{1}{2^{11}\kappa^{2}C_1^2C_{\alpha^*}^2(C_2^2+C_1C_{\alpha^*})^{2}},\frac{1}{2^4\kappa C_1C_{\alpha^*}(C_2^2+C_1C_{\alpha^*})}\right\}$.

Finally, combining \eqref{prooflemma1 total1}, \eqref{prooflemma1 total2} and \eqref{prooflemma1 total3} yields
\begin{align*}
	\PP\left(\sU_1\right)\leq 2\exp\left(-c_2\min\left\{\frac{N}{M\sN^2(\lambda)},\frac{N}{M\sN(\lambda)}\right\}\right)+2\exp\left(-c_3\min\left\{\frac{1}{\lambda^{-2}m^{-2\alpha^*+1}},\frac{1}{\lambda^{-1}m^{-\alpha^*+1/2}}\right\}\right).
\end{align*}
We have gotten \eqref{equation: probability estimation of U1 2}, and the proof of Lemma \ref{lemma: probability estimation of U1 2} is then completed.
\end{proof}
Now we are in the position to prove Theorem \ref{theorem: upper bound2}.

\noindent
\emph{Proof of Theorem \ref{theorem: upper bound2}}.
Let $0\leq \theta\leq \nu_{\Psi}$ and take $\lambda=N^{-\frac{1}{1+2\theta+p}}$, $1/m\leq o\left(N^{-\frac{2+2\theta}{(2{\alpha^*}-1)(1+2\theta+p)}}\log^{-\frac{2}{2\alpha^*-1}}N\right)$ and $M\leq o\left(N^{\frac{1+2\theta-p}{1+2\theta+p}}\log^{-1}N\right)$. Following from \eqref{estimation of N(lambda)} in Lemma \ref{lemma: estimation of N(lambda)}, we have
\begin{align*}
	\frac{M\sN(\lambda)}{N}\lesssim \frac{M\sN^2(\lambda)}{N} \leq o\left(\log^{-1}N\right),
\end{align*}
and
\begin{align*}
	\lambda^{-2}m^{-2\alpha^*+1}\leq \lambda^{-1}m^{-\alpha^*+1/2}\leq o\left(\log^{-1} N\right).
\end{align*}
Then following from \eqref{equation: probability estimation of U1 2} in Lemma \ref{lemma: probability estimation of U1 2}, we have
\begin{align*}
\lambda^{-2\theta}\PP\left(\sU_1\right)\leq&2\lambda^{-2\theta}\exp\left(-\frac{c_2}{2}\min\left\{\frac{N}{M\sN^2(\lambda)},\frac{N}{M\sN(\lambda)}\right\}\right)\nonumber\\
&+2\lambda^{-2\theta}\exp\left(-\frac{c_3}{2}\min\left\{\lambda^{2}m^{2{{\alpha^*}}-1},\lambda m^{{\alpha^*}-1/2}\right\}\right)\lesssim 1,
\end{align*}
and
\begin{align*}
	\lambda^{-1}\PP^{\frac{1}{2}}\left(\sU_1\right)\nonumber\leq&\sqrt{2}\lambda^{-1}\exp\left(-\frac{c_2}{2}\min\left\{\frac{N}{M\sN^2(\lambda)},\frac{N}{M\sN(\lambda)}\right\}\right)\nonumber\\
	&+\sqrt{2}\lambda^{-1}\exp\left(-\frac{c_3}{2}\min\left\{\lambda^{2}m^{2{{\alpha^*}}-1},\lambda m^{{\alpha^*}-1/2}\right\}\right)\lesssim 1.
\end{align*}
These together with \eqref{equation: expecatation upper bound 2 1} and \eqref{equation: expecatation upper bound 2 2} in Lemma \ref{lemma: expecatation upper bound 2} yield
\begin{align}\label{equation: expecatation upper bound 2 3}
	\EE\Big[\mathcal{F}_2\left(\bx_1,\alpha^*,\Psi_\lambda\right)\II_{\sU_1}\Big]\lesssim \lambda^{2\theta}
\end{align}
and
\begin{align}\label{equation: expecatation upper bound 2 4}
	\EE\Big[\mathcal{F}_4\left(S_1,\alpha^*,\Psi_\lambda\right)\II_{\sU_1}\Big]\lesssim \frac{M}{N},
\end{align}

Recalling that $n= N/M$, using \eqref{estimation of N(lambda)} in Lemma \ref{lemma: estimation of N(lambda)}, Combining \eqref{equation: expecatation upper bound 1}, \eqref{equation: upper bound of expectation1}, \eqref{equation: upper bound of expectation2}, \eqref{equation: expecatation upper bound 2 3} and \eqref{equation: expecatation upper bound 2 4} with \eqref{equation: rates upper bound21}, \eqref{equation: rates upper bound22}, \eqref{equation: rates upper bound23} and \eqref{equation: rates upper bound24} in Lemma \ref{lemma: rates upper bound2}, we write
\begin{align*}
	&\EE\left[\left\|\hat{\beta}_{S,{\alpha^*},\lambda}-\beta_0\right\|_{\sW^{{\alpha^*},2}}^2\right]\\
	\lesssim& \lambda^{2\theta}+\lambda^{-2}m^{-2{{\alpha^*}}+1} + \lambda^{-1}\frac{\sN(\lambda)}{N} + \frac{1}{N}\lesssim \lambda^{2\theta}+\lambda^{-2}m^{-2{{\alpha^*}}+1} + \frac{\lambda^{-1-p}}{N}\lesssim N^{-\frac{2\theta}{1+2\theta+p}}.
\end{align*}
We have completed the proof of Theorem \ref{theorem: upper bound2}. \qed
\subsection{Establishing Lower Rates}\label{subsection: lower rates}

In this subsection, we will establish the lower bounds presented in Theorem \ref{theorem: lower bound}. There is already a standard procedure to establish minimax lower bounds based on Fano's method. We follow the same procedure as in our previous paper \cite{liu2024statistical} to prove Theorem \ref{theorem: lower bound}. 

\noindent
\emph{Proof of Theorem \ref{theorem: lower bound}}. Recall that $\left\{\mu_{\alpha^*,j}\right\}_{j\geq1}$ is a positive and decreasing sequence of eigenvalues of $\cT_{{\alpha^*}}$ satisfying $\mu_{\alpha^*,j}\asymp j^{-1/p}$ for some $0< p \leq 1$. That is, there exists a constant $c>0$ such that
\begin{equation}\label{eiganvlaue}
	\mu_{\alpha^*,j+1}\leq \mu_{\alpha^*,j} \mbox{ and } cj^{-1/p}\leq\mu_{\alpha^*,j} \leq \frac{1}{c} j^{-1/p}, \quad \forall j\geq 1.
\end{equation} 
 It is sufficient to consider the case that $\epsilon$ is a centered Gaussian random variable with variance $\sigma^2$ and independent of $X$, i.e., $\epsilon\sim N(0,\sigma^2)$. Then Assumption \ref{assumption: epsilon2} is satisfied with $\sigma>0$. 
 
 Take $J=\lceil a N^{\frac{p}{1+p+2\theta}} \rceil$, which denotes the smallest integer larger than $a N^{\frac{p}{1+p+2\theta}}$ with some constant $a>8$ to be specified later. The well-known Varshamov-Gilbert bound (see, e.g., \cite{duchi2016lecture}) guarantees that there exists a set $\Lambda=\left\{\iota^{(1)},\cdots,\iota^{(L)}\right\}\subset \left\{-1,1\right\}^{J}$ such that
 \begin{align*}
 	L=\left|\Lambda\right|\geq \exp(J/8)
 \end{align*}
and
 	\[\begin{aligned}
 	\left\|\iota - \iota'\right\|_1 = \sum_{j=1}^{J}\left|\iota_j-\iota'_j\right|\geq J/2
 \end{aligned}\] for any $\iota \neq \iota'$ with  $\iota,\iota'\in \Lambda$. Given $0\leq \theta < +\infty$, define
\begin{align*}
	\beta_i = \sum_{j=J+1}^{2J}\frac{1}{\sqrt{J}}\iota^{(i)}_{j-J}\cL_{\cK_{{\alpha^*}}}^{1/2}\mu_{\alpha^*,j}^{\theta}\phi_{\alpha^*,j} = \cL_{\cK_{{\alpha^*}}}^{1/2}\cT_{{\alpha^*}}^{\theta}(g_i), \quad i=1,\cdots,L,
\end{align*}
where $\{\phi_{\alpha^*,j}\}_{j\geq 1}$ are the eigenfunctions given by the singular value decomposition of $\cT_{{\alpha^*}}$ in \eqref{singular value decomposition}, and $g_i=\sum_{j=J+1}^{2J}\frac{1}{\sqrt{J}}\iota^{(i)}_{j-J}\phi_{\alpha^*,j}$ satisfies $\left\|g_i\right\|_{\sL^2}^2=1$. Then $\left\{\beta_i\right\}_{i=1}^L$ satisfy Assumption \ref{assumption: regularity condition} with $0\leq\theta< \infty$. 

For any $1\leq i_1\neq i_2\leq L$, we have
\[\begin{aligned}
	\left\|\beta_{i_1} - \beta_{i_2}\right\|_{\sW^{{\alpha^*},2}}^2 &=\left\|\cL_{\cK_{{\alpha^*}}}^{1/2}\cT_{{\alpha^*}}^{\theta}(g_{i_1}-g_{i_2})\right\|_{\sW^{{\alpha^*},2}}^2\overset{(\romannumeral1)}{=}\left\|\cT_{{\alpha^*}}^{\theta}(g_{i_1}-g_{i_2})\right\|_{\sL^2}^2\\
	&= \sum_{j=J+1}^{2J}\frac{1}{J}\mu_{\alpha^*,j}^{2\theta}\left(\iota^{(i_1)}_{j-J} - \iota^{(i_2)}_{j-J}\right)^2\\
	&\geq \mu_{\alpha^*, 2J}^{2\theta}\frac{2}{J}\sum_{j=J+1}^{2J}\left|\iota^{(i_1)}_{j-J} - \iota^{(i_2)}_{j-J}\right|\\
	&\overset{(\romannumeral2)}{\geq} \mu_{\alpha^*, 2J}^{2\theta}\frac{2}{J}\frac{J}{2}\overset{(\romannumeral3)}{\geq} c^{2\theta} 2^{-\frac{2\theta}{p}}J^{-\frac{2\theta}{p}},
\end{aligned}\] where equality (\romannumeral1) is due to \eqref{equation: norm and inner product relationship}, inequalities (\romannumeral2) and (\romannumeral3) are from \eqref{eiganvlaue}. For any $1\leq i\leq L$, denote by $\{P_i\}_{i=1}^L$ the joint probability distributions of $(X,Y)$ with $Y=\langle X,\beta_i\rangle_{\sL^2}+\epsilon$ and $\epsilon \thicksim N(0,\sigma^2)$. Then for $\forall 1 \leq i_1\neq i_2 \leq L$, the Kullback-Leibler divergence (KL-divergence) between $P_{i_1}$ and $P_{i_2}$ can be calculated as
\begin{align}\nonumber
	\mathcal{D}_{kl} (P_{i_1} \| P_{i_2})
	&=\frac{1}{2\sigma^2}\left\|\cL_\cC^{1/2}\left(\beta_{i_1} - \beta_{i_2}\right)\right\|_{{\sL}^2}^2=\frac{1}{2\sigma^2}\left\|\cL_\cC^{1/2}\cL_{\cK_{{\alpha^*}}}^{1/2}\cT_{{\alpha^*}}^{\theta}\left(g_{i_1}-g_{i_2}\right)\right\|_{{\sL}^2}^2\nonumber\\
	&\overset{(\romannumeral1)}{=}\frac{1}{2\sigma^2}\left\|\cT_{{\alpha^*}}^{1/2+\theta}\left(g_{i_1}-g_{i_2}\right)\right\|_{{\sL}^2}^2= \frac{1}{2\sigma^2}\sum_{j=J+1}^{2J}\frac{1}{J}\mu_{\alpha^*,j}^{1+2\theta}\left|\iota^{({i_1})}_{j-J} - \iota^{({i_2})}_{j-J}\right|^2\nonumber\\
	&\overset{(\romannumeral2)}{\leq} \frac{2}{\sigma^2}\mu_{\alpha^*,J}^{1+2\theta}\overset{(\romannumeral3)}{\leq} \frac{2}{\sigma^2c^{1+2\theta}}J^{-\frac{1+2\theta}{p}},\nonumber
\end{align} where equality (\romannumeral1) is from \eqref{singular value decomposition}, inequalities (\romannumeral2) and (\romannumeral3) are due to \eqref{eiganvlaue}.
Recalling that $J=\lceil a N^{\frac{p}{1+p+2\theta}} \rceil$, following from a direct application of Fano's method (see, \cite{duchi2016lecture} or Lemma 13 in \cite{liu2024statistical}), the above two estimates imply that there holds
\[\begin{aligned}
	&\inf_{\hat{\beta}_S}\sup_{\beta_0}\mathbb{P}\left\{\left\|\hat{\beta}_S-\beta_0\right\|^2_{\sW^{{\alpha^*},2}}\geq \frac{c^{2\theta}}{4}2^{-\frac{2\theta}{p}}J^{-\frac{2\theta}{p}} \geq \frac{c^{2\theta}}{4}2^{-\frac{2\theta}{p}}a^{-\frac{2\theta}{p}}N^{-\frac{2\theta}{1+p+2\theta}}\right\}\\  
	&\geq 1- \frac{\frac{2N}{\sigma^2c^{1+2\theta}}J^{-\frac{1+2\theta}{p}}+\log2}{\log L} \geq 1-\frac{\frac{2N}{\sigma^2c^{1+2\theta}}J^{-\frac{1+2\theta}{p}}+\log 2}{J/8}\\
	&\geq 1 - a^{-\frac{1+2\theta + p}{p}}\frac{16}{\sigma^2c^{1+2\theta}} N^{1-\frac{p}{1+2\theta + p}\cdot\frac{1+2\theta + p}{p}} - \frac{8\log2}{a N^{\frac{p}{1+p+2\theta}}}\\
	&= 1- a^{-\frac{1+2\theta + p}{p}}\frac{16}{\sigma^2c^{1+2\theta}} - \frac{8\log2}{a}N^{-\frac{p}{1+p+2\theta}}.
\end{aligned}\]
Therefore, we have
\[\mathop{\lim\inf}_{N\to \infty} \inf_{\hat{\beta}_{S}} \sup_{\beta_0} \mathbb{P} \left\{\left\|\hat{\beta}_S-\beta_0\right\|^2_{\sW^{{\alpha^*},2}}\geq \frac{c^{2\theta}}{4}2^{-\frac{2\theta}{p}}a^{-\frac{2\theta}{p}}N^{-\frac{2\theta}{1+p+2\theta}}\right\} = 1- a^{-\frac{1+2\theta + p}{p}}\frac{16}{\sigma^2c^{1+2\theta}}\]
and then
\[\lim_{a\to \infty}\mathop{\lim\inf}_{N\to \infty} \inf_{\hat{\beta}_{S}} \sup_{\beta_0} \mathbb{P} \left\{\left\|\hat{\beta}_S-\beta_0\right\|^2_{\sW^{{\alpha^*},2}}\geq \frac{c^{2\theta}}{4}2^{-\frac{2\theta}{p}}a^{-\frac{2\theta}{p}}N^{-\frac{2\theta}{1+p+2\theta}}\right\} = 1.\]
Taking $\gamma = c^{2\theta}2^{-\frac{2\theta}{p}}a^{-\frac{2\theta}{p}}/4$, we have
\[	\lim_{\gamma \to 0}\mathop{\lim\inf}_{N\to \infty} \inf_{\hat{\beta}_{S}} \sup_{\beta_0} \mathbb{P} \left\{\left\|\hat{\beta}_S-\beta_0\right\|^2_{\sW^{{\alpha^*},2}}\geq \gamma N^{-\frac{2\theta}{1+2\theta+p}}\right\} = 1.\]
This completes the proof of Theorem \ref{theorem: lower bound}.\qed

\section*{Appendix A. Sobolev Inequalities}
The following lemma provides some well-known Sobolev inequalities in the unanchored Sobolev spaces (see, e.g., \cite{evans2022partial}).
\begin{lemma}\label{lemma: Sobolev inequalities}
	Suppose that $\beta,\gamma\in \sW^{\alpha,2}(\sT)$ for some $\alpha>1/2$. Then there exists constant $C_1,C_2>0$ such that
	\begin{equation}\label{equation: Sobolev inequalities}
		\|\beta\gamma\|_{\sW^{\alpha,2}}\leq C_1\|\beta\|_{\sW^{\alpha,2}}\|\gamma\|_{\sW^{\alpha,2}} \quad \mbox{  and  } \quad
		\|\beta\|_{\sL^2}\leq C_2\|\beta\|_{\sW^{\alpha,2}}.
	\end{equation}
\end{lemma}
The following lemma is a direct result of the continuous embedding condition \eqref{equation: continuous embedding condition}.
\begin{lemma}\label{lemma: Morrey's inequality}
	Suppose that $\beta\in \sW^{\alpha,2}(\sT)$ for some $\alpha>1/2$. Then there exists a constant $\widetilde{C}_{\alpha}>0$ only depending on $\alpha$ such that
	\begin{align}\label{equation: Morrey's inequality}
		\sup_{t,t'\in \sT}\frac{|\beta(t)-\beta(t')|}{|t-t'|^{\alpha-1/2}}<\widetilde{C}_{\alpha}\|\beta\|_{\sW^{\alpha,2}}.
	\end{align}
\end{lemma}
The following lemma shows that the Riemann sum of a function $\beta\in \sW^{\alpha,2}(\sT)$ for some $\alpha>1/2$ at the discrete sample points $\{r_k\}_{k=1}^{m+1}$ satisfying Assumption \ref{assumption: sampling scheme} can approximate the integral of $\beta$. The proof of this lemma can be found in \cite{wang2022functional}.
\begin{lemma}\label{lemma: integral approximation}
	Suppose that $\beta\in \sW^{\alpha,2}(\sT)$ for some $\alpha>1/2$ and Assumption \ref{assumption: sampling scheme} is satisfied with $C_d>0$, then there holds
	\begin{equation}\label{equation: integral approximation}
		\left|\int_{\sT}\beta(t)dt-\sum_{k=1}^{m}(r_{k+1}-r_k)\beta(r_k)\right|\leq C_{\alpha}\|\beta\|_{\sW^{\alpha,2}}m^{-\alpha+1/2},
	\end{equation}
where $C_{\alpha}:=\widetilde{C}_{\alpha} C_d^{\alpha-1/2}$ is a constant depending on $\alpha$.
\end{lemma}

\section*{Appendix B. Some Technical Lemmas}

The following lemma establishes an upper bound for deviation probability of a positive random variable with bounded arbitrary-order moment.
\begin{lemma}\label{lemma: Bernstein's inequality for r.v.s}
	 Suppose that a random variable $V\geq 0$ satisfy the condition
	 \begin{align*}
	 	\EE\left[\Big|V-\EE\left[V\right]\Big|^{\ell}\right]\leq \frac{\ell!}{2}b^2H^{\ell-2}, \quad \ell\geq2.
	 \end{align*}
 Then for any $x\geq 0$, there holds
 \begin{align*}
 	\PP\left(V\geq x\right) \leq 2\exp\left(-\frac{1}{2}\min\left\{\frac{x^2}{2b^2},\frac{2x}{H}\right\}\right).
 \end{align*}
\end{lemma}
\begin{proof}
	Following the similar argument as in Lemma 3.18 of \cite{duchi2016lecture}, we get $V-\EE\left[V\right]$ is $(2b^2,\frac{H}{2})$-sub-exponential. Combining this with Proposition 3.15 of \cite{duchi2016lecture} completes the proof.
\end{proof}
The following lemma provides an upper bound for tail probability of the sum of random variables in a Hilbert space with bounded arbitrary-order moment. The proof of it can be seen in \cite{yurinskiui1976exponential}.
\begin{lemma}\label{lemma: Bernstein's inequality for r.e.s}
Let $\sH$ be a Hilbert space endowed with norm $\|\cdot\|_{\sH}$. Suppose a finite sequence of independent random elements $\{\xi_i\}_{i=1}^n \in \sH$ satisfy conditions
\begin{align*}
	\EE[\xi_i]&=0,\\
	\EE\left[\|\xi_i\|_{\sH}^\ell\right]&\leq \frac{\ell!}{2}b_i^2H^{\ell-2},\quad \ell\geq2.
\end{align*}
Let $B_n^2=b_1^2+\cdots+b_n^2.$ Then for any $x>0$, there holds
\begin{align*}
	\PP\left(\|\xi_1+\cdots+\xi_n\|_{\sH}\geq x\right)\leq 2\exp\left(-\frac{x^2}{2\left(B_n^2+1.62xH\right)}\right).
\end{align*}
\end{lemma}

\section*{Acknowledgments}
The work described in this paper is supported by the National Natural Science Foundation of China [Grants No.12171039]. The corresponding author is Lei Shi.

\bibliographystyle{plain}
\bibliography{main}
\end{document}